\documentclass[11pt]{article}
\usepackage[utf8]{inputenc}

\usepackage[top=1in, bottom=1in, left=1in, right=1in]{geometry}

\usepackage[utf8]{inputenc} %
\usepackage[T1]{fontenc}    %

\usepackage{url}            %
\usepackage{booktabs}       %
\usepackage{amsfonts}       %
\usepackage{nicefrac}       %
\usepackage{microtype}      %
\usepackage{xcolor}         %

\usepackage[american]{babel}
\usepackage[utf8]{inputenc} %
\usepackage[T1]{fontenc}    %
\usepackage{csquotes}
\usepackage[hidelinks]{hyperref}       %
\usepackage{url}            %
\usepackage{booktabs}       %
\usepackage{amsfonts}       %
\usepackage{nicefrac}       %
\usepackage{microtype}      %
\usepackage{xcolor}         %
\usepackage{array}          %
\usepackage{tikz}
\usetikzlibrary{calc,positioning,decorations.pathreplacing,shapes}

\usepackage{amsmath}
\usepackage{amsthm}
\usepackage{amssymb}
\PassOptionsToPackage{normalem}{ulem}
\usepackage{ulem}
\usepackage{dsfont}
\usepackage[noabbrev, capitalise]{cleveref}

\usepackage[hide=false,setmargin=true,marginparwidth=1.2in]{marginalia}

\usepackage{enumitem}
\usepackage{xspace}
\usepackage{graphicx}
\usepackage{subcaption}

\usepackage[%
    minnames=1,maxnames=99,maxcitenames=2,
    style=numeric-comp,
    sorting=none,
    sortcites=false, %
    doi=true,
    url=true,
    uniquename=init,
    giveninits=true,
    backend=biber,
    hyperref=auto,
    natbib=true]{biblatex}

\DeclareFieldFormat{bibentrysetcount}{\mkbibparens{\mknumalph{#1}}}
\DeclareFieldFormat{labelnumberwidth}{\mkbibbrackets{#1}}

\renewbibmacro{in:}{%
  \ifentrytype{article}{}{\printtext{\bibstring{in}\intitlepunct}}}
\renewbibmacro{in:}{%
  \ifentrytype{inproceedings}{}{\printtext{\bibstring{in}\intitlepunct}}}

\DeclareNameAlias{sortname}{family-given}

\theoremstyle{plain}
\newtheorem{thm}{\protect\theoremname}[section]
  \theoremstyle{plain}
  \newtheorem{prop}[thm]{\protect\propositionname}
  \theoremstyle{plain}
  
  \theoremstyle{remark}
  \newtheorem{rem}[thm]{\protect\remarkname}
  \theoremstyle{plain}
  \newtheorem{conjecture}[thm]{\protect\conjecturename}
  \theoremstyle{plain}
  \newtheorem{lem}[thm]{\protect\lemmaname}
  \theoremstyle{plain}

\usepackage{mathtools}

\usepackage{contour}
\usepackage[normalem]{ulem}

\newcommand{\saferef}[2]{%
#1~\texorpdfstring{\ref{#2}}{\ref*{#2}}}

\newtheorem{assumption}[thm]{Assumption}
\crefname{assumption}{Assumption}{Assumption}
\newtheorem{example}[thm]{Example}
\crefname{example}{Example}{Example}
\newtheorem{definition}[thm]{Definition}
\crefname{definition}{Definition}{Definition}

\makeatother
\crefformat{equation}{(#2#1#3)}

\numberwithin{equation}{section}

\contourlength{0.8pt}

\renewcommand{\underline}[1]{%
  \uline{\phantom{#1}}%
  \llap{\contour{white}{#1}}%
}

\definecolor{OliveGreen}{rgb}{0,0.6,0}
\definecolor{MihaiBlue}{rgb}{0,0,0.6}

\usepackage{babel}
  \providecommand{\conjecturename}{Conjecture}
  \providecommand{\corollaryname}{Corollary}
  \providecommand{\lemmaname}{Lemma}
  \providecommand{\propositionname}{Proposition}
  \providecommand{\remarkname}{Remark}
\providecommand{\theoremname}{Theorem}

\global\long\def\cov{\mathbf{Cov}}
\global\long\def\var{\mathbf{Var}}

\global\long\def\norm#1{\left| #1\right| }
\global\long\def\abs#1{\left|#1\right|}

\global\long\def\floor#1{ {\left\lfloor #1\right\rfloor} }

\global\long\def\bN{\mathbb{N}}

\global\long\def\bR{\mathbb{R}}

\global\long\def\cN{\mathcal{N}}

\global\long\def\rh{\rho}

\global\long\def\ta{\tau}

\global\long\def\vp{\varphi}

\global\long\def\dequal{\stackrel{d}{=}}
\global\long\def\defequal{\coloneqq}

\global\long\def\nin{{n_\text{in}}}
\global\long\def\nout{{n_\text{out}}}
\global\long\def\zout{z_{\text{out}}}
\global\long\def\Wout{W_{\text{out}}}

\title{The Neural Covariance SDE: \\ 
Shaped Infinite Depth-and-Width Networks at Initialization}

\author{
  Mufan (Bill) Li\thanks{
    University of Toronto and Vector Institute, \texttt{mufan.li@mail.utoronto.ca}
  }
  \and 
  Mihai Nica\thanks{
    University of Guelph and Vector Institute, \texttt{nicam@uoguelph.ca}
  }
  \and 
  Daniel M. Roy\thanks{
  University of Toronto and Vector Institute, \texttt{daniel.roy@utoronto.ca}
  }
}

\begin{document}

\maketitle

\begin{abstract}

The logit outputs of a feedforward neural network at initialization are conditionally Gaussian, given a random covariance matrix defined by the penultimate layer. 
In this work, we study the distribution of this random matrix. 
Recent work has shown that shaping the activation function as network depth grows large is necessary for this covariance matrix to be non-degenerate. 
However, the current infinite-width-style understanding of this shaping method is unsatisfactory for large depth: 
infinite-width analyses ignore the microscopic fluctuations from layer to layer, but these fluctuations accumulate over many layers. 

To overcome this shortcoming, we study the random covariance matrix in the shaped infinite-depth-and-width limit. 
We identify the precise scaling of the activation function necessary to arrive at a non-trivial limit, and show that the random covariance matrix is governed by a stochastic differential equation (SDE) that we call the Neural Covariance SDE. 
Using simulations, we show that the SDE closely matches the distribution of the random covariance matrix of finite networks.  
Additionally, we recover an if-and-only-if condition for exploding and vanishing norms of large shaped networks based on the activation function. 
\end{abstract}

\section{Introduction}

Of the many milestones in deep learning theory,
the precise characterization of the infinite-width limit of neural networks at initialization as a Gaussian process with a non-random covariance matrix 
\citep{neal1995bayesian,lee2018deep}
was a turning point.
The so-called Neural Network Gaussian process (NNGP) theory laid the mathematical foundation to study various limiting training dynamics under gradient descent \citep{jacot2018neural,du2019gradient,allen2019convergence,zou2020gradient,chizat2019lazy,lee2019wide,yang2019scaling,yang2020tensor,arora2019exact,chen2021how}. 
The Neural Tangent Kernel (NTK) limit formed the foundation for a rush of theoretical work, including advances in our understanding of generalization for wide networks  \cite{ji2019polylogarithmic,ba2019generalization,bartlett2021deep}. 
Besides the NTK limit, the infinite-width mean-field limit was developed \cite{RVE18,cb18,SS18, MeiE7665}, where the different parameterization demonstrates benefits for feature learning and hyperparameter tuning \cite{featurelearning,yang2022tensor,ba2022high}.

\begin{figure}[t!]
\centering
\begin{tikzpicture}[scale=1, x=0.5\textwidth,y=0.5\textwidth, %
  nplot/.style={anchor=north west}, %
  ncap/.style={anchor=south west}, %
  caption/.style={align=left,font=\normalsize,scale=0.75}, %
  ]

\node[ncap] (figacap) at (0.2,0.8) [caption] {
(Left Column) \textbf{Unshaped ReLU DNNs}, see (\ref*{eq:defn_fcnet})
 };
\node[nplot] (figa) at (0,0.8)  {{\includegraphics[width=0.5\textwidth]{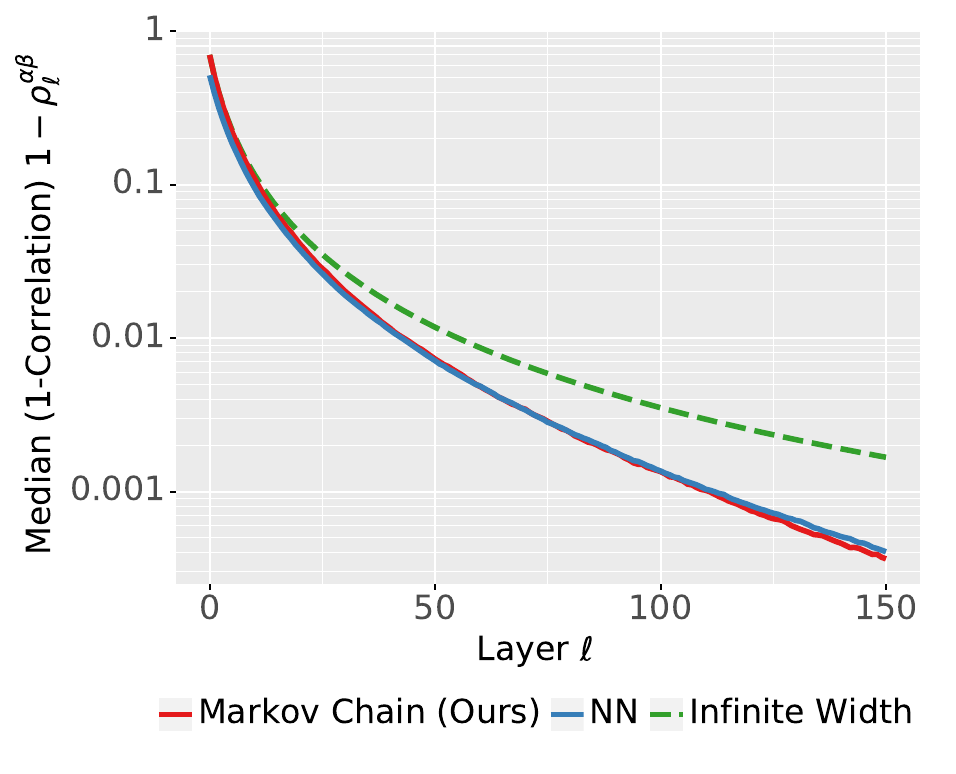}}};  %

\node[ncap] (figbcap) at (1.15,0.8) [caption] {
(Right Column) \textbf{Shaped ReLU DNNs}, see Definition \ref*{def:relu-shaping}
 };
\node[nplot] (figb) at (1,0.8) {{\includegraphics[width=0.5\textwidth]{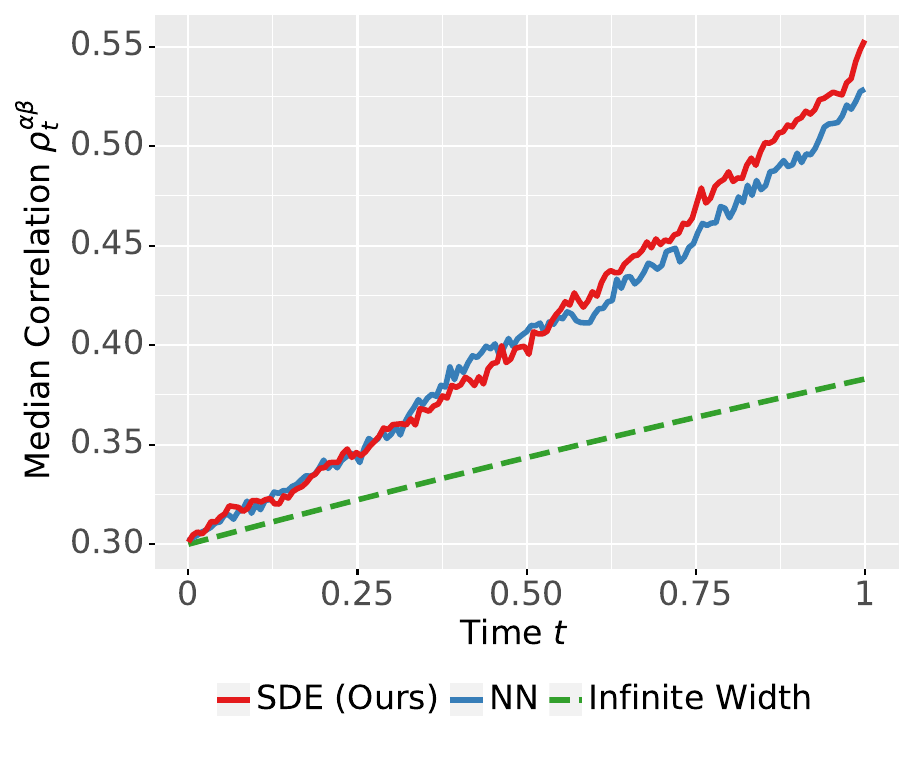}}};

\node[nplot] (figc) at (0,0) {{\includegraphics[width=0.5\textwidth]{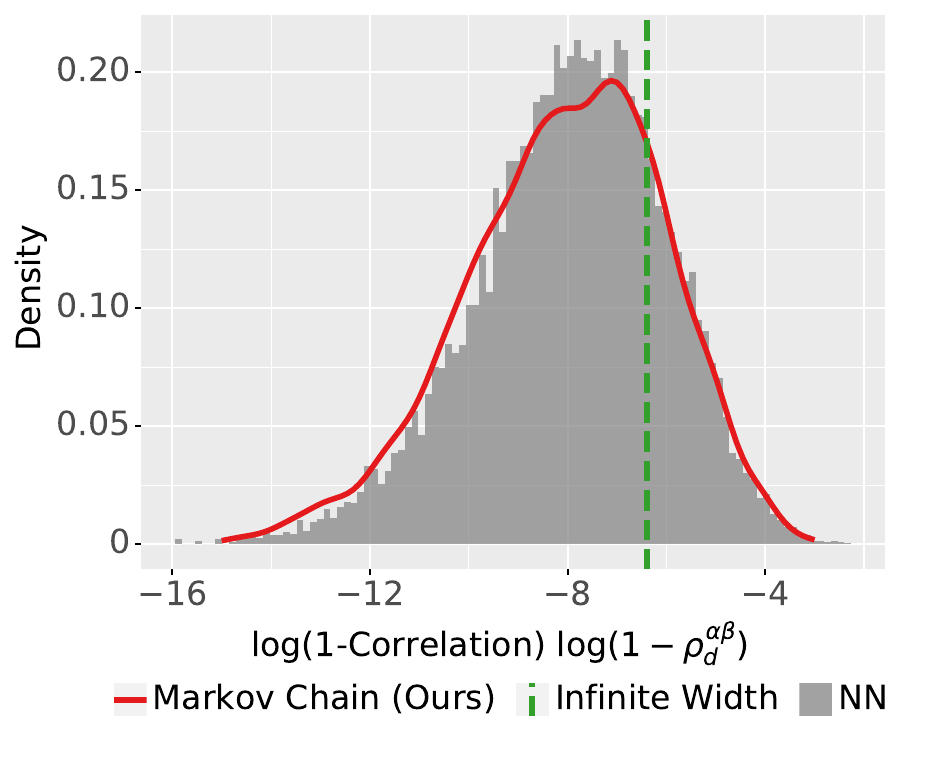}}};

\node[nplot] (figd) at (1.075,0) {{\includegraphics[width=0.455\textwidth]{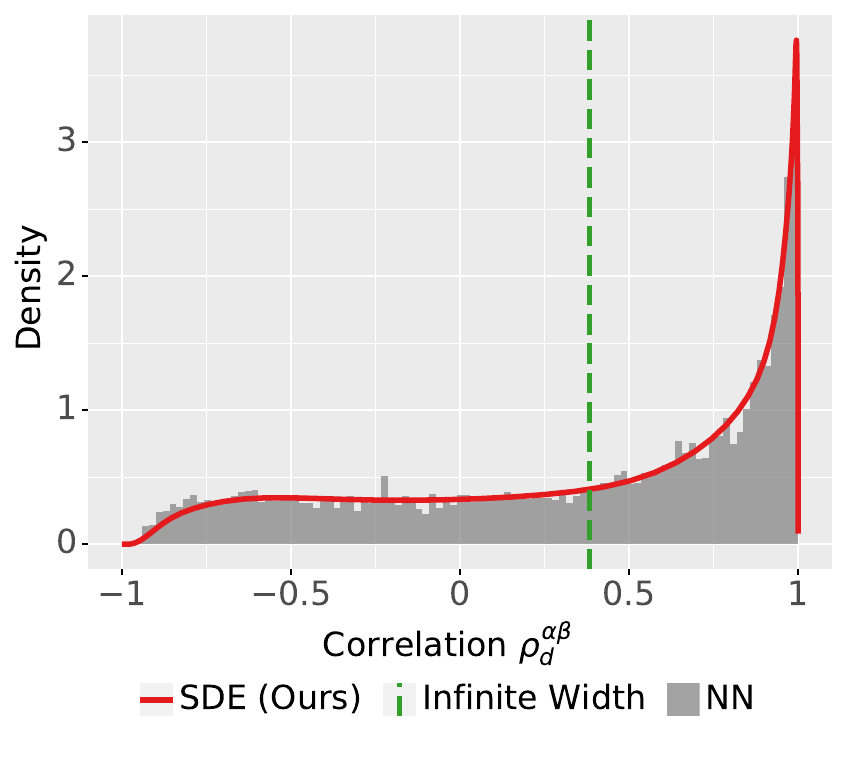}}};

\end{tikzpicture}

\caption{
Simulations of correlation $\rho_\ell^{\alpha \beta} = \frac{ \langle{ \vp^\alpha_\ell, \vp^\beta_\ell \rangle} }{ |\vp^\alpha_\ell||\vp^\beta_\ell| }$ between post-activation vectors in ReLU networks, comparing \emph{finite NNs} vs. our theoretical predictions vs. infinite-width paradigm. \textbf{Left Column:} $\rho^{\alpha \beta}_\ell$ vs. our Markov chain \eqref{eq:corr_update} vs. infinite-width update $\rho_{\ell+1}=c K_1(\rho_\ell)$ (see \eqref{eq:corr_update} and note the log scale and $1-\rho$ here). \textbf{Right Column:} $\rho^{\alpha \beta}_\floor{ t n}$ vs. our Neural Covariance SDE vs. ODE $d\rho_t = \nu(\rho_t)\,dt$ (see \cref{thm:relu_corr}). \textbf{Top Row:} \emph{Median} $\rho$ as a function of layer. \textbf{Bottom Row:} \emph{Full distribution} at final layer $\ell = d$.  \textbf{Simulation details:} $n=d=150,\rho_0=0.3$,  $2^{13}$ samples for each. In right column: $c_+=0,c_-=-1$, DE step size $1\mathrm{e}{-2}$. Densities from kernel density estimation. 
}
\label{fig:density_path_shape}
\end{figure}

Fundamentally, the infinite-width paradigm derives results from the assumption that the depth of the network is held \emph{fixed} while the widths of all layers grow to infinity. 
Unfortunately, this assumption can be problematic for modeling real-world networks, as the microscopic fluctuations from layer to layer are neglected in this limit (see \cref{fig:density_path_shape}). 
In particular, infinite-width predictions are shown to be poor approximations of real networks unless the depth is much less than the width \cite{hanin2019finite,seleznova2020analyzing}.

Impressive achievements of deep networks with billions of parameters 
crystallize %
the importance of understanding extremely large, deep neural networks (DNNs). 
An alternative to the infinite-width paradigm is the infinite-depth-and-width paradigm. 
In this setting, both the network depth $d$ and the width $n$ of each layer are simultaneously scaled to infinity, while their relative ratio $d/n$ remains fixed 
\cite{hanin2019finite,hanin2019products,hu2021random,li2021future,zavatone2021exact,noci2021precise}.
Recent work also explores using $d/n$ as an effective perturbation parameter \cite{yaida2020non,roberts2022principles,zavatone2021asymptotics,hanin2022correlation} or to study concentration bounds in terms of $d/n$ \cite{allen2019convergence,buchanan2021deep}. 
This limit has the distinct advantage of being incredibly accurate at predicting the output distribution for finite size networks at initialization \cite{li2021future} --- a significant improvement over the NNGP theory. 
Furthermore, it has also been shown that there is feature learning in this limit \cite{hanin2019finite}, in contrast to the linear regime of infinite-width limits \cite{lee2019wide}.
Considering the mathematical success of the NNGP techniques, the infinite-depth-and-width limit hints at the possibility of developing an accurate theory for training and generalization. 

An immediate issue of the infinite-depth limit is that this limit predicts that network output becomes degenerate as depth increases: on initialization the network becomes a constant function sending all inputs to the same (random) output \cite{yang2017mean,hayou2021stable,hanin2022correlation}.  
While degenerate outputs are not necessarily an issue in theory, it poses a more serious problem in practice: 
degenerate correlations imply a ``sharp'' input--output Jacobian, and therefore exploding gradients \cite{hanin2018start,hanin2019products}. 
Intuitively, the output is not very sensitive to changes in the input, hence the gradient must be very large in the earlier layers. 

A promising new attack on this problem is to
modify the activation function (``shaping'') to reduce to the effect of degeneracy \cite{martens2021rapid,zhang2022deep}.
In this prior work, extensive experiments show that shaping the activation significantly improves training speed \emph{without the need for normalization layers}.
This method has been proven effective for problems as large as standard ResNets on ImageNet data.
The authors designed several criteria including reducing estimated output correlation, and numerically optimized the shape of activation functions for improved training results. 
However, their deterministic estimation of output correlation using the infinite-width limit leads to a poor approximation of real networks, as the additional randomness has both non-zero mean and heavy skew (see \cref{fig:density_path_shape} right column). 
Furthermore, numerically searching for the activation shape obscures the picture on how shaping should depend on the network depth and width.

In this paper, we address these problems by providing a precise theory of shaped infinite-depth-and-width networks, extending  both the NNGP theories and the activation shaping techniques. 
In particular, we prescribe an exact scaling of the activation function shape as a function of network width $n$ that leads to a non-trivial nonlinear limit. 
By keeping track of microscopic $O(n^{-1/2})$ random fluctuations in each layer of the network, we show that the cumulative effect is described by a stochastic differential equation (SDE) in the limit. 
In contrast to existing infinite-width theory, we are able to characterize the random distribution of the output covariance, which matches closely to simulations of real networks. 
In a similar spirit to how the NNGP theory laid the foundation for studying training and generalization in the infinite-width limit, 
we also see this work as building the mathematical tools for an infinite-depth-and-width theory of training and generalization.

\subsection{Contributions}
Similar to the NNGP approach, we use the fact that the output is Gaussian conditional on the penultimate layer. 
However, unlike in the infinite-width paradigm, the covariance matrix is no longer deterministic in the infinite-depth-and-width limit. 
Our focus in this paper is to study this random covariance matrix. 
Our main contributions are as follows:
\begin{enumerate}[leftmargin=0.125in]
    \itemsep0em 
    \item We introduce the tool of stochastic $\sqrt{n}$-expansions and convergence to SDEs for analyzing the distribution of covariances in DNNs.
    \item For \emph{unshaped} ReLU-like activations, we show that the norm of each layer evolves according to geometric Brownian motion and correlations evolve according to a discrete Markov process. See left column of \cref{fig:density_path_shape} and \cref{sec:unshaped-limits}. 
    \item For both ReLU-like and a large class of smooth activation functions, we derive the Neural Covariance SDE characterizing the distribution of the shaped infinite-depth-and-width limit. See right column of \cref{fig:density_path_shape} and \cref{sec:cov-SDEs-section}. 
    \item We show our prescribed shape scaling is exact, as other rates of scaling leads to either degenerate or linear network limits. See \cref{prop:relu_critical_exponent} and \cref{prop:smooth_critical_exponent}.
    \item For smooth activations, we derive an if-and-only-if condition for exploding/vanishing norms based on properties of the activation function. See \cref{prop:finite_time_explosion} and \cref{sec:discussion}.
    \item We provide simulations to verify theoretical predictions and help interpret properties of real DNNs. See \cref{fig:density_path_shape,fig:blow_up_stable_paths} and supplemental simulations in \cref{sec:app_additional_simulations}.
\end{enumerate}

\section{Limits for Unshaped ReLU-Like Activations}
\label{sec:unshaped-limits}

\begin{table}[t]\label{tab:notations}
\begin{tikzpicture}%
\node[draw,anchor=north east,inner sep=0,outer sep=0,inner sep=3pt] (tab) at (0,0) {
\small\centering
\begin{tabular}{>{\raggedright}p{0.12\textwidth}>{\raggedright}p{0.245\textwidth}>{\raggedright}p{0.145\textwidth}>{\raggedright}p{0.38\textwidth}}
\textbf{Notation} & \textbf{Description} & \textbf{Notation} & \textbf{Description}\tabularnewline[2pt]
$\nin \in \bN$ & Input dimension & $\nout \in \bN$ & Output dimension \tabularnewline
$n \in \bN $ & Hidden layer width & $d \in \bN$ & Number of hidden layers (depth) \tabularnewline
$\vp(\cdot)$ & Base activation 
&
$\vp_{s}(\cdot)$ & Shaped activation
\tabularnewline
$x^\alpha \in \bR^\nin$ & Input for $1\leq \alpha \leq m$  & $W_0 \in \bR^{\nin \times n}$ & Weight matrix at layer 0
\tabularnewline
$\zout^\alpha \in \bR^\nout$ & Network output & \rlap{\smash{$\Wout \in \bR^{n \times \nout}$}} & Weight matrix at final layer 
\tabularnewline[-1.2em]
$z_\ell^\alpha \in \bR^n$ & Neurons (pre-activation) \\ \ \ for layer $1\leq \ell \leq d$
  &  $W_\ell \in \bR^{n \times n}$ & Weight matrix at layer $1 \leq \ell \leq d$\\ \ \ \textbf{All weights initialized iid}
$\sim \cN(0,1)$
\tabularnewline
$\vp_\ell^\alpha \in \bR^n$ & Neurons (post-activation) 
\\ \ \ for layer $1\leq \ell \leq d$
& $c \in \bR$ & Normalizing constant 
\\ \ \ $c \defequal \left( \mathbb{E} \, \varphi(g)^2 \right)^{-1}$ for $g\sim \mathcal{N}(0,1)$
\end{tabular}
};
\node[draw,align=right,anchor=north east,fill=white,outer sep=0] (tabcap) at (0,0) {\scalebox{0.85}{Table 1: Notation}};
\end{tikzpicture}
\end{table}

Using the \textbf{notation in Table~1}, the output of a fully connected feedforward network with $d$ hidden layers of width $n$ on input $x^\alpha$ is defined by vectors of \textbf{pre-activations} $z^\alpha_\ell$ and \textbf{post-activations} $\vp^\alpha_\ell$:
\begin{equation}\label{eq:defn_fcnet}
z_{1}^\alpha \defequal \frac{1}{\sqrt{\nin}}W_0 x^\alpha, \hspace{0.8em}
\vp^\alpha_\ell \defequal \vp(z^\alpha_\ell), \hspace{0.8em}
z_{\ell+1}^\alpha \defequal 
\sqrt{\frac{c}{n}} W_\ell \vp^\alpha_\ell, \hspace{0.8em} %
\zout^\alpha \defequal \sqrt{\frac{c}{n}} \Wout \vp^\alpha_d \,. 
\end{equation}
Note that factors of $\sqrt{c n^{-1}}$ are equivalent to intializing according to the so-called He initialization \cite{he2015delving}. 
We use Greek indices $\alpha,\beta,\ldots$ to denote multiple different inputs. Note that while our results are all stated for fixed width $n$ in each layer, they can be generalized to layer width $n_\ell$ in the limit where all $n_\ell \to \infty$ with $\sum_{\ell = 1}^d {n_\ell}^{-1}$ replacing the role of the depth-to-width ratio $d/n$ \cite{hanin2019products}.

In this section, we analyze \textbf{ReLU-like} activations by which we mean activations which are linear on the negative and positive numbers given respectively by two slopes $s_+$ and $s_-$:
\begin{equation}
\label{eq:ReLU-like}
    \varphi(x) \defequal s_+ \max(x,0) + s_- \min(x,0) = s_+ \varphi_\text{ReLU}(x) - s_- \varphi_\text{ReLU}(-x) \,. 
\end{equation}
These are precisely the  \textbf{positive homogeneous} functions: $\vp(a x)=\abs{a}\vp(x) \, \forall x,a \in \mathbb{R}$.

\subsection{SDE Limits of Markov Chains}

We briefly review the main type of SDE convergence principle used in our main results (see \cref{prop:conv_markov_chain_to_sde} for a more precise version). Let $X_t$, $t\in\mathbb{R^+}$, be a continuous time diffusion process obeying an SDE with drift $b$ and variance $\sigma^2$ as given in \eqref{eq:SDE_convergence}. Suppose that for each $n \in \mathbb{N}$, $Y^n_\ell$ is a discrete time Markov chain $\ell \in \mathbb{N}$ whose increments obey \eqref{eq:SDE_convergence} in terms of the same functions $b,\sigma^2$:
\begin{equation} 
\label{eq:SDE_convergence}
    d X_t = b(X_t) \, dt + \sigma(X_t) \, dB_t \,,
    \hspace{2em} 
    Y^n_{\ell+1} - Y^n_\ell= b(Y^n_\ell)\frac{1}{n} + \sigma(Y^n_\ell) \frac{\xi_\ell}{\sqrt{n}} 
    + O(n^{-3/2}),  
\end{equation}
where $\xi_\ell$ are independent variables with $\mathbb{E}(\xi_\ell)=0, \mathbf{Var}(\xi_\ell)=1$. 
With this setup, under technical conditions described precisely in  \cref{sec:app_markov_chain_convergence_sde}, we have convergence of $Y_\ell$ at $\ell = \floor{tn}$ to $X_t$, or more precisely: with $X^n_t \defequal Y^n_{\floor{tn}}$ we have $X^n \to X$ as $n \to \infty$ in the Skorohod topology. 
In our applications, $n$ is always the width (i.e., number of neurons in each layer) which may appear implicitly and $\ell$ is always the layer number.

\subsection{A Simple SDE: Geometric Brownian Motion Describes \texorpdfstring{$\norm{\vp^\alpha_\ell}^2$}{the Squared Norm of the $\ell$-th Layer}}
\label{sec:geometric_Brownian_motion}

To motivate our approach of SDE limits, we illustrate the method using the example of the squared norm of the $\ell$-th layer, $|\varphi^\alpha_\ell|^2$, where we recall $\varphi^\alpha_\ell = \varphi(z^\alpha_\ell)$. 
For a single fixed input $x^\alpha$ and a ReLU-like activation $\varphi$, the norm of the post-activation neurons $\norm{\vp^\alpha_\ell}^2$ forms a Markov chain in the layer number $\ell$. 
We use the fact that a matrix with iid Gaussian entries applied to any unit vector gives a Gaussian vector of iid $\mathcal{N}(0,1)$ entries. Hence, in each layer, we can define the Gaussian vector $g^\alpha$ as follows, and use \eqref{eq:defn_fcnet} with the positive homogeneity of $\vp$ to write the Markov chain update rule:
\begin{equation}
\label{eq:z_is_cond_Gaussian}
\norm{\vp^\alpha_{\ell+1}}^2 
= \norm{\vp^\alpha_{\ell}}^2  \frac{1}{n} \sum_{i=1}^{n} c \vp( g^\alpha_{i} )^2 
 , \,\,\,
\text{ where }
g^\alpha \defequal W_\ell  \frac{\varphi^\alpha_\ell}{ |\varphi^\alpha_\ell| } 
\overset{d}{=} \mathcal{N}( 0, I_n ) \,. 
\end{equation} 
At this point, the infinite-width approach applies the law of large numbers (LLN) to conclude $\displaystyle\lim_{n \to \infty} \norm{\vp^\alpha_{\ell+1}}^2 = \norm{\vp^\alpha_\ell}^2 \mathbb{E}[c \vp^2(g) ] = \norm{\vp^\alpha_\ell}^2 \cdot 1$ a.s. by definition of $c$. 
However, the LLN cannot be applied when depth $d$ is diverging with $n$, as the cumulative effect of the fluctuations over $d$ layers does not vanish! 
Instead, we keep track of the $O(1/\sqrt{n})$ fluctuations in each layer by introducing the zero mean finite variance random variable $R^{\alpha\alpha}_\ell \defequal \frac{1}{\sqrt{n}} \sum_{i=1}^n \left(c \varphi(g^\alpha_{i})^2 - 1\right)$. 
This allows us to rewrite this Markov chain update rule as 
\begin{equation}
\label{eq:norm_z_CLT}
\norm{ \vp^\alpha_{\ell+1} }^2 
= 
\norm{ \vp^\alpha_{\ell}}^2 \left( 1 + \frac{ 1 }{\sqrt{n}} R^{\alpha\alpha}_\ell \right) \,, 
\end{equation} 
\begin{figure}[t]
\centering
\includegraphics[width=0.95\textwidth]{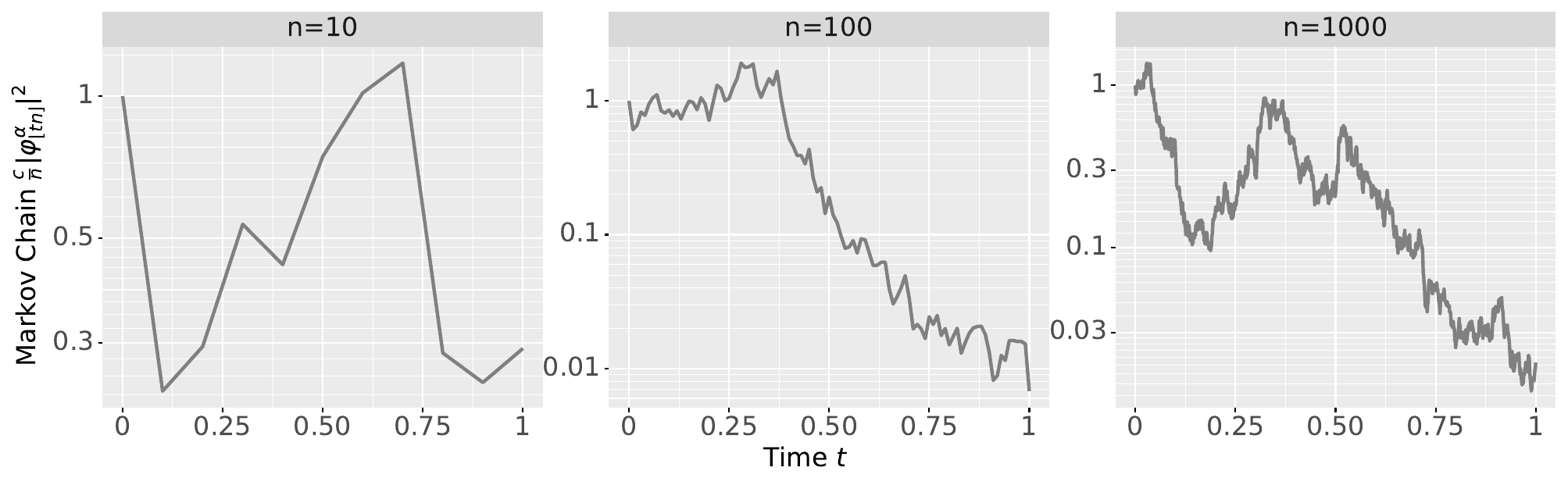}
\caption{
A sample path of the geometric random walk from \cref{eq:norm_z_CLT} converging to geometric Brownian motion as $n$ increases. 
}
\label{fig:rw}
\end{figure}
which allows us to see that the Markov chain $Y^n_\ell = \frac{c}{n} | \vp^\alpha_\ell |^2$ is now in the form of \eqref{eq:SDE_convergence} with $Y^n_0 = \frac{1}{\nin} |x^\alpha|^2$, 
$b(Y) \equiv 0, \sigma^2(Y) = \mathbf{Var}(R^{\alpha\alpha}_\ell) Y^2 = \mathbf{Var}( c \varphi(g)^2 ) Y^2$. 
Consequently, we have that the squared norm Markov chain converges to a geometric Brownian motion $d X_t = \sigma X_t dB_t$, or more precisely 
\begin{equation}
\label{eq:geometric_brownian_motion}
    \lim_{n \to \infty} 
    \frac{c}{n} \norm{ \vp^\alpha_{\floor{t n}} }^2 
    = X_t
    \dequal 
    e^{
    \mathcal{N}(-\frac{\sigma^2}{2}t, \sigma^2 t) } \,,
\end{equation}
where the convergence is in the Skorohod topology (see  \cref{sec:app_markov_chain_convergence_sde}).
When $\vp$ is the ReLU function ($s_+=1,s_-=0$), we have $c=2$ and $\sigma^2=5$, which recovers known results in \cite{hanin2019products,li2021future,zavatone2021exact,noci2021precise}. 
We remark again this simple Markov chain example illustrates the main technique we use in later sections to establish SDE convergence for shaped networks in \cref{sec:cov-SDEs-section}.

\subsection{Non-SDE Markov Chains: the Gram \texorpdfstring{Matrix $\left\langle \vp^\alpha_{\ell}, \vp^\beta_{\ell} \right\rangle$}{Matrix} and \texorpdfstring{Correlation $\rho^{\alpha \beta}_\ell 
$}{Correlation}}

We can generalize \cref{sec:geometric_Brownian_motion} to a collection of $m$ inputs $\{x^\alpha\}_{\alpha=1}^m$ by looking at the entire Gram matrix $[\langle \varphi^\alpha_\ell, \varphi^\beta_\ell \rangle]_{\alpha,\beta=1}^m$, where we again recall $\varphi^\alpha_\ell = \varphi(z^\alpha_\ell)$. 
We note that the convergence of Markov chains to SDEs in \cref{eq:SDE_convergence} can be generalized to $Y^n_\ell \in \mathbb{R}^N$ by considering $\cov (\xi_\ell) = I_N$, $b:\mathbb{R}^N\to\mathbb{R}^N$, and $\sigma: \mathbb{R}^N \to \mathbb{R}^{N\times N}$. 
The Gram matrix is of particular interest because the neurons in any layer are conditionally Gaussian \textbf{when conditioned on the previous layer}, with covariance matrix proportional to the Gram matrix:
\begin{equation}
\label{eq:conditional_gaussian}
\begin{aligned}
    \left. 
    [z^\alpha_{\ell+1}]_{\alpha=1}^m 
    \right| 
    \mathcal{F}_\ell
    &\overset{d}{=} 
    \cN \left( 0, 
    \frac{c}{n} 
    [\langle \varphi^\alpha_\ell, \varphi^\beta_\ell \rangle]_{\alpha,\beta=1}^m
    \otimes I_n  
    \right) \,,  
    \\ 
    \left. 
    [\zout^\alpha]_{\alpha=1}^m 
    \right| 
    \mathcal{F}_d
    &\overset{d}{=} 
    \cN \left( 0, 
    \frac{c}{n} 
    [\langle \varphi^\alpha_d, \varphi^\beta_d \rangle]_{\alpha,\beta=1}^m
    \otimes I_{\nout} 
    \right) \,, 
\end{aligned}
\end{equation}
where $\mathcal{F}_\ell$ denotes the sigma-algebra generated by the $\ell$-th layer $[z^\alpha_{\ell}]_{\alpha=1}^m$, 
and $\otimes$ denotes the Kronecker product (here indicating conditionally independent entries in each vector). %
With this property in mind, we will introduce $\mathbb{E}_\ell [\,\cdot\,] \defequal \mathbb{E}[\,\cdot\,|\mathcal{F}_\ell]$ to denote the conditional expectation, and $\var_\ell(\,\cdot\,) \,, \cov_\ell(\,\cdot\,)$ similarly to denote the conditional variance and covariance. 
If we define $g^\alpha$ as in \eqref{eq:z_is_cond_Gaussian}, we see that the $g^\alpha$ are all marginally $\cN(0,I_n)$. 
Similar to \eqref{eq:z_is_cond_Gaussian}, we can write the update rule for the $\alpha,\beta$-entry of the Gram matrix:
\begin{equation}
\langle \vp^\alpha_{\ell+1}, \vp^\beta_{\ell+1} \rangle = |\vp_\ell^\alpha|  |\vp_\ell^\beta| \frac{1}{n}\sum_{i=1}^{n} c \vp(g^\alpha_i) \vp(g^\beta_i) \,, 
\end{equation}
Just as we did in \eqref{eq:norm_z_CLT}, we can define $R^{\alpha\beta}_\ell \defequal \frac{1}{\sqrt{n}} \sum_{i=1}^n c \varphi(g^\alpha_{i})\varphi(g^\beta_{i}) - \mathbb{E}_\ell[c \varphi(g^\alpha_{i})\varphi(g^\beta_{i})]$ and write  
\begin{equation}
\label{eq:Gram_matrix_update}
\langle \vp^\alpha_{\ell+1}, \vp^\beta_{\ell+1} \rangle  = |\vp_\ell^\alpha|  |\vp_\ell^\beta| \left( \mathbb{E}_\ell\left[ c \vp(g^\alpha_i) \vp(g^\beta_i)\right] + \frac{1}{\sqrt{n}} R^{\alpha \beta}_\ell \right) \,, 
\end{equation}
where $R_\ell^{\alpha \beta}$ are mean zero with covariance $\cov_\ell[ R_\ell^{\alpha \beta}, R_\ell^{\gamma \delta} ] = \cov_\ell[c\vp(g^\alpha)\vp(g^\beta),c\vp(g^\gamma)\vp(g^\delta) ]$. 
(By the Central Limit Theorem, $R^{\alpha \beta}_\ell$ will be approximately Gaussian for large $n$.) 

However, unlike the simple single-data-point case from Section $\ref{sec:geometric_Brownian_motion}$, \textbf{we do not have convergence to a continuous time SDE.} 
This is because the differences $\langle \vp^\alpha_{\ell+1}, \vp^\beta_{\ell+1} \rangle - \langle \vp^\alpha_{\ell}, \vp^\beta_{\ell} \rangle \nrightarrow 0$ as $n \to \infty$. 
Instead, \eqref{eq:Gram_matrix_update} is a discrete recursion update with additive noise of the form $Y^n_{\ell+1} = f(Y^n_\ell) + \frac{1}{\sqrt{n}}\xi$ for some function $f$, 
and consequently $Y^n_{\ell+1}-Y^n_\ell$ does not vanish as $n\to\infty$.

For a clarifying example, we can consider the one-dimensional Markov chain of hidden layer correlations. 
More precisely, we can define 
$\rho^{\alpha \beta}_\ell = \langle \vp^\alpha_{\ell}, \vp^\beta_{\ell} \rangle / |\vp^\alpha_\ell| |\vp^\beta_\ell|$, 
which we observe can be extracted from the entries of the Gram matrix. 
In fact, we can write down an approximate recursion update for $\rh^{\alpha\beta}_\ell$ (see \cref{sec:app_unshaped_relu_markov_chain} and  \cref{prop:unshaped_relu_corr} for details):
\begin{equation}
\label{eq:corr_update}
	\rho^{\alpha\beta}_{\ell+1} 
	\approx c K_1(\rho^{\alpha\beta}_\ell) + \frac{1}{n} \mu_{\text{ReLU}}(\rho^{\alpha\beta}_\ell) 
		+ \frac{\xi_\ell}{\sqrt{n}} \sigma_{\text{ReLU}}(\rho^{\alpha\beta}_\ell) 
		\,, \quad 
    \rho^{\alpha\beta}_0 = \frac{\langle x^\alpha, x^\beta \rangle}{\nin} \,, 
\end{equation}
where $K_1(\rho) \defequal \mathbb{E} \, [ \varphi(g) \varphi(g \rho + w \sqrt{1-\rho^2})]$ for $g,w$ iid $\cN(0,1)$ random variables, 
and $\xi_\ell$ are iid $N(0,1)$. 
For the ReLU case, $c=2$ and $cK_1(\rho)= (\sqrt{1-\rho^2} + \rho \arccos(-\rho) ) / \pi $ was first calculated in \cite{cho2009kernel}. 
In fact, we can observe that as $n\to\infty$, $\rho^{\alpha\beta}_{\lfloor tn \rfloor}$ converges to the fixed point of $cK_1(\cdot)$ at $\rho=1$ for all $t>0$. 
\textbf{We note this limiting behaviour cannot be described by an SDE}, as the solution must jump from the initial condition to the fixed point at $t=0$. 

Despite not having an SDE limit, we observe that the approximate Markov chain \cref{eq:corr_update} already provides a much better approximation to finite size networks compared to the infinite-width theory (see left column of \cref{fig:density_path_shape}). 
This is because the infinite-width approach discards the terms in \eqref{eq:corr_update} that vanish as $n \to \infty$ and consider only the update $\rho^{\alpha\beta}_{\ell+1} = cK_1(\rho^{\alpha\beta}_{\ell})$. 
Analysis of this deterministic equation leads to the prediction that $\rho^{\alpha\beta}_\ell = 1 - O( \ell^{-2} )$ for $\ell \gg 1$ (see (4.8) in \cite{hanin2022correlation} and a new bound in Appendix~\ref{sec:app_independent_interest}). 

Furthermore, we observe that in this case, the microscopic $O(n^{-1})$ and $O(n^{-1/2})$ terms in \eqref{eq:corr_update} accumulate to macroscopic differences! 
For the examples in \cref{fig:density_path_shape}, we see their net effect is that $\rho^{\alpha\beta}_\ell \to 1$ {faster} than the infinite-width prediction. 
Heuristically, the reason for this discrepancy is due to $\sigma_\text{ReLU}(\rho) \to 0$ as $\rho \to 1$. 
This means that the randomness can push $\rho^{\alpha\beta}_\ell$ closer to $1$, but becomes ``trapped'' when $\rho^{\alpha\beta}_\ell$ is close to 1 because $\sigma_\text{ReLU}$ is so small here. 
In the next section, we will see that we are just one step away from achieving limiting SDEs.

\section{Neural Covariance SDEs: Shaped Infinite-Depth-and-Width Limit}
\label{sec:cov-SDEs-section}

\begin{figure}[t]
\centering
\includegraphics[width=0.95\textwidth]{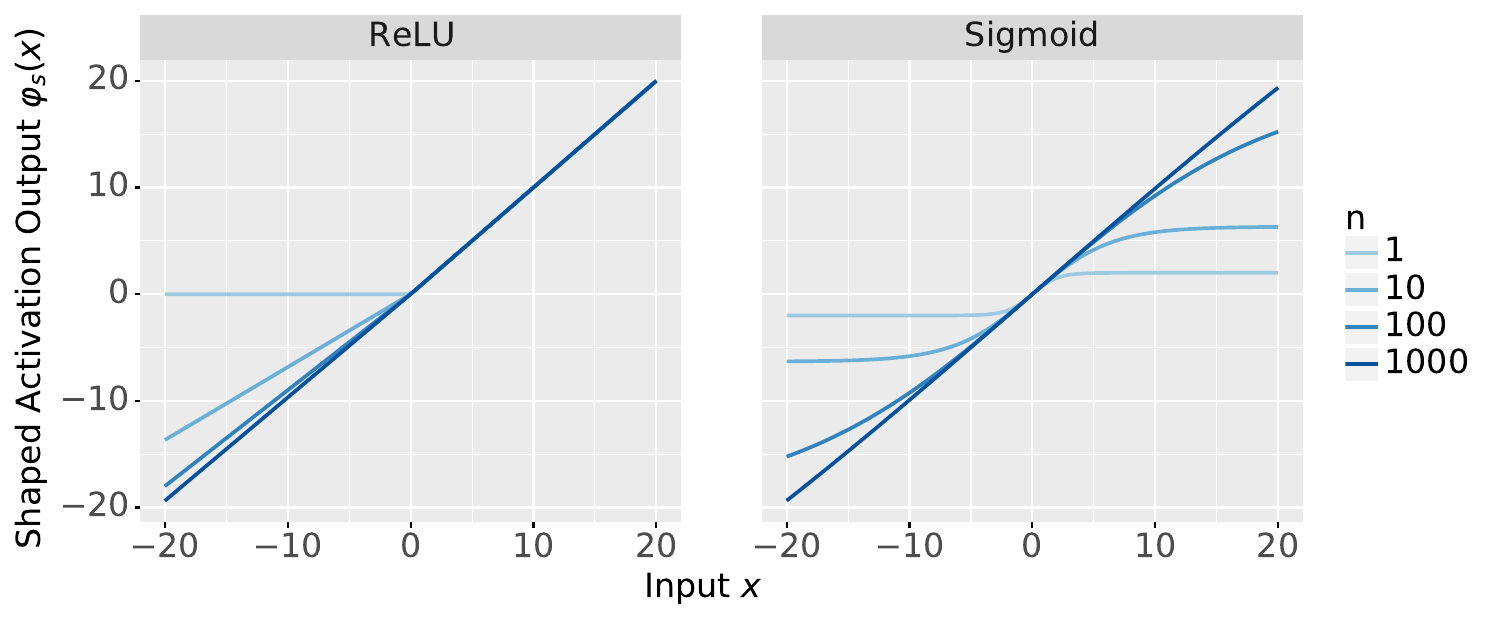}
\caption{
Shaping of activation functions from \cref{def:relu-shaping,def:smooth-shaping} as $n$ increases. 
Here we chose $c_+ = 0, c_- = -1, a = 1$, and sigmoid activation is centered at $x_0 = 0$, i.e. $\varphi(x) = \frac{4}{1 + e^{-x}} - 2$. 
}
\label{fig:shape}
\end{figure}

In this section, we follow the ideas of \cite{martens2021rapid,zhang2022deep} to \emph{reshape} the activation function $\varphi$. Reshaping means to replace the base activation function $\varphi$ in \eqref{eq:defn_fcnet}
with $\varphi_s$ that depends on width $n$. 
We will also replace the normalizing constant $c = \left( \mathbb{E} \, \varphi_s(g)^2 \right)^{-1}$ for $g\sim\cN(0,1)$. 
Specifically, we will choose $\varphi_s$ to depend on $n$ such that in the limit as $n\to\infty$, 
we have that $\varphi$ is approximately an identity function,  $\varphi_s \to \text{Id}$ 
(see \cref{fig:shape}). 
Recalling from \cref{eq:conditional_gaussian} that the output is conditionally Gaussian with covariance determined by the Gram matrix $[\langle \varphi^\alpha_\ell, \varphi^\beta_\ell \rangle]_{\alpha,\beta=1}^m$, %
therefore we recover a complete characterization by describing the random covariance matrix.

\subsection{Neural Covariance SDE for Shaped ReLU-Like Activations}

\begin{definition}
\label{def:relu-shaping}
We shape the ReLU-like activation $\varphi_s(x) \defequal s_+ \max(x,0) + s_- \min(x,0)$, 
by setting the slopes to depend on $n$ according to $s_\pm \defequal 1 + \frac{c_\pm}{\sqrt{n}}$ for some given constants $c_+,c_- \in \mathbb{R}$. 
We will also set $c = \left( \mathbb{E} \, \varphi_s(g)^2 \right)^{-1}$ for $g\sim\cN(0,1)$.
\end{definition}

We will show that with shaping of \cref{def:relu-shaping}, one gets non-trivial SDEs that describe the covariance (\cref{thm:joint_output_m}) and correlations (\cref{thm:relu_corr}) of the network. The precise scaling is shown to be the critical scaling for a non-trivial limit in \cref{prop:relu_critical_exponent}.
All proofs for results in this section appear in \cref{sec:app_relu_proofs}. 

\emph{Remark.} Note that in the statement of our theorems, we abuse notation and use the same letter to denote the pre-limit Markov chain and the limiting SDE. For example, in \cref{thm:joint_output_m} we use $V_\ell$ for the covariance at layer $\ell$ and $V_t$ to denote the limiting SDE at time $t$. %

\begin{thm}
[Covariance SDE, ReLU]
\label{thm:joint_output_m}
Let $V^{\alpha\beta}_\ell \defequal \frac{c}{n} \langle \varphi^\alpha_\ell, \varphi^\beta_\ell \rangle$, 
and define $V_\ell \defequal [ V^{\alpha\beta}_\ell ]_{1\leq \alpha \leq \beta=m}$ to be the upper triangular entries thought of as a vector in $\mathbb{R}^{m(m+1)/2}$.
Then, with $s_\pm = 1 + \frac{c_\pm}{\sqrt{n}}$ as in \cref{def:relu-shaping}, in the limit as $n\to\infty, \frac{d}{n} \to T$, the interpolated process $V_{\lfloor tn \rfloor}$ converges in distribution 
in the Skorohod topology of $D_{\bR_+, \bR^{m(m+1)/2}}$
to the solution of the SDE 
\begin{equation}
    d V_t = b( V_t ) \, dt 
    + \Sigma(V_t)^{1/2} \, dB_t \,, 
    \quad 
    V_0 = \left[ \frac{1}{\nin} 
        \langle x^\alpha, x^\beta \rangle 
    \right]_{1 \leq \alpha \leq \beta \leq m} \,, 
\end{equation}
where 
$\nu(\rho) \defequal \frac{(c_+ - c_-)^2}{2\pi} \left( \sqrt{1-\rho^2} - \rho \arccos \rho \right), 
\rho^{\alpha\beta}_t \defequal  \frac{V^{\alpha\beta}_t}{ \sqrt{ V^{\alpha\alpha}_t V^{\beta\beta}_t } }$ 
\begin{equation}
    b( V_t ) 
    = \left[ \nu\left( 
    \rho^{\alpha\beta}_t
    \right) 
    \sqrt{V^{\alpha\alpha}_t V^{\beta\beta}_t} \right]_{1\leq \alpha \leq \beta \leq m} 
    \text{,\quad and \quad }
    \Sigma( V_t )
    = \left[ 
        V^{\alpha\gamma}_t V^{\beta\delta}_t 
        + V^{\alpha\delta}_t V^{\beta\gamma}_t 
    \right]_{\alpha\leq\beta, \gamma\leq\delta}
    \,. 
\end{equation}
Furthermore, the output distribution can be described conditional on $V_T$ evaluated at final time $T$
\begin{equation}
    \left[ \zout^\alpha \right]_{\alpha=1}^m | {V_T}
    \overset{d}{=}
    \cN\left( 0 , [ V^{\alpha\beta}_T ]_{\alpha,\beta=1}^m 
    \right) \,. 
\end{equation}
\end{thm}

Here we remark that $\nu(1) = 0$, and therefore the drift component of diagonal entries ($V^{\alpha\alpha}_t$) are zero, as they are geometric Brownian motion. 
However, we emphasize that the $m$-point joint output distribution is \emph{not} characterized by the marginal for each of the pairs, as the output $\zout^\alpha$ is \emph{not} Gaussian. 
In particular, we observe the diffusion matrix entry corresponding to $V^{\alpha\beta}_t,V^{\gamma\delta}_t$ involves other processes $V^{\alpha\gamma}_t,V^{\beta\delta}_t,V^{\alpha\delta}_t,V^{\beta\gamma}_t$! 
\emph{This implies that the Neural Covariance SDE limit cannot be described by a kernel, unlike stacking random features or NNGP.}

That being said, it is still instructive to study the marginal for a pair of data points. 
More specifically, it turns out in the generalized ReLU case, we can derive the marginal SDE for the correlation process.

\begin{thm}
[Correlation SDE, ReLU]
\label{thm:relu_corr}
Let $\rho^{\alpha\beta}_\ell \defequal 
\frac{\langle \varphi_\ell^\alpha, \varphi_\ell^\beta \rangle}{
|\varphi_\ell^\alpha| \, |\varphi_\ell^\beta|}$, 
where $\varphi^\alpha_\ell \defequal \varphi_s(z^\alpha_\ell)$. 
In the limit as $n\to\infty$ and $s_\pm = 1 + \frac{c_\pm}{\sqrt{n}}$, the interpolated process $\rho^{\alpha\beta}_{\lfloor tn \rfloor}$ converges in distribution to the solution of the following SDE in the Skorohod topology of $D_{\bR_+, \bR}$ 
\begin{equation}
\label{eq:relu_corr_sde}
    d \rho^{\alpha\beta}_t 
    = \left[ \nu( \rho^{\alpha\beta}_t ) 
        + \mu( \rho^{\alpha\beta}_t ) \right] 
        \, dt 
        + \sigma( \rho^{\alpha\beta}_t ) \, dB_t \,, 
    \quad 
    \rho^{\alpha\beta}_0 = \frac{\langle x^\alpha, x^\beta \rangle}{ |x^\alpha| \, |x^\beta| } \,, 
\end{equation}
where 
\begin{equation}
\label{eq:greeks}
	\nu(\rho) 
	= 
		\frac{ (c_+ - c_-)^2 }{ 2\pi } 
		\left[ 
		\sqrt{1 - \rho^2} - \arccos(\rho) \rho 
		\right] 
		\,, \quad 
	\mu(\rho) 
	= - \frac{1}{2} \rho (1 - \rho^2)
		\,, \quad 
	\sigma(\rho) 
	= 
		1-\rho^2 \,. 
\end{equation}
\end{thm}

To help interpret the SDE, we observe that $\mu$ and $\sigma$ are entirely independent of the activation function. 
In other words, these terms will be present in this limit even for linear networks. 
At the same time, $\nu$ describes the influence of the shaped activation function in this limit. 
\cite{zhang2022deep} has derived a related ordinary differential equation (ODE) of $d\rho_t = \nu(\rho_t) \, dt$ in the sequential limit of $n\to\infty$ then $d\to\infty$, where the activation is shaped depending on depth. 
Here we also note that $\nu(\rho)$ is closely related to the $J_1$ function derived in \cite{cho2009kernel}. 
See \cref{subsec:app_joint_corr_sde} for the $m$-point joint version of the correlation SDE, 
and \cref{sec:app_additional_simulations} for an empirical measure of convergence in the Kolmogorov--Smirnov distance.

It is also possible to transform this SDE via It\^o's Lemma for potentially more interpretability, 
such as the angle form $\theta^{\alpha\beta}_t = \arccos(\rho^{\alpha\beta}_t)$ 
\begin{equation*}
    d \theta_t^{\alpha\beta} 
	= \frac{(c_+ - c_-)^2}{2\pi} \left[ 
		\theta_t^{\alpha\beta} 
		\cot \theta_t^{\alpha\beta} - 1 
		\right] \, dt 
	+ \sin \theta_t^{\alpha\beta} \, dB_t \,, 
\end{equation*}
where for $\theta \approx 0$ we have that $\theta \cot \theta - 1 \approx \frac{-\theta^2}{3}$ and $\sin \theta \approx \theta$, which converges rapidly to $0$. 

One immediate consequence of the correlation SDE is that we can show the $n^{-1/2}$ scaling in \cref{def:relu-shaping} is the only case where the limit is neither degenerate nor a linear network. 

\begin{prop}
[Critical Exponent, ReLU]
\label{prop:relu_critical_exponent}
Let $\rho^{\alpha\beta}_\ell \defequal 
\frac{\langle \varphi_\ell^\alpha, \varphi_\ell^\beta \rangle}{
|\varphi_\ell^\alpha| \, |\varphi_\ell^\beta|}$, 
where $\varphi^\alpha_\ell \defequal \varphi_s(z^\alpha_\ell)$. 
Consider the limit $n\to\infty$ and $s_\pm = 1 + \frac{c_\pm}{n^p}$ for some $p\geq 0$. 
Then depending on the value of $p$, the interpolated process $\rho^{\alpha\beta}_{\lfloor tn \rfloor}$ converges in distribution w.r.t. the Skorohod topology of $D_{\bR_+, \bR}$ to 
\begin{enumerate}[label=(\roman*)]
    \item {the degenerate limit:} $\rho^{\alpha\beta}_t = 1$ for all $t>0$, if $0\leq p < \frac{1}{2}$, and $c_+\neq c_-$, 
    \item the critical limit: the SDE from \cref{thm:relu_corr}, if $p=\frac{1}{2}$, 
    \item {the linear network limit:} if $p > \frac{1}{2}$ , the following SDE, with $\mu,\sigma$ as defined in \eqref{eq:greeks},
    
    \begin{equation}
    d \rho^{\alpha\beta}_t 
    = 
        \mu( \rho^{\alpha\beta}_t ) 
        \, dt 
        + \sigma( \rho^{\alpha\beta}_t ) \, dB_t \,, 
    \quad 
    \rho^{\alpha\beta}_0 = \frac{\langle x^\alpha, x^\beta \rangle}{ |x^\alpha| \, |x^\beta| } \,. 
    \end{equation} 
\end{enumerate}
\end{prop}

Here we remark that the unshaped network case ($p=0$) is contained by the above in case (i). 
At the same time, we observe that case (iii) is equivalent to the correlation SDE in \cref{thm:relu_corr} except with $\nu = 0$. 
In particular, we observe this limit is also reached when $c_+ = c_-$, which implies $\varphi_s(x) = s_+ x$ is linear, which is the reason we call this the linear network limit. 
Furthermore, without much additional work, the same argument also implies the joint covariance SDE also loses the drift component, i.e., $d V_t = \Sigma( V_t )^{1/2} \, dB_t$.

\subsection{Neural Covariance SDE for Shaped Smooth Activations}

In this section, we consider smooth activation functions and derive a similar covariance SDE. 
All the proofs for results in this section can be found in \cref{sec:app_smooth_proofs}. 

\begin{assumption}
\label{asm:reg}
$\varphi \in C^4(\bR)$, 
$\varphi(0) = 0, \varphi'(0) = 1$, and 
$|\varphi^{(4)}(x)| \leq C(1 + |x|^p)$ for some $C,p > 0$. 
\end{assumption}
We note that for any non-constant function $\sigma \in C^1(\bR)$ and $x_0 \in \bR$ such that $\sigma'(x_0) \neq 0$, we can always define $\varphi(x) \defequal \frac{\sigma(x + x_0) - \sigma(x_0)}{ \sigma'(x_0) }$ such that it satisfies $\varphi(0) = 0, \varphi'(0)=1$. 
The choice of $x_0$ will be discussed further in \cref{sec:discussion}. 
The fourth derivative growth condition is used to control the Taylor remainder term in expectation, but any control over the remainder will suffice. 

Following the ideas of \cite{martens2021rapid}, we consider the following shaping of a smooth activation function.

\begin{definition} 
\label{def:smooth-shaping}
For some constant $a>0$, we set $\varphi_s(x) \defequal s \varphi \left( \frac{x}{s} \right)$ with $s = a \sqrt{n}$, 
and $c = \left( \mathbb{E} \, \varphi_s(g)^2 \right)^{-1}$ for $g\sim\cN(0,1)$.
\end{definition}

Observe that in the limit $n\to\infty$, we will achieve that $\varphi_s \to \text{Id}$ as desired. 
We also observe that the shaping factor $s$ outside the activation cancels out with the next layer's $\frac{1}{s}$ factor, therefore it is equivalent shape the entire network. 
More precisely, if we view $\zout$ as an input-output map $f:\mathbb{R}^{\nin} \to \mathbb{R}^{\nout}$ of an unshaped network, then shaping the smooth activation functions is equivalent to the modification $s f\left( \frac{x}{s} \right)$.\footnote[1]{We want to thank Boris Hanin for observing this equivalent parameterization.}

In this regime, we can similarly characterize the joint output distribution, \emph{however the limiting SDEs are not always well behaved.} 
In particular, they can have finite time explosions as described by the Feller test for explosions \cite[Theorem 5.5.29]{karatzas2012brownian}. 
Here the SDE in \cref{prop:finite_time_explosion} is exactly the $V^{\alpha\alpha}_t$ marginal of the Neural Covariance SDE, with the parameter $b$ determined by the activation function $\vp$ and controls whether or not finite time explosions happen (see \cref{eq:smooth_V_diag_marginal}).

\begin{prop}
[Finite Time Explosion]
\label{prop:finite_time_explosion}
Let $X_t \in \bR_+$ be a solution to the following SDE 
\begin{equation}
    dX_t = b X_t (X_t - 1) \, dt + \sqrt{2} X_t \, dB_t \,, \quad 
    X_0 = x_0 > 0 \,, b \in \mathbb{R} \,. 
\end{equation}

Let $\ta^\ast = \sup_{M > 0} \inf\{t :X_t \geq M \text{ or } X_t \leq M^{-1} \}$ be the explosion time, and we say $X_t$ has a finite time explosion if $\ta^\ast < \infty$. 
For this equation, $\mathbb{P}[\ta^\ast = \infty] = 1$ if and only if $b \leq 0$. 
\end{prop}

Technically speaking, the main culprit behind finite time explosions is the non-Lipschitzness of the drift coefficient. This issue requires us to weaken the sense of convergence in this section; the ordinary convergence in the Skorohod topology is in general not true when the diffusion has finite time explosions. A weakened type of convergence is the best we can hope for.  
To this goal, we introduce the following definition. 

\begin{definition}
\label{def:local_convergence}
We say a sequence of processes $X^n$ \textbf{converge locally} to $X$ in the Skorohod topology if for any $r>0$, we define the following stopping times 
\begin{equation}
\label{eq:stopping_time_defn}
\tau^n \defequal \left\{ t \geq 0 : | X^n_t | \geq r \right\} \,, \quad 
\tau \defequal \left\{ t \geq 0 : | X_t | \geq r \right\} \,, 
\end{equation}
and we have that $X^n_{t \wedge \tau^n}$ converge to $X_{t \wedge \tau}$ in the Skorohod topology. 
\end{definition}

This weakened sense of convergence essentially constrains the processes $X^n,X$ in a bounded set by adding an absorbing boundary condition. 
Not only do these stopping times rule out explosions, the drift coefficient is now also Lipschitz on a compact set. 
With this notion of convergence, we can now state a precise Neural Covariance SDE result for general smooth activation functions.

\begin{thm}
[Covariance SDE, Smooth]
\label{thm:smooth_joint_output_m}
Let $\varphi$ satisfy \cref{asm:reg}, 
$V^{\alpha\beta}_\ell \defequal \frac{c}{n} \langle \varphi^\alpha_\ell, \varphi^\beta_\ell \rangle$ 
where $\varphi^\alpha_\ell = \varphi_s(z^\alpha_\ell)$, 
and define $V_\ell \defequal [ V^{\alpha\beta}_\ell ]_{1\leq \alpha \leq \beta=m}$ to be the upper triangular entries thought of as a vector in $\mathbb{R}^{m(m+1)/2}$. 
Then, with $s = a\sqrt{n}$ as in \cref{def:smooth-shaping}, in the limit as $n\to\infty, \frac{d}{n} \to T$, the interpolated process $V_{\lfloor tn \rfloor}$ converges locally in distribution to the solution of the following SDE in the Skorohod topology of $D_{\bR_+, \bR^{m(m+1)/2}}$
\begin{equation}
\label{eq:smooth_cov_sde}
    d V_t = b( V_t ) \, dt 
    + \Sigma(V_t)^{1/2} \, dB_t \,, 
    \quad 
    V_0 = \left[ \frac{1}{\nin} 
        \langle x^\alpha, x^\beta \rangle 
    \right]_{1 \leq \alpha \leq \beta \leq m} \,, 
\end{equation}
where $\Sigma(V_t)$ is the same as \cref{thm:joint_output_m} and 
\begin{equation}
    b^{\alpha \beta}( V_t ) 
    = %
    \frac{\varphi''(0)^2}{4a^2} 
    \left( 
    V^{\alpha\alpha}_t V^{\beta\beta}_t 
    + V^{\alpha\beta}_t ( 2 V^{\alpha\beta}_t - 3 )
    \right) 
    + 
    \frac{\varphi'''(0)}{2a^2} V^{\alpha\beta}_t 
    ( V^{\alpha\alpha}_t + V^{\beta\beta}_t - 2 ) 
    \,. 
\end{equation}

Furthermore, if $V_T$ is finite, then the output distribution can be described conditional on $V_T$ as 
\begin{equation}
    \left[ \zout^\alpha \right]_{\alpha=1}^m | {V_T}
    \overset{d}{=}
    \cN\left( 0  , [ V^{\alpha\beta}_T ]_{\alpha,\beta=1}^m 
    \right) \,, 
\end{equation}
and otherwise the distribution of $[\zout^\alpha]_{\alpha=1}^m$ is undefined. 
\end{thm}

We also have a similar critical scaling result for general smooth activations. 

\begin{prop}
[Critical Exponent, Smooth]
\label{prop:smooth_critical_exponent}
Let $\varphi$ satisfy \cref{asm:reg}, 
$V^{\alpha\beta}_\ell \defequal \frac{c}{n} \langle \varphi^\alpha_\ell, \varphi^\beta_\ell \rangle$ 
where $\varphi^\alpha_\ell = \varphi_s(z^\alpha_\ell)$ with $s = a n^p$ for some $p > 0$,  
and define $V_\ell \defequal [ V^{\alpha\beta}_\ell ]_{1\leq \alpha \leq \beta=m}$ to be the upper triangular entries thought of as a vector. 
Then in the limit as $n\to\infty, \frac{d}{n} \to T$, the interpolated process $V_{\floor{tn}}$ converges locally in distribution w.r.t. the Skorohod topology of $D_{\bR_+, \bR^{m(m+1)/2}}$ to $V$, which depending on the value of $p$ is 
\begin{enumerate}[label=(\roman*)]
    \item {the degenerate limit:} if $0<p<\frac{1}{2}$ 
    \begin{equation}
        \begin{cases}
            V^{\alpha\alpha}_{t} = 0 \text{ or } \infty , 
            & 
            \text{ if } \frac{3}{4} \varphi''(0)^2 + \varphi'''(0) > 0 \text{ and } V^{\alpha\alpha}_0 \neq 0 \,, \\ 
            V^{\alpha\beta}_{t} = \text{const.} \,, 
            & 
            \text{ if } \frac{3}{4} \varphi''(0)^2 + \varphi'''(0) \leq 0 
            \,, \\ 
        \end{cases}
    \end{equation}
    for all $t > 0$ and $1\leq \alpha \leq \beta \leq m$, 
    \item the critical limit: the solution of the SDE from \cref{thm:smooth_joint_output_m}, if $p=\frac{1}{2}$, 
    \item {the linear network limit:} the stopped solution to the SDE $d V_t = \Sigma(V_t) \, dB_t$
    with coefficient $\Sigma$ defined in \cref{thm:relu_corr}, 
    if $p > \frac{1}{2}$. 
\end{enumerate}
\end{prop}

Here we observe that in case (i) when $\frac{3}{4} \varphi''(0)^2 + \varphi'''(0) \leq 0$, we also have a constant (in time) correlation $\rho^{\alpha\beta}_t$ similar to the ReLU case in \cref{prop:relu_critical_exponent}, however in this case $\rho^{\alpha\beta}_t$ is not necessarily equal to $1$. 
At the same time, the linear network limit in case (iii) also has the same covariance SDE as \cref{prop:relu_critical_exponent}.

\section{Consequences, Discussion, and Future Directions}
\label{sec:discussion}

So far, we have derived the Neural Covariance SDE. 
Analysis of this SDE reveals important behaviour of the network on initialization. 
Here we lay out one concrete example and provide some discussion and future directions.

\begin{figure}[t]
\centering
\begin{tikzpicture}[x=0.5\textwidth,y=0.5\textwidth, %
  nplot/.style={anchor=north west}, %
  ncap/.style={anchor=south west}, %
  caption/.style={align=left,font=\normalsize,scale=0.8}, %
  ]

\node[ncap] (figacap) at (0.35,0) [caption] {
(a) Unstable Centre $x_0=0$
 };
\node[nplot] (figa) at (0,0)  {{\includegraphics[width=0.5\textwidth]{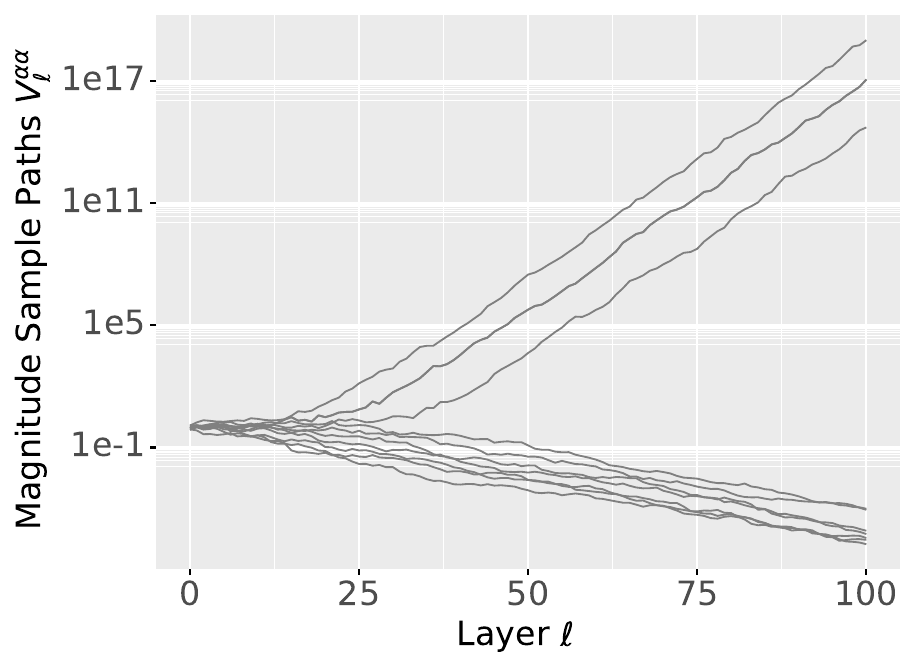}}};  %

\node[ncap] (figbcap) at (1.35,0) [caption] {
(b) Stable Center $x_0 = \log 2$
 };
\node[nplot] (figb) at (1,0) {{\includegraphics[width=0.5\textwidth]{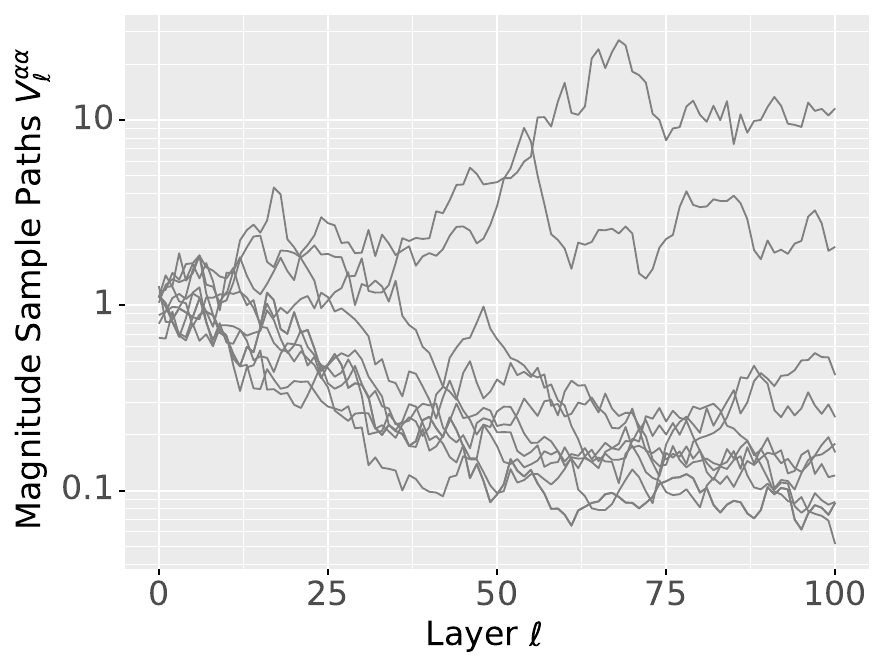}}};

\end{tikzpicture}

\caption{
Simulation of 10 shaped softplus networks as in \cref{ex:soft-plus} with 
$n=d=100, a=1, V^{\alpha\alpha}_0 = \frac{1}{\nin} |x^\alpha|^2 =  1$ 
centred at two different values. ``Stable'' here means the Neural Covariance SDE is guaranteed not to have finite time explosions; unstable networks can explode on initialization!
}
\label{fig:blow_up_stable_paths}
\end{figure}

\vspace{0.2cm}
\noindent
\textbf{Exploding and Vanishing Norms.}
Here we consider the behaviour of shaping smooth activation functions, as it is done in the experiments of \cite{martens2021rapid}.
While the authors here avoided exploding and vanishing norms by numerically optimizing shaping parameters, we can actually describe the precise behaviour a priori with the Neural Covariance SDE.
Recall the shaping parameter $a$ from \cref{def:smooth-shaping}. Let $V_t$ be the solution to the SDE in \cref{eq:smooth_cov_sde}.
We can write down the marginal SDE for $V^{\alpha\alpha}_t$ as 
\begin{equation}
\label{eq:smooth_V_diag_marginal}
    d V^{\alpha\alpha}_t 
    = 
    \left( \frac{3}{4} \varphi''(0)^2 + \varphi'''(0) \right) 
    \frac{ V^{\alpha\alpha}_t }{ a^2 } 
    ( V^{\alpha\alpha}_t - 1 ) \, dt 
    + \sqrt{2} V^{\alpha\alpha}_t \, dB_t \,, 
\end{equation}
which implies by \cref{prop:finite_time_explosion} that $V_t$ has a finite time explosion (with non-zero probability) \textbf{if and only if} 
$\frac{3}{4}  \varphi''(0)^2 + \varphi'''(0) > 0$. 
This criterion can be used to help choose how activation functions should be centered for shaping; below are two examples.

\begin{example}
[Sigmoid and $\tanh$ at $x_0 = 0$]
We start with the sigmoid activation $\sigma(x) = \frac{1}{1 + e^{-x}}$, then we can define $\varphi(x) \defequal 4 \sigma(x) - 2$ to satisfy \cref{asm:reg}, which leads to $\varphi''(0) = 0, \varphi'''(0) = - \frac{1}{2}$, and therefore leads to a stable network. 
It turns out $\varphi(x) \defequal \tanh(x)$ already satisfies \cref{asm:reg}, which leads to $\varphi''(0) = 0, \varphi'''(0) = -2$, and therefore is also stable. 
\end{example}

More generally, if $\sigma$ behaves like a cumulative distribution function for a symmetric unimodal density, we will have that $\varphi''(0)=0$ and $\varphi'''(0) < 0$ as desired. 

\begin{example}
[Soft Plus at General $x_0 \in \bR$]
\label{ex:soft-plus}
Let us consider $x_0 \in \bR$ and $\sigma(x) = \log ( 1 + e^{x+x_0} )$, which implies 
$\varphi(x) 
    \defequal 
    (1 + e^{-x_0})
    \log \frac{1 + e^{x+x_0}}{1 + e^{x_0}}$ 
satisfies \cref{asm:reg}. 
This gives us $\varphi''(0) = \frac{1}{1 + e^{x_0}}, \varphi'''(0) = \frac{1 - e^{x_0} }{ (1 + e^{x_0})^2 }$, and therefore 
$\frac{ 3 }{4} \varphi''(0)^2 + \varphi'''(0) 
    = \frac{1}{(1 + e^{x_0})^2 } 
    \left( \frac{5}{4} - e^{x_0}
    \right)$.
In other words, the shaped network is stable if and only if $x_0 \geq \log \frac{5}{4}$ (see \cref{fig:blow_up_stable_paths}). 
We note that the authors of \cite{martens2021rapid} numerically found a shift of $x_0 \approx 0.41$, which is in the stable regime of $x_0 \geq \log \frac{5}{4} \approx 0.097$.
\end{example}
\noindent
\textbf{Relationship to Edge of Chaos.}
The finite time explosion example above resembles the Edge of Chaos (EOC) analysis of gradient stability \cite{schoenholz2016deep,yang2017mean,hayou2019impact,murray2022activation}, where the weight and bias variance at initialization determines a stability criterion. 
However, we note that the EOC regime is sufficiently different that the results are not directly comparable. 
More precisely, the EOC analysis is in the sequential limit of infinite-width and then infinite-depth, which also leaves the activation function unchanged. 
Under very weak assumptions, the variance (diagonal of $V_t$) will not explode in this regime; instead, the gradient can explode due to the covariance (off diagonals). 
On the other hand, our finite explosion result is in the joint limit of depth and width, where the variance (diagonal of $V_t$) can explode instead.

\vspace{0.2cm}
\noindent
\textbf{Posterior Inference.} 
Similar to the NNGP setting, we can use the Neural Covariance SDE to generate a prior over functions $f:\mathbb{R}^{\nin} \to \mathbb{R}^{\nout}$. 
Consequently, an interesting future direction would be to study the posterior distribution, i.e. the output $\zout^{m+1}$ conditioned on $x^{m+1}$ and a training dataset $(x^\alpha, \zout^\alpha)_{\alpha=1}^m$. 
However, to our best knowledge, it is not straightforward to explicitly compute or sample from the conditional distributions for this SDE structure. 
It would be desirable to extend existing approaches in the perturbative regime \cite{yaida2020non,roberts2022principles} to our setting.

\vspace{0.2cm}
\noindent
\textbf{Extension to Other Architectures.}
The key step to deriving the covariance SDE is the conditional Gaussian distribution in \cref{eq:conditional_gaussian}, which directly leads to a Markov chain. 
It follows immediately that ResNets \cite{he2016deep} admit a similar conditional structure. 
With a bit more work for convolutional networks, we can obtain $z^\alpha_{\ell+1} | \mathcal{F}_\ell \sim \mathcal{N}( 0 , \mathcal{A}(V_\ell) \otimes I_n )$ where $\mathcal{A}$ is an affine transformation and $V_\ell$ is the previous layer's Gram matrix \cite{novak2018bayesian}. 
We note that recurrent networks will not lead to a Markov chain or SDE limit, as the weight matrix is reused from layer to layer.

\vspace{0.2cm}
\noindent
\textbf{Simulating SDEs.}
Both the Markov chains and SDEs predict neural networks at initialization very well (see \cref{fig:density_path_shape}), but the SDE is significantly faster to simulate. 
In particular, we can view the Markov chain as an approximate Euler discretization of the SDE, but with a very small step size $n^{-1}$. 
In contrast, to simulate the SDE we should only need a step size that is small on the scale of depth-to-width ratio $T = d/n$, which is \emph{independent of width} $n$. 
Therefore, practitioners using the shaping techniques of \cite{martens2021rapid,zhang2022deep} can now simulate the covariance SDEs at a low computational cost to significantly improve estimates of the output correlation (see \cref{fig:density_path_shape} and additional simulations in \cref{sec:app_additional_simulations}). 

\vspace{0.2cm}
\noindent
\textbf{Analytical Tractability of SDEs.}
Besides numerical tractability, the SDEs are also far more tractable to analyze. 
For example, in the one input case, we arrive at geometric Brownian motion \cref{eq:geometric_brownian_motion}, which is known to have a log-normal distribution at fixed times. 
Similarly, our finite time explosions hinge on the fact we identified an SDE limit. 
In the same way that NNGP theory played a major role in the infinite-width regime, the Neural Covariance SDEs and the techniques developed here also serve as a mathematical foundation for studying training and generalization.

\section*{Acknowledgement}

We would like to thank 
Sinho Chewi, 
James Foster, 
Boris Hanin, 
Cameron Jakub, 
Jeffrey Negrea, 
Nuri Mert Vural, 
Guodong Zhang, 
Matthew S. Zhang, 
and Yuchong Zhang 
for helpful discussions and draft feedback. 
We would like to thank Sam Buchanan and Soufiane Hayou for pointing out a gap in the proof of \cref{prop:unshaped_relu_corr}.
ML is supported by Ontario Graduate Scholarship 
and the Vector Institute. 
MN is supported by an NSERC Discovery Grant.
DMR is supported in part by Canada CIFAR AI Chair funding through the Vector Institute, an NSERC Discovery Grant, Ontario Early Researcher Award, a stipend provided by the Charles Simonyi Endowment, and a New Frontiers in Research Exploration Grant.

\appendix

\printbibliography

\section{Background on Markov Chain Convergence to SDEs}
\label{sec:app_markov_chain_convergence_sde}

In this section we briefly review the background and technical results required to characterize the convergence of a Markov chain to an SDE. 
Majority of the content in this section are based on \cite{kallenberg2021foundations,ethier2009markov,stroock1997multidimensional}. 

To start we first introduce the Skorohod $J_1$-topology  \cite[Appendix 5]{kallenberg2021foundations}.
Let $S$ be a complete separable metric space, and $D_{\bR_+, S}$ be the space of c\`adl\`ag functions (right continuous with left limits) from $\bR_+ \to S$. 
Here we write $x_n \xrightarrow{ul} x$ to denote locally uniform convergence (i.e., uniform on compact subsets of $\bR_+$). 
We also consider bijections $\lambda$ on $\bR_+$ so that $\lambda$ is strictly increasing with $\lambda_0 = 0$. 
We can now define \emph{Skorohod convergence} $x_n \xrightarrow{s} x$ on $D_{\bR_+, S}$ if there exists a sequence of bijections $\lambda_n$ satisfying the above conditions and 
\begin{equation}
    \lambda_n \xrightarrow{ul} \text{Id} \,, 
    \quad 
    x_n \circ \lambda_n \xrightarrow{ul} x \,. 
\end{equation}

The most important result is that $D_{\bR_+, S}$ 
equipped with the above sense of convergence is indeed a well behaved probability space, which we state below. 

\begin{thm}
[Theorem A5.3, \cite{kallenberg2021foundations}]
For any separable complete metric space $S$, there exists a topology $\mathcal{T}$ on $D_{\bR_+, S}$ such that 
\begin{enumerate}[label=(\roman*)]
    \item $\mathcal{T}$ induces the Skorohod convergence $x_n \xrightarrow{s} x$, 
    \item $D_{\bR_+, S}$ is Polish (separable completely metrizable topological space) under $\mathcal{T}$, 
    \item $\mathcal{T}$ generates the Borel $\sigma$-field generated by the evaluation maps $\pi_t$, $t \geq 0$, where $\pi_t(x) = x_t$.
\end{enumerate}
\end{thm}

We also need to define Feller semi-groups. 
To start we let $S$ be a locally compact separable metric space and $C_0 \defequal C_0(S)$ be the space of continuous functions that vanishes at infinity, and we equip $C_0$ with the sup norm to make it a Banach space. 
$T: C_0 \to C_0$ is a \emph{positive contraction operator} if for all $0\leq f \leq 1$ we have $0 \leq Tf \leq 1$. 
A semi-group of such operators $(T_t)$ on $C_0$ is called a \emph{Feller semi-group} if it additionally satisfies 
\begin{equation}
\begin{aligned}
    & T_t C_0 \subset C_0 \,, \quad t \geq 0 \,, \\ 
    & T_t f(x) \to x \text{ as } t \to 0\,, 
    \quad f \in C_0, x \in S \,. 
\end{aligned}
\end{equation}

Let $\mathcal{D} \subset C_0$ and $A:\mathcal{D} \to C_0$, and we say that $(A,\mathcal{D})$ is a \emph{generator} of $(T_t)$ if $\mathcal{D}$ is the maximal set such that for all $f \in \mathcal{D}$, we have that 
\begin{equation}
    \lim_{t\to 0} \frac{T_t f - f}{t} = A f \,. 
\end{equation}

An operator $A$ with domain $\mathcal{D}$ on a Banach space $B$ is said to be \emph{closed}, if its graph $G = \{ (f, Af) | f \in \mathcal{D} \}$ is a closed subset of $B\times B$. If the closure of $G$ is the graph of an operator $\bar A$, we say $\bar A$ is the \emph{closure} of $A$. 
Finally, we will define a linear subspace $D \subset \mathcal{D}$ as a \textbf{core} of $A$ if the closure of $A|_D$ is $A$. 
If $(A,\mathcal{D})$ is a generator of a Feller semigroup, every dense invariant subspace $D \subset {D}$ is a core of $A$ \cite[Proposition 17.9]{kallenberg2021foundations}. 
In particular, we will work with the core $C^\infty_0$ of smooth functions vanishing at infinity. 

We will state a sufficient condition required for an semi-group to be Feller based on its generator. 

\begin{thm}
[Section 8, Theorem 2.5, \cite{ethier2009markov}]
\label{thm:feller_generator}
Let $a^{ij} \in C^2(\mathbb{R}^d)$ with $\partial_k \partial_\ell a^{ij}$ be bounded for all $i,j,k,\ell \in [d]$. 
Further let $b:\mathbb{R}^d \to \mathbb{R}^d$ be Lipschitz. 
Then the generator defined by 
\begin{equation}
\label{eq:generator_defn}
    Af = \frac{1}{2} \sum_{i,j=1}^d a^{ij} \partial_i \partial_j f 
        + \sum_{i=1}^d b^i \partial_i f \,, 
\end{equation}
generates a Feller semi-group on $C_0$. 
\end{thm}

We will next state a set of equivalent criterion for convergence of Feller processes. 

\begin{thm}
[Theorem 17.25, \cite{kallenberg2021foundations}]
\label{thm:conv_feller_process}
Let $X, X^1, X^2, X^3, \cdots$ be Feller processes in $S$ with semi-groups $(T_t), (T_{n,t})$ and generators $(A, \mathcal{D}), (A_n, \mathcal{D}_n)$, respectively, and fix a core $D$ for $A$. 
Then these conditions are equivalent: 
\begin{enumerate}[label=(\roman*)]
    \item for any $f \in D$, there exists some $f_n \in \mathcal{D}_n$ with $f_n \to f$ and $A_n f_n \to A f$, 
    \item $T_{n,t} \to T_t$ strongly for each $t > 0$, 
    \item $T_{n,t} f \to T_t f $ for every $f \in C_0$, uniformly for bounded $t>0$, 
    \item $X^n_0 \xrightarrow{d} X_0$ in $S \Rightarrow X^n \xrightarrow{d} X$ in the Skohorod topology of $D_{\bR_+, S}$. 
\end{enumerate}
\end{thm}

Once again, we note that it is common to choose the core $D = C^\infty_0$, and that checking condition (i) is sufficient for convergence in the Skorohod topology. 
This is translated to the Markov chain setting by the next theorem. 

\begin{thm}
[Theorem 17.28, \cite{kallenberg2021foundations}]
\label{thm:conv_markov_chain}
Let $Y^1, Y^2, Y^3, \cdots$ be discrete time Markov chains in $S$ with transition operators $U_1, U_2, U_3, \cdots$, and let $X$ be a Feller process with semi-group $(T_t)$ and generator $A$. Fix a core $D$ for $A$, and let $0 < h_n \to 0$. 
Then conditions $(i)-(iv)$ of \cref{thm:conv_feller_process} remain equivalent for the operators and processes 
\begin{equation}
    A_n = h_n^{-1} ( U_n - I ) \,, \quad 
    T_{n,t} = U_n^{\lfloor t / h_n \rfloor} \,, \quad 
    X^n_t = Y^n_{ \lfloor t / h_n \rfloor } \,. 
\end{equation}
\end{thm}

It remains to check that the generators $A_n$ converges to $A$ with respect to the core $D = C^\infty_0$, and we will use a criterion from \cite{stroock1997multidimensional}. 
Here we will first let $\Pi_n(x,dy)$ be the Markov transition kernel of $Y^n$, and define 
\begin{equation}
\begin{aligned}
    a^{ij}_n(x) 
    &= \frac{1}{h_n} \int_{|y-x| \leq 1} 
        (y_i - x_i) (y_j - x_j) \, \Pi_n(x, dy) \,, \\ 
    b^i_n(x) 
    &= \frac{1}{h_n} \int_{|y-x| \leq 1} 
        (y_i - x_i) \, \Pi_n(x, dy) \,, \\ 
    \Delta^\epsilon_n(x) 
    &= \frac{1}{h_n} \Pi_n( x, \mathbb{R}^d \setminus B(x,\epsilon) ) \,. 
\end{aligned}
\end{equation}

\begin{lem}
[Lemma 11.2.1, \cite{stroock1997multidimensional}]
\label{lm:conv_generator}
The following two conditions are equivalent: 
\begin{enumerate}[label=(\roman*)]
    \item For any $R>0, \epsilon>0$ we have that 
    \begin{equation}
        \lim_{n\to\infty} \sup_{|x| \leq R} 
        \| a_n(x) - a(x) \|_{op} + |b_n(x) - b(x)| 
        + \Delta^\epsilon_n(x) 
        = 0 \,,
    \end{equation} 
    \item For each $f \in C^\infty_0(\mathbb{R}^d)$, 
    we have that 
    \begin{equation}
        \frac{1}{h_n} A_n f \to A f \,, 
    \end{equation}
    uniformly on compact sets of $\mathbb{R}^d$, 
    where $A$ is defined as \cref{eq:generator_defn}. 
\end{enumerate}
\end{lem}

Finally, we summarize the above results in a user friendly form for our applications. 

\begin{prop}
[Convergence of Markov Chains to SDE]
\label{prop:conv_markov_chain_to_sde}
Let $Y^n$ be a discrete time Markov chain on $\mathbb{R}^N$ defined by the following update 
for $p,\delta > 0$
\begin{equation}
    Y^n_{\ell+1} 
    = Y^n_\ell + \frac{\widehat b_n(Y^n_\ell, \omega^n_\ell )}{n^{2p}} 
    + \frac{\sigma_n(Y^n_\ell)}{n^p} \xi^n_\ell 
    + O(n^{-2p-\delta}) \,, 
\end{equation}
where $\xi^n_\ell \in \mathbb{R}^N$ are iid random variables with zero mean, identity covariance, and moments uniformly bounded in $n$. 
Furthermore, $\omega^n_\ell$ are also iid random variables such that 
$\mathbb{E}[ \widehat b_n(Y^n_\ell, \omega^n_\ell) | Y^n_\ell = y ] = b_n(y)$ 
and $\widehat b_n(y, \omega^n_\ell)$ has uniformly bounded moments in $n$. 
Finally, $\sigma_n$ is a deterministic function, and the remainder terms in $O(n^{-2p-\delta})$ have uniformly bounded moments in $n$. 

Suppose $b_n, \sigma_n$ are uniformly Lipschitz functions in $n$ and converges to $b, \sigma$ uniformly on compact sets, 
then in the limit as $n\to\infty$, the process $X^n_t = Y^n_{\lfloor t n^{2p} \rfloor}$ converges in distribution to the solution of the following SDE in the Skorohod topology of $D_{\bR_+, \bR^N}$ 
\begin{equation}
    d X_t = b(X_t) \, dt + \sigma(X_t) \, dB_t \,, 
    \quad 
    X_0 = \lim_{n\to\infty} Y_0^n \,. 
\end{equation}

Suppose otherwise $b_n, \sigma_n$ are only locally Lipschitz (but still uniform in $n$), then $X^n$ converges locally to $X$ in the same topology (see \cref{def:local_convergence}). 
More precisely, for any fixed $r > 0$, we consider the stopping times 
\begin{equation}
    \tau^n \defequal \inf \left\{ t \geq 0 : |X^n_t| \geq r \right\} 
    \,, \quad 
    \tau \defequal \inf \left\{ t \geq 0 : |X_t| \geq r
    \right\} \,, 
\end{equation}
then the stopped process $X^n_{t \wedge \tau^n}$ converges in distribution to the stopped solution $X_{t \wedge \tau}$ of the above SDE in the same topology. 
\end{prop}

\begin{proof}

We will essentially check the criterion of \cref{thm:conv_markov_chain} directly for the metric space $S = \bR^N$ if $b,\sigma$ is globally Lipschitz, and $S = B(0,r)$ otherwise. 
In both of these cases, $b, \sigma$ are Lipschitz on $S$, therefore the limiting process (either $X_t$ or $X_{t\wedge \tau}$) is Feller in $S$ by \cref{thm:feller_generator}. 

In the equivalent criteria of \cref{thm:conv_feller_process}, we will use the implication of (i) $\Rightarrow$ (iv) to get convergence of $X^n$ to $X$ in the Skorohod topology of $D_{\bR_+, S}$. 
More precisely, it is sufficient to choose $h_n = \frac{1}{n^{2p}}$ as the natural time scale, and check $\frac{1}{h_n} A_n f \to A f$ for any $f \in C^\infty_0$. 
Given \cref{lm:conv_generator}, it is sufficient to check the convergence of the coefficients and $\Delta^\epsilon_n$. 

We start with $\Delta^\epsilon_n(x)$. Given that the randomness in the Markov chain have bounded moments (uniform in $n$), then by a Markov inequality we have that for any $q>0$ 
\begin{equation}
    \Pi_n(x, \mathbb{R}^d \setminus B(x,\epsilon))
    = \mathbb{P} \left[ \left| 
    \frac{\widehat b(x, \omega^n_\ell)}{n^{2p}} + \frac{\sigma}{n^p} \xi^n_\ell + O(n^{-2p-\delta})
    \right|^{2q} \geq \epsilon^{2q} 
    \right]
    \leq O\left( \epsilon^{-2q} n^{-2pq} \right) \,, 
\end{equation}
therefore choosing $q>1$ we have $\sup_{|x|\leq R} \Delta^\epsilon_n(x) = O(n^{-2p(q-1)}) \to 0$ for any fixed $\epsilon$. 

We can rewrite $b_n(x)$ as 
\begin{equation}
    b_n(x) 
    = n^{2p} \mathbb{E} [ Y^n_{\ell+1} - Y^n_\ell 
        | Y^n_\ell = x ] 
    + O( n^{-\delta} ) 
    \to b(x) \,, 
\end{equation}
since $\xi_\ell^n$ has zero mean and the remainder terms have bounded moments (uniform in $n$), which also gives the desired convergence of $\sup_{x\leq |R|} | b_n(x) - b(x) | \to 0$. 

Similarly we can rewrite $a_n(x)$ as 
\begin{equation}
    a_n(x) 
    = n^{2p} \mathbb{E} [ (Y^n_{\ell+1} - Y^n_\ell) (Y^n_{\ell+1} - Y^n_\ell)^\top | Y^n_\ell = x ] 
    + O(n^{-2\delta} + n^{-2p}) 
    \to \sigma(x) \sigma(x)^\top \,, 
\end{equation}
where we note the drift's randomness contributes the higher order $n^{-2p}$ term and therefore also vanishes in the limit. 
This implies $\sup_{x \leq |R|} \|a_n(x) - a(x)\|_{op}  \to 0$, which gives us the desired result. 

\end{proof}

\section{Unshaped ReLU Markov Chain}
\label{sec:app_unshaped_relu_markov_chain}

In this section, we will derive the Markov chain update \cref{eq:corr_update} with explicit coefficients. 
For the rest of this section, we will adopt the following notation. 
Let $\varphi(x) \defequal \max(x,0)$ be the ReLU activation function. 
Let $f(x) = \frac{1}{\sqrt{2\pi}} e^{-x^2/2}$ 
be the density of a standard Gaussian, 
and let $F(x) = \int_{-\infty}^x f(t) \, dt$ be the cumulative distribution function (CDF). 

\begin{lem}[Gaussian Integration-by-Parts with Indicator Function]
\label{lm:gip_ind}
For $g\sim \cN(0,1)$ and $h$ is weakly differentiable, we have that 
\begin{equation}
	\mathbb{E} \, g \mathds{1}_{\{ g > -a\}} h(g) 
	= h(-a) f(a) + \mathbb{E} \, \mathds{1}_{\{g>-a\}} h'(g) \,, 
\end{equation}
where $f$ is the standard Gaussian density. 
\end{lem}

\begin{proof}

We start by writing the expectation as an integral 
\begin{equation}
	\mathbb{E} \, g \mathds{1}_{\{ g > -a\}} h(g) 
	= \int_{-a}^\infty x h(x) f(x) \, dx \,. 
\end{equation}

Here by observing that $f'(x) = - x f(x)$, 
we can use integration by parts for 
$u = h(x), dv = x f(x) \, dx$ to get $du = h'(x) \, dx, v = - f(x)$, 
and therefore 
\begin{equation}
	\int_{-a}^\infty x h(x) f(x) \, dx
	= \left[ -h(x) f(x) \right]_{-a}^\infty 
		+ \int_{-a}^\infty h'(x) f(x) \, dx 
	= h(-a) f(-a) + \mathbb{E} \, \mathds{1}_{\{g>-a\}} h'(g) \,. 
\end{equation}

Finally we recover the desired result using symmetry of $f(-a)=f(a)$. 

\end{proof}

We will note the special case of $a=0$ to get 
\begin{equation}
    \mathbb{E} \, g \mathds{1}_{\{g > 0\}} h(g)
    = \frac{h(0)}{\sqrt{2\pi}} + \mathbb{E} \, \mathds{1}_{\{g>0\}} h'(g) \,. 
\end{equation}

\begin{lem}[Gaussian Density Substitution]
\label{lm:gaussian_density_sub}
Let $g \sim \cN(0,1), \rho \in [0,1], q = \sqrt{1 - \rho^2}$, then we have that 
\begin{equation}
	\mathbb{E} \, h(g) f\left( \frac{\rho g + a}{q} \right) 
	= q f(a) \mathbb{E} \, h( qg - \rho a ) \,. 
\end{equation}
\end{lem}

\begin{proof}

We will again write the expectation as an integral 
\begin{equation}
	\mathbb{E} \, h(g) f\left( \frac{\rho g + a}{q} \right) 
	= 
	\int h(x) f\left( \frac{\rho x + a}{q} \right) f(x) \, dx \,. 
\end{equation}

Here observe that 
\begin{equation}
	f\left( \frac{\rho x + a}{q} \right) f(x) 
	= \frac{1}{2\pi} \exp\left[ - \frac{(\rho x+a)^2}{ 2 q^2 } - \frac{x^2}{2} \right] 
	= \frac{1}{2\pi} \exp\left[ - \frac{ \rho^2 x^2 + a^2 + 2 a \rho x + q^2 x^2 }{2 q^2} \right] \,, 
\end{equation}
at this point, we can complete the square to write 
\begin{equation}
	\rho^2 x^2 + a^2 + 2 a \rho x + q^2 x^2 
	= (x+a\rho)^2 - a^2 \rho^2 + a 
	= (x+a\rho)^2 - a^2 q^2 \,. 
\end{equation}

This implies that we have 
\begin{equation}
	f\left( \frac{\rho x + a}{q} \right) f(x) 
	= \frac{1}{2\pi} \exp\left[ - \frac{ (x+a\rho)^2 }{2q^2} - \frac{a^2}{2} \right] 
	= f\left( \frac{x+a\rho}{q} \right) f(a) \,. 
\end{equation}

Finally, we can use the substitution 
$y = \frac{x+a\rho}{q}, dy = \frac{1}{q} dx$ to get 
\begin{equation}
	\int h(x) f\left( \frac{\rho x + a}{q} \right) f(x) \, dx 
	= 
		\int h(qy - \rho a) f(y) f(a) q \, dy 
	= 
		q f(a) \mathbb{E} \, h(qg - \rho a) \,, 
\end{equation}
which is the desired result. 

\end{proof}

We will start by calculating simpler quantities. 
\begin{lem}
[Moments]
Let $g \sim \mathcal{N}(0,1)$, then 
\begin{equation}
\begin{aligned}
    \mathbb{E} \, \varphi(g) 
    = \frac{1}{\sqrt{2\pi}} \,, \quad 
    \mathbb{E} \, \varphi(g)^2 
    = \frac{1}{2} \,, \quad 
    \mathbb{E} \, \varphi(g)^4 
    = \frac{3}{2} \,. 
\end{aligned}
\end{equation}
\end{lem}

\begin{proof}

For the second and fourth moments, we simply observe that $g^2$ is symmetric and $\varphi$ is exactly half of of the integral. 
For the first integral we will use Gaussian integration-by-parts with $h(g) = 1$ to get 
\begin{equation}
    \mathbb{E} \, \varphi(g) 
    = \mathbb{E} \, g \mathds{1}_{\{g > 0\}} 
    = \frac{1}{\sqrt{2\pi}} \,, 
\end{equation}
which is the desire result. 

\end{proof}

We will also recall the following result from 
\cite{cho2009kernel}

\begin{lem}
[$\bar J_0, \bar J_1, \bar J_2$]
Let $\rho \in [0,1], q = \sqrt{1-\rho^2}$ and let 
$\rho,w \sim \mathcal{N}(0,1)$ be independent. 
Then we have that 
\begin{equation}
\begin{aligned}
    \bar J_0(\rho) 
    &= \mathbb{E} \, \mathds{1}_{\{g>0\}} 
    \mathds{1}_{\{\rho g + qw > 0\}} 
    = \frac{\arccos(-\rho)}{2\pi} \,, \\
    \bar J_1(\rho) 
    &= \mathbb{E} \, \varphi(g) \varphi(\rho g+qw) 
    = \frac{ q + \rho \arccos(-\rho) }{ 2\pi } \,, \\ 
    \bar J_2(\rho) 
    &= \mathbb{E} \, \varphi(g)^2 \varphi(\rho g+qw)^2 
    = \frac{3 \rho q + \arccos(-\rho) (1 + 2 \rho^2 ) }{ 2 \pi } \,. 
\end{aligned}
\end{equation}
\end{lem}

We will need to compute the following quantity. 
\begin{lem}
[$\bar J_{3,1}$]
Let $\rho \in [0,1], q = \sqrt{1-\rho^2}$ and let 
$\rho,w \sim \mathcal{N}(0,1)$ be independent. 
Then we have that 
\begin{equation}
    \bar J_{3,1}(\rho) 
    = \mathbb{E} \, \varphi(g)^3 \varphi(\rho g+qw) 
    = \frac{q ( 2 + \rho^2) 
			+ 3 \arccos(-\rho) \rho}{2\pi} \,. 
\end{equation}
\end{lem}

\begin{proof}

We start by using Gaussian integration-by-parts with 
$h(g) = \mathbb{E}_g \, g^2 \varphi( \rho g + qw )$ 
where we use $\mathbb{E}_g [\,\cdot\,] \defequal \mathbb{E} \, [ \,\cdot\, | g ]$ to denote conditional expectation 
\begin{equation}
\begin{aligned}
    \mathbb{E} \, \varphi(g)^3 \varphi(\rho g+qw) 
    &= 
    \mathbb{E} \, g \mathds{1}_{\{g>0\}} h(g) 
    \\ 
    &= \mathbb{E} \, \mathds{1}_{\{g > 0\}} \left[ 
        2g \mathbb{E}_g \, \varphi(\rho g+qw) 
        + g^2 \mathbb{E}_g \rho \mathds{1}_{\{\rho g+qw>0\}}
        \right] 
        \\ 
    &= 2 \bar J_1(\rho) 
        + \rho \mathbb{E}\, g \mathds{1}_{\{g>0\}} 
        \mathbb{E}_g g \mathds{1}_{\{\rho q+qw>0\}} \,. 
\end{aligned}
\end{equation}

Here we observe that $\mathbb{E}_g g \mathds{1}_{\rho q+qw>0} = g F(\rho g/q)$, and we can again set this to the new $h(g)$ and use integration-by-parts to write 
\begin{equation}
    \mathbb{E}\, g \mathds{1}_{\{g>0\}} 
        \mathbb{E}_g g \mathds{1}_{\{\rho q+qw>0\}}
    = 
    \mathbb{E} \, \mathds{1}_{\{g>0\}} \mathds{1}_{\{\rho g+qw>0\}} 
    + \mathbb{E} \, \mathds{1}_{\{g>0\}} \frac{\rho g}{q}
    f\left( \frac{\rho g}{q} \right) \,. 
\end{equation}

At this point we can use the substitution formula from \cref{lm:gaussian_density_sub} to write 
\begin{equation}
    \mathbb{E} \, \mathds{1}_{\{g>0\}} \frac{\rho g}{q}
    f\left( \frac{\rho g}{q} \right)
    = \frac{\rho}{q} f(0) \mathbb{E} \varphi(q g) 
    = \frac{\rho q}{2\pi} \,. 
\end{equation}

Putting this together, we have 
\begin{equation}
    \bar J_{3,1}(\rho) 
    = 2 \bar J_1(\rho) + \rho \bar J_0(\rho) + \frac{\rho^2 q}{2\pi} \,, 
\end{equation}
which is the desired result after simplifying. 

\end{proof}

We will now recall the ReLU-like activations 
for $s = (s_+, s_-) \in \mathbb{R}^2$ 
\begin{equation}
	\varphi_s(x) := s_+ \max(x,0) + s_- \min(x, 0) 
		= s_+ \varphi(x) - s_- \varphi(-x) \,, 
\end{equation}
where $\varphi(x) := \max(x,0)$ is the usual ReLU activation. 

We will compute several basic moments first. 

\begin{lem}[Moments, $c, M_2$]
\label{lm:moments_leaky}
Let $g \sim N(0,1)$, we have that 
\begin{equation}
\begin{aligned}
	\mathbb{E} \, \varphi_s(g) 
	= 
		\frac{s_+ - s_-}{\sqrt{2\pi}} \,, \quad
	\mathbb{E} \, \varphi_s(g)^2 
	= 
		\frac{s_+^2 + s_-^2}{2} \,, \quad 
	\mathbb{E} \, \varphi_s(g)^4 
	= 
		\frac{3}{2}( s_+^4 + s_-^4 ) \quad 
        \,. 
\end{aligned}
\end{equation}
Furthermore, this implies the normalizing constant is $c = \frac{2}{s_+^2 + s_-^2}$ and 
\begin{equation}
    M_2 \defequal \mathbb{E} \, [c \varphi_s(g)^2 - 1]^2 
    = 6 \frac{ s_+^4 + s_-^4 }{ (s_+^2 + s_-^2)^2 } - 1 \,. 
\end{equation}
\end{lem}

\begin{proof}

To start we first recall the Gaussian integration by parts calculation 
\begin{equation}
	\mathbb{E} \, \varphi(g) = f(0) = \frac{1}{\sqrt{2\pi}} \,, 
\end{equation}
then the first moment follows immediately from rewriting in terms of $\varphi$ 
\begin{equation}
	\mathbb{E} \, \varphi_s(g) 
	= s_+ \mathbb{E} \, \varphi(g) - s_- \mathbb{E} \, \varphi(-g) 
	= \frac{s_+ - s_-}{\sqrt{2\pi}} \,. 
\end{equation}

For the second moment, we will also rewrite in terms of $\varphi$ 
\begin{equation}
	\mathbb{E} \, \varphi_s(g)^2 
	= \mathbb{E} \, s_+^2 \varphi(g)^2 + s_-^2 \varphi(-g)^2 
		- 2 s_+ s_- \varphi(g) \varphi(-g) 
	= ( s_+^2 + s_-^2 ) \mathbb{E} \, \varphi(g)^2 \,, 
\end{equation}
where we used that $\varphi(g) \varphi(-g) = 0$ almost surely 
and $g \overset{d}{=} -g$, and the desire result follows from 
Gaussian integration by parts 
\begin{equation}
	\mathbb{E} \, \varphi(g)^2 
	= 0 f(0) + \mathbb{E} \, \mathds{1}_{\{g>0\}} 
	= \frac{1}{2} \,. 
\end{equation}

For the fourth moment, we will similarly observe that all mixed moments 
$\varphi(g)^p \varphi(-g)^r = 0$ almost surely whenever $p,r > 0$, 
which allows us to write 
\begin{equation}
	\mathbb{E} \, \varphi_s(g)^4 
	= \mathbb{E} \, s_+^4 \varphi(g)^4 + s_-^4 \varphi(-g)^4 
	= (s_+^4 + s_-^4) \mathbb{E} \, \varphi(g)^4 \,, 
\end{equation}
and the desire result follows from the Gaussian integration by parts calculation 
\begin{equation}
	\mathbb{E} \, \varphi(g)^4 
	= 0^3 f(0) + \mathbb{E} \, 3 g^2 \mathds{1}_{\{g>0\}} 
	= 3(0^3 f(0) + \mathbb{E} \, \mathds{1}_{\{g>0\}} ) 
	= \frac{3}{2} \,. 
\end{equation}

\end{proof}

We will also convert the $\bar J_{k,\ell}$ formulas to $K_{k,\ell}$ formulas, i.e. the following quantities 
\begin{equation}
\begin{aligned}
	K_{p,r}(\rho) &:= \mathbb{E} \, \varphi_s(g)^p \varphi_s(\hat g)^r \,, 
\end{aligned}
\end{equation}
where $g, w \sim N(0,1)$ and we define $\hat g = \rho g + qw$ with $q = \sqrt{1-\rho^2}$. 
We will also use the short hand notation to write 
$\bar J_p := \bar J_{p,p}, K_p := K_{p,p}$. 

\begin{lem}[$K_1, K_2, K_{3,1}$]
\label{lm:joint_moments_leaky}
Let $\rho \in [-1,1]$, $q = \sqrt{1-\rho^2}$, $g, w \sim \cN(0,1)$, 
and $\hat g = \rho g + qw$. 
Then we have the following formulas 
\begin{equation}
\begin{aligned}
	K_1(\rho) &= 
		(s_+^2 + s_-^2) \bar J_1(\rho) - 2 s_+ s_- \bar J_1(- \rho) \,, \\ 
	K_2(\rho) &= 
		(s_+^4 + s_-^4) \bar J_2(\rho) + 2 s_+^2 s_-^2 \bar J_2(- \rho) \,, \\ 
	K_{3,1}(\rho) &= 
		(s_+^4 + s_-^4) \bar J_{3,1}(\rho) 
		- s_+ s_- ( s_+^2 + s_-^2 ) \bar J_{3,1}(- \rho) \,. 
\end{aligned}
\end{equation}
\end{lem}

\begin{proof}

Before we start, we will make several observations. 
Using the fact that $(g,w) \overset{d}{=} (\pm g, \pm w)$, 
we have the following equality in distribution relations 
\begin{equation}
\begin{aligned}
	(g, \rho g+qw) &\overset{d}{=} ( -g, -\rho g-qw ) = (-g,-\hat g) \,, \\ 
	(g, -\hat g) &\overset{d}{=} (g, -\rho g + qw) 
	\overset{d}{=} (-g, \rho g + qw) = (-g, \hat g) \,. 
\end{aligned}
\end{equation}
In particular, we note that the two Gaussian random variable $(g, -\hat g)$ 
have correlation $-\rho$. 

This allows us to simplify $K_1$ 
\begin{equation}
\begin{aligned}
	K_1(\rho) 
	&= 
		\mathbb{E} \, \varphi_s(g) \varphi_s(\hat g) \\ 
	&= 
		\mathbb{E} \, s_+^2 \varphi(g) \varphi(\hat g) 
			+ s_-^2 \varphi(-g) \varphi(-\hat g) 
		- s_+ s_- \varphi(g) \varphi(-\hat g) 
		- s_+ s_- \varphi(-g) \varphi(\hat g) 
		\\ 
	&= ( s_+^2 + s_-^2 ) \bar J_1(\rho) 
		- 2 s_+ s_- \bar J_1(-\rho) \,, 
\end{aligned}
\end{equation}
which is the desired result. 

With $K_2$, we will additionally make use of the fact that 
$\varphi(g) \varphi(-g) = 0$ almost surely to write 
\begin{equation}
\begin{aligned}
	K_2(\rho) 
	&= 
		\mathbb{E} \, ( s_+^2 \varphi(g)^2 + s_-^2 \varphi(-g)^2 ) 
			( s_+^2 \varphi(\hat g)^2 + s_-^2 \varphi(-\hat g)^2 ) 
			\\ 
	&= 
		\mathbb{E} \, s_+^4 \varphi(g)^2 \varphi(\hat g)^2 
			+ s_-^4 \varphi(-g)^2 \varphi(-\hat g)^2 
			+ s_+^2 s_-^2 \varphi(g)^2 \varphi(-\hat g)^2 
			+ s_+^2 s_-^2 \varphi(-g)^2 \varphi(\hat g)^2 
			\\ 
	&= ( s_+^4 + s_-^4 ) \bar J_2(\rho) 
		+ 2 s_+^2 s_-^2 \bar J_2(-\rho) \,. 
\end{aligned}
\end{equation}

$K_{3,1}$ follows from a similar calculation 
\begin{equation}
\begin{aligned}
	K_{3,1}(\rho) 
	&= 
		\mathbb{E} \, ( s_+^3 \varphi(g)^3 - s_-^3 \varphi(-g)^3 ) 
			( s_+ \varphi(\hat g) - s_- \varphi(-\hat g) ) 
			\\ 
	&= 
		\mathbb{E} \, s_+^4 \varphi(g)^3 \varphi(\hat g)
			s_-^4 \varphi(-g)^3 \varphi(-\hat g) 
			- s_+^3 s_- \varphi(g)^3 \varphi(-\hat g) 
			- s_+ s_-^3 \varphi(-g)^3 \varphi(\hat g) 
			\\ 
	&= ( s_+^4 + s_-^4 ) \bar J_{3,1}(\rho) 
		- s_+ s_- ( s_+^2 + s_-^2 ) \bar J_{3,1}(- \rho) \,. 
\end{aligned}
\end{equation}

\end{proof}

Finally, we to get to state the desired formulas for the approximate Markov chain. 
Here we will make introduce several definitions first. 
In the event that $|\varphi^\alpha_\ell|=0$ or $|\varphi^\beta_\ell|=0$, 
the formula $\rho^{\alpha\beta}_\ell \defequal \frac{ \langle \varphi^\alpha_\ell, \varphi^\beta_\ell \rangle }{ |\varphi^\alpha_\ell| \, |\varphi^\alpha_\ell| }$ is undefined. 
We will remedy this by introducing an additional point $\mathbf{e}$ in the state space $\mathbb{R} \cup \{\mathbf{e}\}$, and set $\rho^{\alpha\beta}_\ell = \mathbf{e}$ in this event. 
We note that once $\rho^{\alpha\beta}_\ell = \mathbf{e}$, then the next step $\rho^{\alpha\beta}_{\ell+1} = \mathbf{e}$ as well since either $z^{\alpha}_{\ell+1},z^\beta_{\ell+1} = 0$. 
For all $x \in \mathbb{R}$ we will define the distance $|x-\mathbf{e}|=\infty$. 
Consequently, $\mathbb{R} \cup \{\mathbf{e}\}$ is a Polish space (complete separable metric space), and therefore it's a well behaved probability space (e.g. admits conditional densities). 
For a random variable $X$, we write $X=O(n^p)$ if all moments of $n^{-p} X$ are bounded by a constant independent of $n$.

We will also define the bounded Lipschitz function norm as 
\begin{equation}
    \|h\|_{BL} 
    \defequal 
        \|h\|_\infty + \sup_{x\neq y} \frac{|h(x) - h(y)|}{|x-y|} \,, 
\end{equation}
which induces the bounded Lipschitz distance for probability measures 
\begin{equation}
    d_{BL}(\mu, \nu) 
    \defequal 
    \sup_{ \|h\|_{BL} \leq 1 } 
    \int h \, d\mu 
    - \int h \, d\nu \,. 
\end{equation}

\begin{prop}
[Unshaped ReLU Correlation]
\label{prop:unshaped_relu_corr}
Let $\rho^{\alpha\beta}_\ell \defequal \frac{ \langle \varphi^\alpha_\ell, \varphi^\beta_\ell \rangle }{ |\varphi^\alpha_\ell| \, |\varphi^\alpha_\ell| }$ when defined, and $\mathbf{e}$ when either $|\varphi^\alpha_\ell| \,, |\varphi^\alpha_\ell| = 0$. 
Let us also define the approximate Markov chain 
\begin{equation}
\label{eq:unshaped_relu_corr}
    p_{\ell+1} 
    = c K_1( p_\ell ) 
    + \frac{ \mu_{\text{ReLU}}(p_\ell)}{n} 
    + \sigma_{\text{ReLU}}(p_\ell) 
    \frac{z_\ell}{\sqrt{n}} \,, 
\end{equation}
where $z_\ell$ are iid $\cN(0,1)$ and 
\begin{equation}
\begin{aligned}
    \mu_{\text{ReLU}}(\rho^{\alpha\beta}_\ell) 
    &= 
        \frac{c}{4} \left[ K_1 ( c^2 K_2 + 3 M_2 + 3 ) - 4c K_{3,1} 
		\right] \,, \\ 
    \sigma^2_{\text{ReLU}}(\rho^{\alpha\beta}_\ell) 
    &= 
        \frac{c^2}{2} \left[ 
		K_1^2 ( c^2 K_2 + M_2 + 1 ) - 4c K_1 K_{3,1} + 2K_2 
		\right] \,, 
\end{aligned}
\end{equation}
where we write $K_{\cdot} = K_{\cdot}(\rho^{\alpha\beta}_\ell)$, 
and the formulas for $K_1, K_2, K_{3,1}, c, M_2$ are calculated in \cref{lm:moments_leaky} and \cref{lm:joint_moments_leaky}. 

Let $\Pi(x, dy), P(x,dy)$ be the Markov transition kernels of $\rho^{\alpha\beta}_\ell$ and $p_\ell$ respectively, then 
\begin{equation}
    d_{BL}( \Pi(x, \cdot), P(x, \cdot) ) = O(n^{-1}) \,, 
    \quad \text{ for all } x \in \mathbb{R} \cup \{\mathbf{e}\} \,. 
\end{equation}
\end{prop}

\begin{rem}
The infinite-width ($n\to\infty$) approximation of 
the Markov chain corresponds to the update 
$q_{\ell+1} = cK_1(q_\ell)$, 
and this is an $O(n^{-1/2})$ approximation to the chain $\{\rho^{\alpha\beta}_\ell\}$. 
On the other hand, the $\{p_\ell\}$ chain we propose is 
an \emph{improved approximation} up to 
the zero mean terms up to $O(n^{-1/2})$, 
and the expected value of non-zero mean terms up to $O(n^{-1})$. 
In the SDE limit of \cref{prop:conv_markov_chain_to_sde}, these are exactly the terms that do not vanish, which leads us to speculate that this approximation is sufficiently close when studying the infinite-depth-and-width limit. 

We will also note that $O(n^{-1})$ error in the result arise from replacing the $O(n^{-1/2})$ with a Gaussian due to Berry--Esseen, 
and the $O(n^{-1})$ term with its expectation, 
as these are the dominant error terms in the approximation. 
\end{rem}

\begin{proof}

We start by defining the notations 
\begin{equation}
    g^{\alpha}_\ell 
    \defequal W_\ell \frac{\varphi^\alpha_\ell}{ |\varphi^\alpha_\ell| } \,, \quad 
    R^{\alpha\beta}_\ell 
    \defequal \frac{1}{\sqrt{n}} \sum_{i=1}^n 
    c \varphi_s(g^\alpha_{\ell,i}) 
    \varphi_s( g^\beta_{\ell,i} ) 
    - c K_1( \rho^{\alpha\beta}_\ell ) \,, 
\end{equation}
and using positive homogeneity we can write 
$\varphi_s( \sqrt{\frac{c}{n}} W_\ell \varphi^\alpha_\ell ) = \sqrt{\frac{c}{n}} |\varphi^\alpha_\ell| \varphi_s(g^\alpha_\ell)$, which gives us 
\begin{equation}
    \langle \varphi^\alpha_{\ell+1}, \varphi^\beta_{\ell+1} \rangle
    = 
    |\varphi^\alpha_{\ell}| \, 
    |\varphi^\beta_{\ell}| 
    \frac{c}{n} \sum_{i=1}^n 
    \varphi_s( g^\alpha_{\ell,i} ) 
    \varphi_s( g^\beta_{\ell,i} ) 
    = 
    |\varphi^\alpha_{\ell}| \, 
    |\varphi^\beta_{\ell}| 
    \left( cK_1( \rho^{\alpha\beta}_\ell ) + \frac{1}{\sqrt{n}} R^{\alpha\beta}_\ell 
    \right) \,. 
\end{equation}

Now consider the same case for $R^{\alpha\alpha}_\ell$ and $R^{\beta\beta}_\ell$ with $K_1(1) = c^{-1}$, we also get 
\begin{equation}
    \rho^{\alpha\beta}_{\ell+1} 
    = 
    \begin{cases}
        \frac{ \langle \varphi^\alpha_{\ell+1}, \varphi^\beta_{\ell+1} \rangle }{ |\varphi^\alpha_{\ell+1}| \, |\varphi^\alpha_{\ell+1}| } 
        =
        \frac{cK_1(\rho^{\alpha\beta}_\ell) + \frac{1}{\sqrt{n} } R^{\alpha\beta}_\ell }{
        \sqrt{ (1 + \frac{1}{\sqrt{n}} R^{\alpha\alpha}_\ell)
        (1 + \frac{1}{\sqrt{n}} R^{\beta\beta}_\ell)} 
        } \,, 
        & \text{ if } |\varphi^\alpha_{\ell+1}| \,, |\varphi^\beta_{\ell+1}| > 0 \,, 
        \\ 
        \mathbf{e} \,, 
        & \text{ otherwise. }
    \end{cases}
\end{equation}

We observe that whenever $|\varphi^\ell_\alpha|>0$, we have that 
$1 + \frac{1}{\sqrt{n}} R^{\alpha\alpha}_\ell = \frac{|\varphi^\alpha_{\ell+1}|^2 }{ |\varphi^\alpha_\ell|^2} \geq 0$. 
Therefore the event $E \defequal \{R^{\alpha\alpha}_\ell,  R^{\beta\beta}_\ell \leq -\sqrt{n}\}$ 
is the same as $\{\rho^{\alpha\beta}_{\ell+1} = \mathbf{e}\}$, 
which is equivalent to when $z^\alpha_{\ell+1}$ or $z^{\beta}_{\ell+1}$ has only non-positive entries. 
When conditioned on the previous layer, all the entries are independent, this event has probability $\Pi(x, \{\mathbf{e}\}) = O(2^{-n})$. 
We will see later that modifying this Markov chain to remove this event will incur only a "minor cost" of $O(2^{-n})$.

Let us fix any realization of $R^{\alpha\alpha}, R^{\beta\beta}, R^{\alpha\beta}$ outside of the event $E$ (i.e. by viewing it as a map $R^{\alpha\alpha}:\Omega \to \mathbb{R}$ from the probability space for some fixed $\omega \in \Omega$), 
we can compute the 
Taylor expansion with respect to $1/\sqrt{n}$ about $0$ (Taylor expansion done using SymPy \cite{sympy} Python package)
\begin{equation}
\label{eq:unshaped_corr_markov_chain_raw}
\begin{aligned}
    \rho^{\alpha\beta}_{\ell+1} 
    &= c K_1(\rho^{\alpha\beta}_\ell) + \frac{1}{\sqrt{n}} \left[ 
		R^{\alpha\beta}_\ell - \frac{cK_1(\rho^{\alpha\beta}_\ell)}{2} 
			(R^{\alpha\alpha}_\ell + R^{\beta\beta}_\ell) 
		\right] 
		\\ 
	&\quad 
		+ \frac{1}{n} \left[ \frac{c K_1(\rho^{\alpha\beta}_\ell)}{8} 
			( 3 (R^{\alpha\alpha}_\ell + R^{\beta\beta}_\ell)^2 
				- 4 R^{\alpha\alpha}_\ell R^{\beta\beta}_\ell ) 
			- \frac{1}{2} R^{\alpha\beta}_\ell ( R^{\alpha\alpha}_\ell + R^{\beta\beta}_\ell )
		\right] 
		+ O(n^{-3/2}) 
		\,, 
\end{aligned}
\end{equation}
where we recall the $X = O(n^{-3/2})$ notation denotes a random variable (the Taylor remainder term) where all moments of $n^{3/2} X$ are bounded by a constant independent of $n$. 

We can simplify these terms further by computing the mean and variance of the expansion (without conditioning on $E^c$). 
More specifically, each of the $R^{\alpha\alpha}_\ell, R^{\beta\beta}_\ell, R^{\alpha\beta}_\ell$ have zero mean and covariance 
\begin{equation}
    \cov_\ell \left( 
    \begin{bmatrix}
        R^{\alpha\alpha}_\ell \\ 
        R^{\beta\beta}_\ell \\ 
        R^{\alpha\beta}_\ell 
    \end{bmatrix}
    \right)
    = 
    \begin{bmatrix}
		M_2 & c^2 K_2 - 1 & c^2 K_{3,1} - c K_1 \\ 
		c^2 K_2 - 1 & M_2 & c^2 K_{3,1} - c K_1 \\ 
		c^2 K_{3,1} - c K_1 & c^2 K_{3,1} - c K_1 & c^2 (K_2 - K_1^2)
	\end{bmatrix}
	\,, 
\end{equation}
where we recall $\cov_\ell$ is the conditional covariance given the sigma-algebra $\mathcal{F}_\ell$ generated by the $\ell$-th layer $[z^\alpha_\ell]_{\alpha=1}^m$. 
We can now recover the desired result by calculating the drift and variance coefficients using SymPy \cite{sympy} again 
\begin{equation}
\begin{aligned}
	\sigma_{\text{ReLU}}^2(\rho^{\alpha\beta}_\ell) 
	&:= \mathbb{E}_\ell \, \left[ 
		R^{\alpha\beta}_\ell - \frac{c K_1}{2} 
			(R^{\alpha\alpha}_\ell + R^{\beta\beta}_\ell) 
		\right]^2 
		\\ 
	&= 
		\frac{c^2}{2} \left[ 
		K_1^2 ( c^2 K_2 + M_2 + 1 ) - 4c K_1 K_{3,1} + 2K_2 
		\right] 
		\,, \\ 
	\mu_{\text{ReLU}}(\rho^{\alpha\beta}_\ell) 
	&:= \mathbb{E}_\ell \, \left[ \frac{c K_1}{8} 
			( 3 (R^{\alpha\alpha}_\ell + R^{\beta\beta}_\ell)^2 
				- 4 R^{\alpha\alpha}_\ell R^{\beta\beta}_\ell ) 
			- \frac{1}{2} R^{\alpha\beta}_\ell ( R^{\alpha\alpha}_\ell + R^{\beta\beta}_\ell )
		\right] 
		\\ 
	&= 
		\frac{c}{4} \left[ K_1 ( c^2 K_2 + 3 M_2 + 3 ) - 4c K_{3,1} 
		\right] \,, 
\end{aligned}
\end{equation}
where we recall $\mathbb{E}_\ell[\,\cdot\,] = \mathbb{E}[\,\cdot\,|\mathcal{F}_\ell]$ is the conditional expectation given the sigma-algebra $\mathcal{F}_\ell$ generated by the $\ell$-th layer $[z^\alpha_\ell]_{\alpha=1}^m$. 

This allows us to write (considering the well defined case) 
\begin{equation}
    \rho^{\alpha\beta}_{\ell+1} = c K_1(\rho^{\alpha\beta}_\ell) + \frac{\sigma_{\text{ReLU}}(\rho^{\alpha\beta}_\ell) }{\sqrt{n}} \xi_\ell 
    + 
    \frac{\mu_{\text{ReLU}}(\rho^{\alpha\beta}_\ell)
        + \eta(\rho^{\alpha\beta}_\ell) }{n} 
    + O(n^{-3/2}) \,, 
\end{equation}
where $\xi_\ell$ is has zero mean and unit variance (when not conditioned on $\rho^{\alpha\beta}_{\ell+1} = \mathbf{e}$), and $\eta(\rho^{\alpha\beta}_\ell)$ has zero mean. 
Observe that there are three differences between $\{\rho^{\alpha\beta}_\ell\}$ and the approximate chain $\{p_\ell\}$: 
\begin{enumerate}
    \item $\rho^{\alpha\beta}_{\ell+1} = \mathbf{e}$ with probability $O(2^{-n})$, 
    \item $\xi_\ell$ is replaced by $z_\ell \sim \cN(0,1)$, 
    \item $\eta(\rho^{\alpha\beta}_\ell)$ and the higher order $O(n^{-3/2})$ terms in the Taylor expansion are removed. 
\end{enumerate}

To complete the proof, we will need to control these differences in terms of the bounded Lipschitz distance on the Markov transition kernels. 
To this goal, we let $h$ be such that $\|h\|_{BL} \leq 1$, hence it must be both bounded by $1$ and at worst $1$-Lipschitz. 
We will first condition on $E^c$ to write the Taylor expansion, and then ``uncondition'' to recover the original distribution, both at a cost of an $O(2^{-n})$ error term. 
More precisely, we will write 
\begin{equation}
\begin{aligned}
    &\mathbb{E}_\ell \, h( \rho^{\alpha\beta}_{\ell+1} ) 
    \\ 
    &= 
    \mathbb{E}_\ell [
        h( \rho^{\alpha\beta}_{\ell+1} ) | E^c 
        ] \, 
    \mathbb{P}_\ell( E^c )
    + 
    \mathbb{E}_\ell [ h( \mathbf{e} ) | E ] \, 
    \mathbb{P}_\ell(E) 
    \\ 
    &= 
        \mathbb{E}_\ell [ 
        h( \rho^{\alpha\beta}_{\ell+1} ) | E^c ]
        \, 
        \mathbb{P}_\ell(E^c) 
        + 
        O(1) 
        O(2^{-n}) 
        \\ 
    &= 
        \mathbb{E}_\ell \left[ \left.  
        h\left( 
        c K_1(\rho^{\alpha\beta}_\ell) 
        + \frac{\sigma_{\text{ReLU}}(\rho^{\alpha\beta}_\ell) }{\sqrt{n}} \xi_\ell 
        + \frac{\mu_{\text{ReLU}}(\rho^{\alpha\beta}_\ell) + \eta(\rho^{\alpha\beta}_\ell) }{n} 
        + O(n^{-3/2})
        \right)
        \right| E^c \right] \, 
        \mathbb{P}_\ell(E^c) 
        + O(2^{-n}) \,, 
\end{aligned}
\end{equation}
where we recall $\mathbb{E}_\ell[\,\cdot\,] = \mathbb{E}[\,\cdot\,|\mathcal{F}_\ell]$, 
and we define $\mathbb{P}_\ell(E) \defequal \mathbb{E}_\ell \mathds{1}_E$. 

At this point we observe that we can now ``uncondition'' the Taylor expansion by essentially doing the same trick, or more precisely observe that 
\begin{equation}
    \mathbb{E}_\ell \left[ \left.  
        h\left( 
        c K_1(\rho^{\alpha\beta}_\ell) 
        + \frac{\sigma_{\text{ReLU}}(\rho^{\alpha\beta}_\ell) }{\sqrt{n}} \xi_\ell 
        + \frac{\mu_{\text{ReLU}}(\rho^{\alpha\beta}_\ell) + \eta(\rho^{\alpha\beta}_\ell) }{n} 
        + O(n^{-3/2})
        \right)
        \right| E \right] \, 
    \mathbb{P}_\ell(E) 
    = O(2^{-n}) \,, 
\end{equation}
therefore we can write 
\begin{equation}
\begin{aligned}
    &\mathbb{E}_\ell \, h( \rho^{\alpha\beta}_{\ell+1} ) 
    \\ 
    &= 
    \mathbb{E}_\ell \left[ \left.  
        h\left( \cdots \right)
        \right| E^c \right] \, 
    \mathbb{P}_\ell(E^c) 
    +
    \mathbb{E}_\ell \left[ \left.  
        h\left( \cdots \right)
        \right| E \right] \, 
    \mathbb{P}_\ell(E) 
    + 
    O(2^{-n}) 
    \\ 
    &= 
    \mathbb{E}_\ell \, h\left( 
    c K_1(\rho^{\alpha\beta}_\ell) 
        + \frac{\sigma_{\text{ReLU}}(\rho^{\alpha\beta}_\ell) }{\sqrt{n}} \xi_\ell 
        + \frac{\mu_{\text{ReLU}}(\rho^{\alpha\beta}_\ell) + \eta(\rho^{\alpha\beta}_\ell) }{n} 
        + O(n^{-3/2}) 
    \right) 
    + 
    O(2^{-n}) \,. 
\end{aligned}
\end{equation}

Since $h$ is $1$-Lipschitz, we have that $h(x+y) \leq h(x) + |y|$, and therefore we can write 
\begin{equation}
\begin{aligned}
    \mathbb{E}_\ell \, h( \rho^{\alpha\beta}_{\ell+1} ) 
    &\leq 
    \mathbb{E}_\ell \, h\left( 
        c K_1(\rho^{\alpha\beta}_\ell) 
        + \frac{\sigma_{\text{ReLU}}(\rho^{\alpha\beta}_\ell) }{\sqrt{n}} z_\ell 
        + \frac{\mu_{\text{ReLU}}(\rho^{\alpha\beta}_\ell)}{n} 
    \right) \\ 
    &\quad
    + 
    \mathbb{E}_\ell \, \frac{\sigma_{\text{ReLU}}(\rho^{\alpha\beta}_\ell) }{\sqrt{n}} | \xi_\ell - z_\ell | 
    + 
    \mathbb{E}_\ell \, 
    \frac{|\eta( \rho^{\alpha\beta}_\ell )|}{n} 
    + 
    O(n^{-3/2} + 2^{-n}) \,. 
\end{aligned}
\end{equation}

Observe that the first term is exactly the transition kernel of $p_\ell$ applied to $h$,  
i.e. $\mathbb{E}_\ell \, h(p_{\ell+1}) = \int h(y) \, P(p_\ell, dy)$, 
which means it's sufficient to control the leftover terms at order $O(n^{-1})$ 
for a chosen coupling of $\xi_\ell$ and $z_\ell$. 
Since clearly $\mathbb{E}_\ell \, \eta(\rho^{\alpha\beta}_\ell) = O(1)$ 
as it does not depend on $n$, 
we just need to show $\mathbb{E}_\ell \, |\xi_\ell - z_\ell| = O(n^{-1/2})$. 
Observe that by definition, we have 
\begin{equation}
    \xi_\ell 
    = 
    \frac{1}{\sqrt{n}} 
    \sum_{i=1}^n 
    \frac{1}{\sigma_{ \text{ReLU} }(\rho^{\alpha\beta}_\ell) } 
    \left[ 
    c \varphi_s(g^\alpha_{\ell,i}) 
    \varphi_s( g^\beta_{\ell,i} ) 
    - c K_1( \rho^{\alpha\beta}_\ell ) 
    - \frac{c K_1(\rho^{\alpha\beta}_\ell)}{2} 
    \left( 
    c \varphi_s(g^\alpha_{\ell,i})^2 
    + c \varphi_s(g^\beta_{\ell,i})^2 
    - 2 
    \right) 
    \right] \,, 
\end{equation}
where the terms of the sum are iid with zero mean and unit variance 
(since each neuron is independent conditioned on the previous layer). 
Therefore, we can invoke a standard $L^1$ Berry--Esseen bound, 
e.g. Theorem 4.2 of \cite{chen2011normal}. 
In this case, we let $F$ be the CDF of $\xi_\ell$ and $G$ be the CDF of $z_\ell$, 
and by duality of $L^1$ (equation 4.6 of \cite{chen2011normal}) we have that 
\begin{equation}
    \inf \mathbb{E} \, |\xi_\ell - z_\ell| 
    = 
    \| F - G \|_{L^1} 
    \leq 
        O( n^{-1/2} ) \,, 
\end{equation}
where the $\inf$ is over all couplings of $\xi_\ell, z_\ell$. 

Finally since the above results do not depend on the choice of the test function $h$, 
so we have that 
\begin{equation}
    d_{BL}( \Pi(x,\cdot), P(x,\cdot) ) 
    = 
    \sup_{\|h\|_{BL} \leq 1} 
    \mathbb{E}_\ell \, \left( 
        h( \rho^{\alpha\beta}_{\ell+1} ) 
        - h( p_{\ell+1} ) 
        \right) 
    \leq O(n^{-1}) \,, 
\end{equation}
which is the desired result. 

\end{proof}

\section{Proofs for ReLU Shaping Results}
\label{sec:app_relu_proofs}

In this section, we first recall the ReLU-like activation function for $s = (s_+, s_-) \in \mathbb{R}^2$ defined as 
\begin{equation}
	\varphi_s(x) \defequal s_+ \max(x,0) + s_- \min(x, 0) 
		= s_+ \varphi(x) - s_- \varphi(-x) \,, 
\end{equation}
where $\varphi(x) \defequal \max(x,0)$ is the usual ReLU activation. 

We will also recall the definitions  
\begin{equation}
\begin{aligned}
	\bar J_{p,r}(\rho) \defequal \mathbb{E} \, \varphi(g)^p \varphi(\hat g)^r \,, \quad 
	K_{p,r}(\rho) \defequal \mathbb{E} \, \varphi_s(g)^p \varphi_s(\hat g)^r \,, 
\end{aligned}
\end{equation}
where $g, w$ are iid $\cN(0,1)$ and we define $\hat g = \rho g + qw$ with $q = \sqrt{1-\rho^2}$.
We will also use the short hand notation to write 
$\bar J_p := \bar J_{p,p}, K_p := K_{p,p}$. 

Here we recall from \cite{cho2009kernel} 
\begin{equation}
    \bar J_1(\rho) = \frac{ \sqrt{1-\rho^2} + (\pi - \arccos \rho) \rho }{ 2\pi } \,. 
\end{equation}

We will also recall from \cref{lm:moments_leaky} the following moment calculations 
\begin{equation}
\label{eq:moments_relu}
\begin{aligned}
	c^{-1}
	&= \mathbb{E} \, \varphi_s(g)^2 
	= 
		\frac{s_+^2 + s_-^2}{2} \,, \\ 
	K_1(\rho) 
	&= \mathbb{E} \, \varphi_s(g) \, \varphi_s(\hat g) 
	= 
		(s_+^2 + s_-^2) \bar J_1(\rho) - 2 s_+ s_- \bar J_1(- \rho) \,. 
\end{aligned}
\end{equation}

In the shaped case, we will calculate a Taylor expansion for the function $c K_1(\rho)$. 

\begin{lem}
[Shaping Correlation Function Expansion]
\label{lm:shape_nu_expansion}
Let $s_\pm = 1 + \frac{c_\pm}{\sqrt{n}}$, 
then
\begin{equation}
    c K_1(\rho) = \rho + \frac{\nu(\rho)}{n} + O(n^{-3/2}) \,, 
\end{equation}
where $\nu(\rho) = \frac{(c_+ - c_-)^2}{2\pi} \left( \sqrt{1-\rho^2} + \rho \arccos \rho \right)$. 
\end{lem}

\begin{proof}

We start by consider plugging in the formula from \cref{eq:moments_relu} to get 
\begin{equation}
\begin{aligned}
    c K_1(\rho) 
    &= 
    \frac{2}{s_+^2 + s_-^2} 
    \left( 
    (s_+^2 + s_-^2) \bar J_1(\rho) - 2 s_+ s_- \bar J_1(- \rho) 
    \right) \\ 
    &= 
    \frac{2}{s_+^2 + s_-^2} 
    \frac{1}{2\pi} 
    \left( 
    (s_+^2 + s_-^2) 
    \left( \sqrt{1-\rho^2} + (\pi - \arccos \rho) \rho \right)
    - 2 s_+ s_- 
    \left( \sqrt{1-\rho^2} - (\pi - \arccos (-\rho)) \rho \right)
    \right) \\
    &= 
    \frac{2}{s_+^2 + s_-^2} 
    \frac{1}{2\pi} 
    \left( 
    (s_+^2 + s_-^2) 
    \left( \sqrt{1-\rho^2} + (\pi - \arccos \rho) \rho \right)
    - 2 s_+ s_- 
    \left( \sqrt{1-\rho^2} - (\arccos \rho) \rho \right)
    \right)
    \,. 
\end{aligned}
\end{equation}
where we used the fact that $\arccos(-\rho) = \pi - \arccos(\rho)$. 

After substituting $s_\pm = 1 + \frac{c_\pm}{\sqrt{n}}$, 
we can use SymPy \cite{sympy} to Taylor expand with respect to the variable $x = n^{-1/2}$ about $x_0=0$ and get 
\begin{equation}
\begin{aligned}
    c K_1(\rho) 
    &= 
    \frac{\rho \arccos{\left(\rho \right)}}{\pi} + \frac{\rho \left(\pi - \arccos{\left(\rho \right)}\right)}{\pi} 
    \\ 
    & \quad + 
    \left(n^{-1/2}\right)^{2} 
    \Bigg(
    \frac{- \rho c_{+}^{2} \arccos{\left(\rho \right)} + 2 \rho c_{+} c_{-} \arccos{\left(\rho \right)} - \rho c_{-}^{2} \arccos{\left(\rho \right)} }{2\pi} 
    \\ 
    &\qquad\qquad\qquad\qquad + 
    \frac{
    c_{+}^{2} \sqrt{1 - \rho^{2}} - 2 c_{+} c_{-} \sqrt{1 - \rho^{2}} + c_{-}^{2} \sqrt{1 - \rho^{2}}}{2 \pi}
    \Bigg) 
    \\ 
    & \quad + 
    O\left(\left(n^{-1/2}\right)^{3}\right)
    \,, 
\end{aligned}
\end{equation}
where we used the simplify function on the coefficients to reduce the size of the expression. 

We can further simplify to get 
\begin{equation}
    c K_1(\rho) 
    = 
    \rho 
    + \frac{1}{n} \frac{(c_+ - c_-)^2}{2\pi} 
    \left( \sqrt{1-\rho^2} - \rho \arccos \rho \right) 
    + 
    O( n^{-3/2} ) \,, 
\end{equation}
which is the desired result. 

\end{proof}

We will also need an approximation result for fourth moments. 

\begin{lem}
[Fourth Moment Approximation]
\label{lm:relu_fourth_moment}
Let $g^\alpha, g^\beta, g^\gamma, g^\delta \in \mathbb{R}$ be jointly Gaussian such that 
\begin{equation}
    \begin{bmatrix}
    g^\alpha \\ g^\beta 
    \end{bmatrix}
    \sim 
    \mathcal{N} \left( 
    0 \,, 
    \begin{bmatrix}
    1 & \rho^{\alpha\beta} \\ 
    \rho^{\alpha\beta} & 1 
    \end{bmatrix}
    \right) \,, 
\end{equation}
and similarly for other pairs of $\alpha,\beta,\gamma,\delta$. 
Then 
\begin{equation}
    \mathbb{E} \, \varphi_s(g^\alpha) 
    \varphi_s(g^\beta) 
    \varphi_s(g^\gamma) 
    \varphi_s(g^\delta) 
    = \mathbb{E} \, g^{\alpha} 
    g^{\beta} g^{\gamma} g^{\delta} 
    + O(n^{-1/2}) 
    = \rho^{\alpha\beta} \rho^{\gamma\delta} 
    + \rho^{\alpha\gamma} \rho^{\beta\delta} 
    + \rho^{\alpha\delta} \rho^{\beta\gamma} 
    + O(n^{-1/2}) 
    \,, 
\end{equation}
where the constant in the $O(\cdot)$ notation is universal. 
\end{lem}

\begin{proof}

We start by writing 
\begin{equation}
    \varphi_s(x) 
    = x + \frac{1}{\sqrt{n}} \left( 
        c_+ \varphi(x) 
        - c_- \varphi(-x) 
    \right) \,, 
\end{equation}
and this allows us to write 
\begin{equation}
    \mathbb{E} \, \varphi_s(g^\alpha) 
    \varphi_s(g^\beta) 
    \varphi_s(g^\gamma) 
    \varphi_s(g^\delta) 
    = \mathbb{E} \, g^{\alpha} 
    g^{\beta} g^{\gamma} g^{\delta} 
    + O(n^{-1/2}) \,. 
\end{equation}

Then by Isserlis' Theorem, we can write 
\begin{equation}
    \mathbb{E} \, g^{\alpha} 
    g^{\beta} g^{\gamma} g^{\delta} 
    = 
    \mathbb{E} \, g^{\alpha} 
    g^{\beta} 
    \mathbb{E} \, g^{\gamma} 
    g^{\delta} 
    + \mathbb{E} \, g^{\alpha} 
    g^{\gamma} 
    \mathbb{E} \, g^{\beta} 
    g^{\delta} 
    + \mathbb{E} \, g^{\alpha} 
    g^{\delta} 
    \mathbb{E} \, g^{\beta} 
    g^{\gamma} \,, 
\end{equation}
which gives us the desired result.

\end{proof}

We will also calculate a useful covariance. 

\begin{lem}
[Covariance of $R^{\alpha\beta}$]
\label{lm:cov_R_ab}
Let $g^\alpha, g^\beta, g^\gamma, g^\delta \in \mathbb{R}^n$ be jointly Gaussian vectors such that 
\begin{equation}
    \begin{bmatrix}
    g^\alpha \\ g^\beta 
    \end{bmatrix}
    \sim 
    \mathcal{N} \left( 
    0 \,, 
    \begin{bmatrix}
    1 & \rho^{\alpha\beta} \\ 
    \rho^{\alpha\beta} & 1 
    \end{bmatrix}
    \otimes I_n 
    \right) \,, 
\end{equation}
and similarly for other pairs of $\alpha,\beta,\gamma,\delta$. 
If we also define 
\begin{equation}
    R^{\alpha\beta} 
    \defequal 
    \frac{1}{\sqrt{n}} 
    \sum_{i=1}^{n} 
    \left[
    c \varphi_s( g^\alpha_{i} ) 
    \varphi_s( g^\beta_{i} ) 
    - c K_1( \rho^{\alpha\beta} )
    \right] \,, 
\end{equation}
then we have the following covariance formula: 
\begin{equation}
    \mathbb{E} \, R^{\alpha\beta} R^{\gamma\delta} 
    = 
    \rho^{\alpha\gamma} \rho^{\beta\delta} 
    + \rho^{\alpha\delta} \rho^{\beta\gamma} 
    + O( n^{-1/2} ) \,. 
\end{equation}
\end{lem}

\begin{proof}

We first observe that since each entry of the sum in $R^{\alpha\beta}$ are iid and zero mean, it is sufficient to just compute the covariance a single term. 
In other words 
\begin{equation}
    \mathbb{E} \, R^{\alpha\beta} R^{\gamma\delta} 
    = \mathbb{E} \, c^2 
    \left( \varphi_s( g^\alpha_{i} ) 
    \varphi_s( g^\beta_{i} ) 
    - K_1( \rho^{\alpha\beta} )
    \right) \, 
    \left( \varphi_s( g^\gamma_{i} ) 
    \varphi_s( g^\delta_{i} ) 
    - K_1( \rho^{\gamma\delta} )
    \right) \,. 
\end{equation}

Since $c = 1 + O(n^{-1/2})$ and 
$K_1(\rho) = \rho + O(n^{-1})$ from \cref{lm:shape_nu_expansion}, 
we can further write this as 
\begin{equation}
    \mathbb{E} \, R^{\alpha\beta} R^{\gamma\delta} 
    = 
    \mathbb{E} 
    \left( \varphi_s( g^\alpha_{i} ) 
    \varphi_s( g^\beta_{i} ) 
    - \rho^{\alpha\beta} 
    \right) \, 
    \left( \varphi_s( g^\gamma_{i} ) 
    \varphi_s( g^\delta_{i} ) 
    - \rho^{\gamma\delta} 
    \right) 
    + 
    O(n^{-1/2}) \,, 
\end{equation}
and we can use the fourth moment approximation \cref{lm:relu_fourth_moment} to get 
\begin{equation}
\begin{aligned}
    \mathbb{E} \, R^{\alpha\beta} R^{\gamma\delta} 
    &= 
    \rho^{\alpha\beta} \rho^{\gamma\delta} 
    + \rho^{\alpha\gamma} \rho^{\beta\delta} 
    + \rho^{\alpha\delta} \rho^{\beta\gamma} 
    - \rho^{\alpha\beta} \rho^{\gamma\delta} 
    - \rho^{\alpha\beta} \rho^{\gamma\delta} 
    + \rho^{\alpha\beta} \rho^{\gamma\delta} 
    + O(n^{-1/2}) 
    \\ 
    &= \rho^{\alpha\gamma} \rho^{\beta\delta} 
    + \rho^{\alpha\delta} \rho^{\beta\gamma} 
    + O( n^{-1/2} ) \,, 
\end{aligned}
\end{equation}
which is the desired result.

\end{proof}

\subsection{Proof of \saferef{Theorem}{thm:joint_output_m} (Covariance SDE, ReLU)}

\label{subsec:proof_relu_joint_output}

We start by restating the theorem. 

\begin{thm}
[Covariance SDE, ReLU]
Let $V^{\alpha\beta}_\ell \defequal \frac{c}{n} \langle \varphi^\alpha_\ell, \varphi^\beta_\ell \rangle$, 
and define $V_\ell \defequal [ V^{\alpha\beta}_\ell ]_{1\leq \alpha \leq \beta=m}$ to be the upper triangular entries thought of as a vector in $\mathbb{R}^{m(m+1)/2}$.
Then, with $s_\pm = 1 + \frac{c_\pm}{\sqrt{n}}$ as in \cref{def:relu-shaping}, in the limit as $n\to\infty, \frac{d}{n} \to T$, the interpolated process $V_{\lfloor tn \rfloor}$ converges in distribution to the solution of the following SDE in the Skorohod topology of $D_{\bR_+, \bR^{m(m+1)/2}}$
\begin{equation}
    d V_t = b( V_t ) \, dt 
    + \Sigma(V_t)^{1/2} \, dB_t \,, 
    \quad 
    V_0 = \left[ \frac{1}{\nin} 
        \langle x^\alpha, x^\beta \rangle 
    \right]_{1 \leq \alpha \leq \beta \leq m} \,, 
\end{equation}
where we denote $\nu(\rho) \defequal \frac{(c_+ - c_-)^2}{2\pi} \left( \sqrt{1-\rho^2} - \rho \arccos \rho \right), 
\rho^{\alpha\beta}_t \defequal  \frac{V^{\alpha\beta}_t}{ \sqrt{ V^{\alpha\alpha}_t V^{\beta\beta}_t } }$ and write 
\begin{equation}
    b( V_t ) 
    = \left[ \nu\left( 
    \rho^{\alpha\beta}_t
    \right) 
    \sqrt{V^{\alpha\alpha}_t V^{\beta\beta}_t} \right]_{1\leq \alpha \leq \beta \leq m} \,, 
    \quad 
    \Sigma( V_t )
    = \left[ 
        V^{\alpha\gamma}_t V^{\beta\delta}_t 
        + V^{\alpha\delta}_t V^{\beta\gamma}_t 
    \right]_{\alpha\leq\beta, \gamma\leq\delta}
    \,. 
\end{equation}
Furthermore, the output distribution can be described conditional on $V_T$ evaluated at final time $T$
\begin{equation}
    \left[ \zout^\alpha \right]_{\alpha=1}^m | {V_T}
    \overset{d}{=}
    \cN\left( 0 , [ V^{\alpha\beta}_T ]_{\alpha,\beta=1}^m 
    \right) \,. 
\end{equation}
\end{thm}

\begin{proof}

We start by recalling the definitions 
\begin{equation}
    V^{\alpha\beta}_{\ell+1} 
    \defequal 
    \frac{c}{n} \langle \varphi^\alpha_{\ell+1} , 
    \varphi^\beta_{\ell+1} \rangle 
    = \frac{c}{n} \left\langle 
    \varphi_s \left( \sqrt{\frac{c}{n}} W_\ell \varphi^\alpha_\ell 
    \right) , 
    \varphi_s \left( \sqrt{\frac{c}{n}} W_\ell \varphi^\beta_\ell 
    \right)
    \right\rangle \,. 
\end{equation}

At this point, we can define 
\begin{equation}
    g^\alpha_\ell \defequal W_\ell \frac{\varphi^\alpha_\ell}{ |\varphi^\alpha_\ell| } \,, 
\end{equation}
and observe that 
\begin{equation}
    \left. \begin{bmatrix}
    g^\alpha_\ell \\ g^\beta_\ell
    \end{bmatrix} 
    \right| \mathcal{F}_\ell 
    \overset{d}{=} 
    N \left( 
    0 \,, 
    \begin{bmatrix}
    1 & \rho^{\alpha\beta}_\ell \\ 
    \rho^{\alpha\beta}_\ell & 1  
    \end{bmatrix} 
    \otimes 
    I_n 
    \right) \,, 
\end{equation}
where $\mathcal{F}_\ell$ is the sigma-algebra generated by $[z^{\alpha}_\ell]_{\alpha=1}^m$, 
$\rho^{\alpha\beta}_\ell \defequal \frac{ \langle \varphi^\alpha_\ell , \varphi^\beta_\ell \rangle }{ |\varphi^\alpha_\ell| \, |\varphi^\beta_\ell| }$, 
and $\otimes$ denotes the Kronecker product. 
Then we can use positive homogeneity (i.e. $\varphi_s(cx) = |c| \varphi_s(x)$) to write 
\begin{equation}
\begin{aligned}
    V^{\alpha\beta}_{\ell+1} 
    &= 
    \frac{c}{n} |\varphi^\alpha_\ell| \, 
    |\varphi^\beta_\ell| \, 
    \frac{c}{n} 
    \langle \varphi_s( g^\alpha_\ell ) , 
    \varphi_s( g^\beta_\ell ) \rangle 
    \\ 
    &= 
    \sqrt{ V^{\alpha\alpha}_\ell V^{\beta\beta}_\ell } \, 
    \left( c K_1( \rho^{\alpha\beta}_\ell ) 
    + \frac{1}{\sqrt{n}} \frac{1}{\sqrt{n}} 
    \sum_{i=1}^{n} 
    \left[ c \varphi_s( g^\alpha_{\ell,i} ) 
    \varphi_s( g^\beta_{\ell,i} ) 
    - c K_1( \rho^{\alpha\beta}_\ell ) 
    \right] 
    \right) \\ 
    &=: 
    \sqrt{ V^{\alpha\alpha}_\ell V^{\beta\beta}_\ell }
    \left( c K_1( \rho^{\alpha\beta}_\ell ) 
    + \frac{1}{\sqrt{n}} 
    R^{\alpha\beta}_\ell 
    \right) \,, 
\end{aligned}
\end{equation}
where we defined $R^{\alpha\beta}_\ell \defequal 
\frac{1}{\sqrt{n}} 
    \sum_{i=1}^{n} 
    \left[
    c \varphi_s( g^\alpha_{\ell,i} ) 
    \varphi_s( g^\beta_{\ell,i} ) 
    - c K_1( \rho^{\alpha\beta}_\ell )
    \right]$. 

Next we use the expansion of $c K_1(\rho^{\alpha\beta}_\ell)$ from \cref{lm:shape_nu_expansion} to write 
\begin{equation}
\begin{aligned}
    V^{\alpha\beta}_{\ell+1} 
    &= 
    \sqrt{ V^{\alpha\alpha}_\ell V^{\beta\beta}_\ell }
    \left( \rho^{\alpha\beta}_\ell 
    + \frac{ \nu(\rho^{\alpha\beta}_\ell) }{n} 
    + \frac{1}{\sqrt{n}} 
    R^{\alpha\beta}_\ell 
    \right) 
    + O(n^{-3/2}) 
    \\ 
    &= 
    V^{\alpha\beta}_\ell 
    + 
    \frac{1}{n} 
    \nu(\rho^{\alpha\beta}_\ell) 
    \sqrt{ V^{\alpha\alpha}_\ell V^{\beta\beta}_\ell } 
    + 
    \frac{1}{\sqrt{n}} 
    \sqrt{ V^{\alpha\alpha}_\ell V^{\beta\beta}_\ell } R^{\alpha\beta}_\ell 
    + 
    O(n^{-3/2}) 
    \,, 
\end{aligned}
\end{equation}
which essentially recovers the Markov chain form we want from \cref{prop:conv_markov_chain_to_sde}, where the drift is 
\begin{equation}
    b( V ) 
    = \nu\left(\rho^{\alpha\beta}_\ell \right) 
    \sqrt{ V^{\alpha\alpha}_\ell V^{\beta\beta}_\ell }
    \,, 
\end{equation}
as desired. 

It remains to simply compute the covariance conditioned on previous layer. 
To this end, we will use \cref{lm:cov_R_ab} to write 
\begin{equation}
\begin{aligned}
    \Sigma(V_\ell)_{\alpha\beta,\gamma\delta} 
    &= 
    \mathbb{E}_\ell \left[ 
    \sqrt{ V^{\alpha\alpha}_\ell V^{\beta\beta}_\ell } R^{\alpha\beta}_\ell 
    \sqrt{ V^{\gamma\gamma}_\ell V^{\delta\delta}_\ell } R^{\gamma\delta}_\ell 
    \right]
    \\ 
    &= 
    \sqrt{ V^{\alpha\alpha}_\ell V^{\beta\beta}_\ell
    V^{\gamma\gamma}_\ell V^{\delta\delta}_\ell } 
    \left( \rho^{\alpha\gamma}_\ell \rho^{\beta\delta}_\ell 
    + \rho^{\alpha\delta}_\ell \rho^{\beta\gamma}_\ell
    + O(n^{-1/2}) 
    \right) 
    \\ 
    &= 
    V^{\alpha\gamma}_\ell V^{\beta\delta}_\ell 
    + V^{\alpha\delta}_\ell V^{\beta\gamma}_\ell 
    + O(n^{-1/2}) \,, 
\end{aligned}
\end{equation}
where we recall $\mathbb{E}_\ell[\,\cdot\,] = \mathbb{E}[\,\cdot\,|\mathcal{F}_\ell]$ is the conditional expectation given the sigma-algebra generated by $\{z^\alpha_\ell\}_{\alpha=1}^m$. 
By setting $\sigma = \Sigma^{1/2}$, we then recover the desired SDE via 
\cref{prop:conv_markov_chain_to_sde} 
on the Markov chain of $V^{\alpha\beta}_\ell$. 

\end{proof}

\subsection{Proof of \saferef{Theorem}{thm:relu_corr} (Correlation SDE, ReLU)}
\label{subsec:proof_relu_corr}

We start by restating the theorem. 

\begin{thm}
[Correlation SDE, ReLU]
Let $\rho^{\alpha\beta}_\ell \defequal 
\frac{\langle \varphi_\ell^\alpha, \varphi_\ell^\beta \rangle}{
|\varphi_\ell^\alpha| \, |\varphi_\ell^\beta|}$, 
where $\varphi^\alpha_\ell \defequal \varphi_s(z^\alpha_\ell)$. 
In the limit as $n\to\infty$ and $s_\pm = 1 + \frac{c_\pm}{\sqrt{n}}$, the interpolated process $\rho^{\alpha\beta}_{\lfloor tn \rfloor}$ converges in distribution to the solution of the following SDE in the Skorohod topology of $D_{\bR_+, \bR}$ 
\begin{equation}
    d \rho^{\alpha\beta}_t 
    = \left[ \nu( \rho^{\alpha\beta}_t ) 
        + \mu( \rho^{\alpha\beta}_t ) \right] 
        \, dt 
        + \sigma( \rho^{\alpha\beta}_t ) \, dB_t \,, 
    \quad 
    \rho^{\alpha\beta}_0 = \frac{\langle x^\alpha, x^\beta \rangle}{ |x^\alpha| \, |x^\beta| } \,, 
\end{equation}
where 
\begin{equation}
	\nu(\rho) 
	= 
		\frac{ (c_+ - c_-)^2 }{ 2\pi } 
		\left[ 
		\sqrt{1 - \rho^2} - \arccos(\rho) \rho 
		\right] 
		\,, \quad 
	\mu(\rho) 
	= - \frac{1}{2} \rho (1 - \rho^2)
		\,, \quad 
	\sigma(\rho) 
	= 
		1-\rho^2 \,. 
\end{equation}
\end{thm}

\begin{proof}

While it is possible to obtain this result as a consequence of \cref{thm:joint_output_m} via It\^o's Lemma, we will show an alternative derivation by extending the steps of \cref{prop:unshaped_relu_corr}, where we can directly compute the Taylor expansion in the event $E \defequal \{ |\varphi_{\ell+1}^\alpha| , |\varphi_{\ell+1}^\beta| > 0 \}$ 
\begin{equation}
\label{eq:corr_markov_chain}
    \rho^{\alpha\beta}_{\ell+1} 
    = 
        \frac{\langle \varphi_{\ell+1}^\alpha, \varphi_{\ell+1}^\beta \rangle}{
        |\varphi_{\ell+1}^\alpha| \, |\varphi_{\ell+1}^\beta|} 
    = 
    c K_1( \rho^{\alpha\beta}_\ell ) 
    + \frac{ \widetilde \mu(\rho^{\alpha\beta}_\ell)}{n} 
    + \sigma(\rho^{\alpha\beta}_\ell) 
    \frac{\xi_\ell}{\sqrt{n}} 
    + O(n^{-3/2})
    \,, 
\end{equation}
where (unconditioned on $E$) $\xi_\ell$ are iid with mean zero variance one and 
\begin{equation}
\begin{aligned}
    \mu(\rho^{\alpha\beta}_\ell) 
    &\defequal 
        \mathbb{E}_\ell \, \widetilde \mu(\rho^{\alpha\beta}_\ell) 
    = 
        \frac{c}{4} \left[ K_1 ( c^2 K_2 + 3 M_2 + 3 ) - 4c K_{3,1} 
		\right] \,, \\ 
    \sigma^2(\rho^{\alpha\beta}_\ell) 
    &\defequal 
        \frac{c^2}{2} \left[ 
		K_1^2 ( c^2 K_2 + M_2 + 1 ) - 4c K_1 K_{3,1} + 2K_2 
		\right] \,, 
\end{aligned}
\end{equation}
where we replaced $\mu_{\text{ReLU}}, \sigma_{\text{ReLU}}$ with $\mu, \sigma$ as we will be shaping the activation function, 
and we recall $\mathbb{E}_\ell[\,\cdot\,] = \mathbb{E}[\,\cdot\,|\mathcal{F}_\ell]$ is the conditional expectation given the sigma-algebra generated by $\{z^\alpha_\ell\}_{\alpha=1}^m$.

We note that the undefined event $E$ occurs only when $z^\alpha_{\ell+1}$ or $z^{\beta}_{\ell+1}$ has all negative entries, which occurs with probability $O(2^{-n})$. 
Since all the terms of interest have finite moments, we can proceed by removing this event $E$ in a similar fashion as \cref{prop:unshaped_relu_corr}. 

Using the expansion of $c K_1(\rho)$ from \cref{lm:shape_nu_expansion}, we can now write 
\begin{equation}
    \rho^{\alpha\beta}_{\ell+1} 
    = 
    \rho^{\alpha\beta}_\ell 
    + \frac{\nu(\rho^{\alpha\beta}_\ell) + \widetilde \mu(\rho^{\alpha\beta}_\ell)}{n} 
    + \sigma(\rho^{\alpha\beta}_\ell) 
    \frac{\xi_\ell}{\sqrt{n}} 
    + O(n^{-3/2}) \,. 
\end{equation}

Furthermore, we also have that by \cref{lm:shape_nu_expansion} and \cref{lm:relu_fourth_moment} 
\begin{equation}
\begin{aligned}
 	K_1 
 	& 
	= \rho^{\alpha\beta}_\ell + O(n^{-1}) 
		\,, \quad
	K_2 
	= 2 (\rho^{\alpha\beta}_\ell)^2 + 1 
	    + O(n^{-1/2}) 
		\,, \\ 
	K_{3,1} 
	&
	= 3 \rho^{\alpha\beta}_\ell 
	    + O(n^{-1/2}) 
		\,, \quad 
	M_2 
	= 2 + O(n^{-1/2}) \,, 
\end{aligned} 
\end{equation}
which gives us the desired formula of 
\begin{equation}
    \mu(\rho) 
	= - \frac{1}{2} \rho (1 - \rho^2)
		\,, \quad 
	\sigma(\rho) 
	= 
		1-\rho^2 \,. 
\end{equation}

Finally, we can recover the desired SDE via \cref{prop:conv_markov_chain_to_sde}.

\end{proof}

\subsection{Joint Correlation SDE}
\label{subsec:app_joint_corr_sde}

In this section, we will extend \cref{thm:relu_corr} to a general joint process over all the possible pairs of correlations. 

\begin{thm}
[Joint Correlation SDE]
\label{thm:joint_corr_sde_m}
Let $\rho^{\alpha\beta}_\ell \defequal \frac{\langle \varphi^\alpha_\ell, \varphi^\beta_\ell \rangle}{ |\varphi^\alpha_\ell| \, |\varphi^\beta_\ell| }$, 
and define $\rho_\ell \defequal [ \rho^{\alpha\beta}_\ell ]_{1\leq \alpha \leq \beta=m}$ to be the upper triangular entries thought of as a vector in $\mathbb{R}^{m(m+1)/2}$.
Then, with $s_\pm = 1 + \frac{c_\pm}{\sqrt{n}}$ as in \cref{def:relu-shaping}, in the limit as $n\to\infty, \frac{d}{n} \to T$, the interpolated process $\rho_{\lfloor tn \rfloor}$ converges in distribution to the solution of the following SDE in the Skorohod topology of $D_{\bR_+, \bR^{m(m+1)/2}}$
\begin{equation}
    d \rho_t = b( \rho_t ) \, dt 
    + \Sigma(\rho_t)^{1/2} \, dB_t \,, 
    \quad 
    \rho_0 = \left[ 
        \frac{ \langle x^\alpha, x^\beta \rangle }{ |x^\alpha| \, |x^\beta| } 
    \right]_{1 \leq \alpha \leq \beta \leq m} \,, 
\end{equation}
where the coefficients are defined by 
\begin{equation}
\begin{aligned}
    b( \rho_t ) 
    &= \left[ 
    \nu( \rho^{\alpha\beta}_t ) 
    + \mu( \rho^{\alpha\beta}_t ) 
    \right]_{1\leq \alpha \leq \beta \leq m} \,, 
    \quad \\ 
    \Sigma( \rho_t )
    &= 
        \bigg[ \rho^{\alpha\gamma} \rho^{\beta\delta} 
		+ \rho^{\alpha\delta} \rho^{\beta\gamma} 
		+ \frac{1}{2} \rho^{\alpha\beta} \rho^{\gamma\delta} 
		\left( ( \rho^{\alpha\gamma} )^2 
			+ ( \rho^{\beta\gamma} )^2 
			+ ( \rho^{\alpha\delta} )^2 
			+ ( \rho^{\beta\delta} )^2 
		\right) 
		\\ 
	&\qquad\qquad\qquad\qquad
		- \rho^{\alpha\beta} 
			\left( \rho^{\alpha\gamma} \rho^{\alpha\delta} 
			+ \rho^{\beta\gamma} \rho^{\beta\delta} \right) 
		- \rho^{\gamma\delta} 
			\left( \rho^{\alpha\gamma} \rho^{\beta\gamma} 
			+ \rho^{\alpha\delta} \rho^{\beta\delta} \right) 
	\bigg]_{ \alpha\leq\beta, \gamma\leq\delta } 
	\,, 
\end{aligned}
\end{equation}
with $\nu,\mu$ defined as in \cref{thm:relu_corr}. 
\end{thm}

\begin{proof}

It's sufficient to just compute the covariance matrix $\Sigma$ for the random terms of the Markov chain \cref{eq:unshaped_corr_markov_chain_raw}, which reduces down to 
\begin{equation}
	\Sigma(\rho_\ell)_{\alpha\beta, \gamma\delta} 
	= \mathbb{E}_\ell \, 
		\left( R^{\alpha\beta}_\ell - \frac{c}{2} K_1^{\alpha\beta} 
			( R^{\alpha\alpha}_\ell + R^{\beta\beta}_\ell )
	 	\right) 
	 	\left( R^{\gamma\delta}_\ell - \frac{c}{2} K_1^{\gamma\delta} 
			( R^{\gamma\gamma}_\ell + R^{\delta\delta}_\ell )
	 	\right) \,, 
\end{equation}
where we recall $\mathbb{E}_\ell[\,\cdot\,] = \mathbb{E}[\,\cdot\,|\mathcal{F}_\ell]$ is the conditional expectation given the sigma-algebra generated by $\{z^\alpha_\ell\}_{\alpha=1}^m$, 
and we write $K_1^{\alpha\beta} := K_1( \rho^{\alpha\beta}_\ell )$. 

Using \cref{lm:shape_nu_expansion} and \cref{lm:cov_R_ab}, we can calculate this explicitly as 
\begin{equation}
\begin{aligned}
	\Sigma(\rho_\ell)_{\alpha\beta, \gamma\delta} 
	&= 
		\mathbb{E}_\ell \, R_{\alpha\beta} R_{\gamma\delta} 
		+ \frac{c^2}{4} K_1^{\alpha\beta} K_1^{\gamma\delta} 
		\mathbb{E}_\ell \, ( R_{\alpha\alpha} + R_{\beta\beta} ) 
			( R_{\gamma\gamma} + R_{\delta\delta} ) 
			\\ 
	&\quad 
		- \frac{c}{2} K^{\alpha\beta} 
		\mathbb{E}_\ell \, R_{\gamma\delta} ( R_{\alpha\alpha} + R_{\beta\beta} ) 
		- \frac{c}{2} K^{\gamma\delta} 
		\mathbb{E}_\ell \, R_{\alpha\beta} ( R_{\gamma\gamma} + R_{\delta\delta} ) 
		\\ 
	&= 
		\rho^{\alpha\gamma} \rho^{\beta\delta} 
		+ \rho^{\alpha\delta} \rho^{\beta\gamma} 
		+ \frac{1}{2} \rho^{\alpha\beta} \rho^{\gamma\delta} 
		\left( ( \rho^{\alpha\gamma} )^2 
			+ ( \rho^{\beta\gamma} )^2 
			+ ( \rho^{\alpha\delta} )^2 
			+ ( \rho^{\beta\delta} )^2 
		\right) 
		\\ 
	&\quad 
		- \rho^{\alpha\beta} 
			\left( \rho^{\alpha\gamma} \rho^{\alpha\delta} 
			+ \rho^{\beta\gamma} \rho^{\beta\delta} \right) 
		- \rho^{\gamma\delta} 
			\left( \rho^{\alpha\gamma} \rho^{\beta\gamma} 
			+ \rho^{\alpha\delta} \rho^{\beta\delta} \right) 
	+ O(n^{-1/2}) 
			\,, 
\end{aligned}
\end{equation}
which is the desired result.

\end{proof}

\subsection{Proof for 
\saferef{Proposition}{prop:relu_critical_exponent} (Critical Exponent, ReLU)}
\label{subsec:proof_relu_critical_exponent}

We start by restating the proposition. 

\begin{prop}
[Critical Exponent, ReLU]
Let $\rho^{\alpha\beta}_\ell \defequal 
\frac{\langle \varphi_\ell^\alpha, \varphi_\ell^\beta \rangle}{
|\varphi_\ell^\alpha| \, |\varphi_\ell^\beta|}$, 
where $\varphi^\alpha_\ell \defequal \varphi_s(z^\alpha_\ell)$. 
Consider the limit $n\to\infty$ and $s_\pm = 1 + \frac{c_\pm}{n^p}$ for some $p\geq 0$. 
Then depending on the value of $p$, the interpolated process $\rho^{\alpha\beta}_{\lfloor tn \rfloor}$ converges in distribution w.r.t. the Skorohod topology of $D_{\bR_+, \bR}$ to 
\begin{enumerate}[label=(\roman*)]
    \item {the degenerate limit:} $\rho^{\alpha\beta}_t = 1$ for all $t>0$, if $0\leq p < \frac{1}{2}$, and $c_+\neq c_-$, 
    \item the critical limit: the SDE from \cref{thm:relu_corr}, if $p=\frac{1}{2}$, 
    \item {the linear network limit:} if $p > \frac{1}{2}$ , the following SDE, with $\mu,\sigma$ as defined in \eqref{eq:greeks},
    
    \begin{equation}
    d \rho^{\alpha\beta}_t 
    = 
        \mu( \rho^{\alpha\beta}_t ) 
        \, dt 
        + \sigma( \rho^{\alpha\beta}_t ) \, dB_t \,, 
    \quad 
    \rho^{\alpha\beta}_0 = \frac{\langle x^\alpha, x^\beta \rangle}{ |x^\alpha| \, |x^\beta| } \,,
    \end{equation} 
    
\end{enumerate}
\end{prop}

\begin{proof}

Case (ii) follows from \cref{thm:relu_corr}, therefore it is sufficient to only consider cases (i) and (iii). 
In the case that $p=0$, we can recover the following recursion in the limit as $n\to\infty$
\begin{equation}
    \rho^{\alpha\beta}_{\ell+1} 
    = c K_1(\rho^{\alpha\beta}_\ell) \,, 
\end{equation}
which matches the infinite-width limit, and it is known that $\rho^{\alpha\beta}_\ell \to 1$ as $\ell \to 1$ (see also \cref{sec:app_independent_interest} for an upper bound). 

Next we will recall the result of \cref{lm:shape_nu_expansion} and observe that we can simply replace $\sqrt{n}$ with $n^p$ to recover the expansion 
\begin{equation}
    c K_1(\rho) 
    = \rho + \frac{ \nu(\rho) }{n^{2p}} 
    + O( n^{-3p} ) \,. 
\end{equation}

This gives us the following Markov chain from the proof of \cref{thm:relu_corr} 
\begin{equation}
    \rho^{\alpha\beta}_{\ell+1} 
    = 
    \rho^{\alpha\beta}_\ell 
    + \frac{\nu(\rho^{\alpha\beta}_\ell)}{n^{2p}}
    + \frac{\mu(\rho^{\alpha\beta}_\ell)}{n} 
    + \sigma(\rho^{\alpha\beta}_\ell) 
    \frac{\xi_\ell}{\sqrt{n}} 
    + O(n^{-3p} + n^{-3/2}) \,. 
\end{equation}

In the case that $0 < p < 1/2$, we can consider the time step size $h_n = n^{-2p}$ instead of $n^{-1}$ and apply \cref{prop:conv_markov_chain_to_sde}, where we recover the ODE 
\begin{equation}
    \partial_s \hat \rho^{\alpha\beta}_s = 
    \nu( \hat \rho^{\alpha\beta}_s) \,, 
\end{equation}
but on the time scale of $\hat \rho^{\alpha\beta,n}_s = \hat \rho^{\alpha\beta}_{\floor{ s n^{2p} }}$. 
Converting it back to the time scale of $\rho^{\alpha\beta,n}_t = \rho^{\alpha\beta}_{\floor{ t n }}$ implies that we have 
\begin{equation}
    \rho^{\alpha\beta}_t = \hat \rho^{\alpha\beta}_\infty \,, 
    \quad 
    \text{ for all } t > 0 \,. 
\end{equation}
And since $\nu(\rho) > 0$ for all $\rho < 1$ and that $\nu(\rho) = C (1-\rho)^{3/2} + O((1-\rho)^{5/2})$ as $\rho\to 1$, we have that $\hat \rho^{\alpha\beta}_\infty = 1$ as desired. 

In the case $p > \frac{1}{2}$, we have that since $\nu$ is deterministic, we observe the drift term used in \cref{prop:conv_markov_chain_to_sde} in the limit as $n\to\infty$ is 
\begin{equation}
    b_n(\rho) = \nu(\rho) n^{1-2p} + \mu(\rho) \to b(\rho) = \mu(\rho) \,, 
\end{equation}
which would simply recover the desired SDE with drift $\mu$ only.

\end{proof}

\section{Proofs for Smooth Shaping Results}
\label{sec:app_smooth_proofs}

In this section, we consider smooth activation functions $\varphi$ satisfying \cref{asm:reg}, that is $\varphi \in C^4, \varphi(0) = 0, \varphi'(0) = 1$, and that $| \varphi^{(4)}(x) | \leq C(1+|x|^p)$ for some $C,p > 0$. 
We recall the shaping we consider for activations of this type is via the following definition for $s > 0$ 
\begin{equation}
	\varphi_s(x) \defequal s \varphi\left( \frac{x}{s} \right) \,, 
\end{equation}
so that $\lim_{s\to\infty} \varphi_s(x) = x$. 

Before we start, 
we will calculate the behaviour of the normalizing constant $c$ 
up an error order of $s^{-3}$. 

\begin{lem}
\label{lm:c_expansion}
Let $\varphi_s$ be defined as above with $\varphi$ satisfying \cref{asm:reg}. 
Then if $g\sim N(0,1)$, we have that 
\begin{equation}
	c = 1 - \frac{1}{s^2} \left( \frac{3}{4} \varphi''(0)^2 
		+ \varphi'''(0) \right) + O(s^{-3}) \,. 
\end{equation}
\end{lem}

\begin{proof}

We will first Taylor expand $\varphi_s(g)$ about $g=0$ 
\begin{equation}
	\varphi_s(g) 
	= 0 + g + \frac{\varphi''(0)}{2s} g^2 
		+ \frac{ \varphi'''(0) }{ 6s^2 } g^3 
		+ O(s^{-3}) \,, 
\end{equation}
where we note by \cref{asm:reg} the remainder term is at most polynomial in $g$. 

Therefore the second moment satisfies 
\begin{equation}
\begin{aligned}
	\mathbb{E} \, \varphi_s(g)^2 
	&= \mathbb{E} \, g^2 + \frac{\varphi''(0)}{s} g^3 
		+ \frac{1}{s^2} \left( \frac{1}{4} \varphi''(0)^2 
		+ \frac{2}{6} \varphi'''(0) \right) g^4 
		+ O(s^{-3}) 
		\\ 
	&= 1 + \frac{1}{s^2} \left( \frac{3}{4} \varphi''(0)^2 
		+ \varphi'''(0) \right) 
		+ O(s^{-3}) \,, 
\end{aligned}
\end{equation}
where $O(s^{-3})$ is bounded due to Gaussians have all bounded moments. 

Therefore, for $s>0$ sufficiently small, we have the following expansion 
\begin{equation}
	c 
	= \frac{1}{ \mathbb{E} \, \varphi_s(g)^2 } 
	= \frac{1}{1 - (-bs^{-2} + O(s^{-3}) )} 
	= 1 - \frac{b}{s^2} + O(s^{-3}) \,, 
\end{equation}
where $ b = \frac{3}{4} \varphi''(0)^2 + \varphi'''(0)$, 
which is the desired result. 

\end{proof}

\subsection{Proof of
\saferef{Theorem}{thm:smooth_joint_output_m} (Covariance SDE, Smooth)}
\label{subsec:proof_smooth_joint_output_m}

We start by restating the theorem. 

\begin{thm}
[Covariance SDE, Smooth]
Let $\varphi$ satisfy \cref{asm:reg}, 
$V^{\alpha\beta}_\ell \defequal \frac{c}{n} \langle \varphi^\alpha_\ell, \varphi^\beta_\ell \rangle$ 
where $\varphi^\alpha_\ell = \varphi_s(z^\alpha_\ell)$, 
and define $V_\ell \defequal [ V^{\alpha\beta}_\ell ]_{1\leq \alpha \leq \beta=m}$ to be the upper triangular entries thought of as a vector in $\mathbb{R}^{m(m+1)/2}$. 
Then, with $s = a\sqrt{n}$ as in \cref{def:smooth-shaping}, in the limit as $n\to\infty, \frac{d}{n} \to T$, the interpolated process $V_{\lfloor tn \rfloor}$ converges locally in distribution to the solution of the following SDE in the Skorohod topology of $D_{\bR_+, \bR^{m(m+1)/2}}$
\begin{equation}
    d V_t = b( V_t ) \, dt 
    + \Sigma(V_t)^{1/2} \, dB_t \,, 
    \quad 
    V_0 = \left[ \frac{1}{\nin} 
        \langle x^\alpha, x^\beta \rangle 
    \right]_{1 \leq \alpha \leq \beta \leq m} \,, 
\end{equation}
where $\Sigma(V_t)$ is the same as \cref{thm:joint_output_m} and 
\begin{equation}
    b^{\alpha \beta}( V_t ) 
    = %
    \frac{\varphi''(0)^2}{4a^2} 
    \left( 
    V^{\alpha\alpha}_t V^{\beta\beta}_t 
    + V^{\alpha\beta}_t ( 2 V^{\alpha\beta}_t - 3 )
    \right) 
    + 
    \frac{\varphi'''(0)}{2a^2} V^{\alpha\beta}_t 
    ( V^{\alpha\alpha}_t + V^{\beta\beta}_t - 2 ) 
    \,. 
\end{equation}

Furthermore, if $V_T$ is finite, then the output distribution can be described conditional on $V_T$ as 
\begin{equation}
    \left[ \zout^\alpha \right]_{\alpha=1}^m | {V_T}
    \overset{d}{=}
    \cN\left( 0  , [ V^{\alpha\beta}_T ]_{\alpha,\beta=1}^m 
    \right) \,, 
\end{equation}
and otherwise the distribution of $[\zout^\alpha]_{\alpha=1}^m$ is undefined. 
\end{thm}

\begin{proof}

We start by defining $g^\alpha_\ell \defequal W_\ell \frac{ \varphi^\alpha_\ell }{ |\varphi^\alpha_\ell| }$, and observe that 
\begin{equation}
    \left. \begin{bmatrix}
    g^\alpha_\ell \\ g^\beta_\ell 
    \end{bmatrix}
    \right| 
    \mathcal{F}_\ell 
    \overset{d}{=} 
    \mathcal{N} 
    \left( 0\,, 
    \begin{bmatrix}
    1 & \rho^{\alpha\beta}_\ell \\ 
    \rho^{\alpha\beta}_\ell & 1 
    \end{bmatrix} 
    \otimes I_n 
    \right) \,, 
\end{equation}
where $\mathcal{F}_\ell$ is the sigma-algebra generated by the $\ell$-th layer $[z^\alpha_\ell]_{\alpha=1}^m$, 
$\rho^{\alpha\beta}_\ell \defequal 
\frac{ \langle \varphi^\alpha_\ell, \varphi^\beta_\ell \rangle }{ 
|\varphi^\alpha_\ell| \, 
|\varphi^\beta_\ell| 
}$, 
and $\otimes$ denotes the Kronecker product. 
We can then write the Taylor expansion for $\varphi_s$ about $0$ as 
\begin{equation}
\begin{aligned}
    \varphi^\alpha_{\ell+1,i} 
    &= 
    \varphi_s \left( \sqrt{\frac{c}{n}} |\varphi^\alpha_\ell| g^\alpha_{\ell,i} 
    \right) \\ 
    &= 
    \varphi_s(0) 
    + \varphi_s'(0) \sqrt{\frac{c}{n}} |\varphi^\alpha_\ell| g^\alpha_{\ell,i} 
    + \frac{\varphi_s''(0)}{2} 
    \left( \sqrt{\frac{c}{n}} |\varphi^\alpha_\ell| g^\alpha_{\ell,i} \right)^2 
    + \frac{\varphi_s'''(0)}{6} 
    \left( \sqrt{\frac{c}{n}} |\varphi^\alpha_\ell| g^\alpha_{\ell,i} \right)^3 
    \\ 
    &\quad + R_3 \left( \sqrt{\frac{c}{n}} |\varphi^\alpha_\ell| g^\alpha_{\ell,i} \right)
    \,, 
\end{aligned}
\end{equation}
where $R_3(\cdot)$ is the Taylor remainder term, which has polynomial growth by \cref{asm:reg}. 

By using the fact that $\varphi(0)=0,\varphi'(0)=1$ and observing that the derivatives of $\varphi_s$ satisfies 
$\varphi^{(k)}_s(0) = \frac{ \varphi^{(k)}(0) }{s^{k-1}}$, 
we can further write 
\begin{equation}
    \varphi^\alpha_{\ell+1,i} 
    = \sqrt{\frac{c}{n}} |\varphi^\alpha_\ell| g^\alpha_{\ell,i} 
    + \frac{\varphi''(0)}{2s} 
    \left( \sqrt{\frac{c}{n}} |\varphi^\alpha_\ell| g^\alpha_{\ell,i} \right)^2 
    + \frac{\varphi'''(0)}{6s^2} 
    \left( \sqrt{\frac{c}{n}} |\varphi^\alpha_\ell| g^\alpha_{\ell,i} \right)^3 
    + O(s^{-3}) \,, 
\end{equation}
where the remainder term is at most polynomial in $g^\alpha_{\ell,i}$. 

Then we can compute the inner product with the same expansion as 
\begin{equation}
\begin{aligned}
	&\frac{c}{n} \langle \varphi_{\ell+1}^\alpha, \varphi_{\ell+1}^\beta \rangle 
	\\ 
	&= \frac{c}{n} \sum_{i=1}^n 
	\left( 
	\sqrt{\frac{c}{n}} |\varphi_\ell^\alpha| g^\alpha_{\ell,i} 
	+ \frac{\varphi''(0)}{2s} \frac{c}{n} |\varphi_\ell^\alpha|^2 ( g^\alpha_{\ell,i} )^2 
	+ \frac{\varphi'''(0)}{6s^2} 
		\left( \frac{c}{n} |\varphi_\ell^\alpha|^2 \right)^{3/2} 
		(g^\alpha_{\ell,i})^3 
	+ O(s^{-3}) 
	\right) 
	\\ 
	&\quad\quad\quad\quad 
	\left( 
	\sqrt{\frac{c}{n}} |\varphi_\ell^\beta| g^\beta_{\ell,i} 
	+ \frac{\varphi''(0)}{2s} \frac{c}{n} |\varphi_\ell^\beta|^2 (g^\beta_{\ell,i})^2 
	+ \frac{\varphi'''(0)}{6s^2} 
		\left( \frac{c}{n} |\varphi_\ell^\beta|^2 \right)^{3/2} 
		(g^\beta_{\ell,i})^3 
	+ O(s^{-3}) 
	\right) 
	\,, 
\end{aligned}
\end{equation}
and we will proceed by analyzing the product terms separately. 
We start with the terms of order $O(s^0)$ first, which are 
\begin{equation}
\begin{aligned}
    \frac{c}{n} \sum_{i=1}^n 
    \frac{c}{n} |\varphi^\alpha_\ell| 
    |\varphi^\beta_\ell| 
    g^\alpha_{\ell,i} 
    g^\beta_{\ell,i} 
    &= 
    \frac{c}{n} |\varphi^\alpha_\ell| 
    |\varphi^\beta_\ell| 
    c \left( \rho^{\alpha\beta}_\ell 
    + \frac{1}{\sqrt{n}} \frac{1}{\sqrt{n}} 
    \sum_{i=1}^n 
    g^\alpha_{\ell,i} 
    g^\beta_{\ell,i} - \rho^{\alpha\beta}_\ell 
    \right) 
    \\ 
    &= 
    \sqrt{V^{\alpha\alpha}_\ell V^{\beta\beta}_\ell} 
    c \left( 
    \rho^{\alpha\beta}_\ell 
    + \frac{1}{\sqrt{n}} R^{\alpha\beta}_\ell 
    \right) \\ 
    &= 
    c V^{\alpha\beta}_\ell 
    + 
    c \sqrt{V^{\alpha\alpha}_\ell V^{\beta\beta}_\ell} 
    \frac{R^{\alpha\beta}_\ell }{\sqrt{n}} 
    \,, 
\end{aligned}
\end{equation}
where we used the definitions 
$V^{\alpha\beta}_\ell \defequal \frac{c}{n} \langle \varphi^\alpha_\ell, \varphi^\beta_\ell \rangle$ and 
$R^{\alpha\beta}_\ell \defequal \frac{1}{\sqrt{n}} \sum_{i=1}^n 
g^\alpha_{\ell,i} 
g^\beta_{\ell,i} - \rho^{\alpha\beta}_\ell$. 

For the first order terms, i.e., terms of order $O(s^{-1})$, we have the terms 
\begin{equation}
\begin{aligned}
    &\frac{c}{n} \sum_{i=1}^n 
    \sqrt{\frac{c}{n}} |\varphi_\ell^\alpha| g^\alpha_{\ell,i} 
	\frac{\varphi''(0)}{2s} \frac{c}{n} |\varphi_\ell^\beta|^2 ( g^\beta_{\ell,i} )^2
    + 
    \sqrt{\frac{c}{n}} |\varphi_\ell^\beta| g^\beta_{\ell,i} 
	\frac{\varphi''(0)}{2s} \frac{c}{n} |\varphi_\ell^\alpha|^2 ( g^\alpha_{\ell,i} )^2
    \\ 
    &=
    \frac{\varphi''(0)}{2s} 
    \sqrt{ V^{\alpha\alpha}_\ell 
    V^{\beta\beta}_\ell } 
    \frac{c}{n} \sum_{i=1}^n 
    g^\alpha_{\ell,i} 
    g^\beta_{\ell,i} 
    \left( 
    \sqrt{V^{\alpha\alpha}_\ell} g^\alpha_{\ell,i}
    + \sqrt{V^{\beta\beta}_\ell} g^\beta_{\ell,i} 
    \right) 
    \\ 
    &= 
    \frac{\varphi''(0)}{2s} 
    \sqrt{ V^{\alpha\alpha}_\ell 
    V^{\beta\beta}_\ell } 
    \frac{c}{\sqrt{n}} \widehat R^{\alpha\beta}_\ell \,, 
\end{aligned}
\end{equation}
where we define 
$\widehat R^{\alpha\beta}_\ell \defequal 
\frac{1}{\sqrt{n}}
\sum_{i=1}^n 
    g^\alpha_{\ell,i} 
    g^\beta_{\ell,i} 
    \left( 
    \sqrt{V^{\alpha\alpha}_\ell} g^\alpha_{\ell,i}
    + \sqrt{V^{\beta\beta}_\ell} g^\beta_{\ell,i} 
    \right)$ 
and observe this random variable has zero mean and a finite variance. 
Therefore in view of \cref{prop:conv_markov_chain_to_sde}, this term cannot contribute to the drift due to having zero mean, nor can this term contribute to the diffusion term due to $s = a \sqrt{n}$ leading to the term being order $\frac{1}{n}$. 
In other words, the effect of this term will vanish in the limit as $n\to\infty$. 

We then turn our attention to the second order terms, i.e., terms of order $O(s^{-2})$ 
\begin{equation} 
\begin{aligned}
    &\frac{c}{n} \sum_{i=1}^n 
	\frac{\varphi''(0)^2}{4s^2} 
	\frac{c}{n} |\varphi_\ell^\alpha|^2 
	\frac{c}{n} |\varphi_\ell^\beta|^2 
	(g^\alpha_{\ell,i})^2 (g^\beta_{\ell,i})^2 
	\\ 
	&+ 
	\frac{\varphi'''(0)}{6s^2} \left( 
		\sqrt{\frac{c}{n}} |\varphi_\ell^\alpha|
		\left( \sqrt{\frac{c}{n}} |\varphi_\ell^\beta| \right)^3 
		g^\alpha_{\ell,i} (g^\beta_{\ell,i})^3 
		+ 
		\left( \sqrt{\frac{c}{n}} |\varphi_\ell^\alpha| \right)^3
		\sqrt{\frac{c}{n}} |\varphi_\ell^\beta| 
		(g^\alpha_{\ell,i})^3 g^\beta_{\ell,i}
	\right) \,. 
\end{aligned}
\end{equation}
Since this term is order $s^{-2} = \frac{1}{a^2 n}$, it can only contribute to the drift term, and in view of \cref{prop:conv_markov_chain_to_sde}, we only need to compute its mean. 
To this goal, we will simply invoke Isserlis' Theorem and calculate 
\begin{equation}
    \mathbb{E}_\ell \, (g^\alpha_{\ell,i})^2 (g^\beta_{\ell,i})^2 
    = 1 + 2 (\rho^{\alpha\beta}_\ell)^2 \,, 
    \quad 
    \mathbb{E}_\ell \, g^\alpha_{\ell,i} (g^\beta_{\ell,i})^3 
    = \mathbb{E}_\ell \, (g^\alpha_{\ell,i})^3 g^\beta_{\ell,i} 
    = 3 \rho^{\alpha\beta}_\ell \,, 
\end{equation}
where we recall $\mathbb{E}_\ell[\,\cdot\,] = \mathbb{E}[\,\cdot\,|\mathcal{F}_\ell]$ is the conditional expectation given the sigma-algebra generated by $\{z^\alpha_\ell\}_{\alpha=1}^m$. 
This allows us to compute the conditional expectation $\mathbb{E}_\ell$ for the terms of order $s^{-2}$ as 
\begin{equation}
\begin{aligned}
    & c \left[
		\frac{\varphi''(0)^2}{4s^2} 
		\frac{c}{n} |\varphi^\ell_\alpha|^2 
		\frac{c}{n} |\varphi^\ell_\beta|^2 
		(1 + 2 (\rho^{\alpha\beta}_\ell)^2 ) 
		+ 
		\frac{\varphi'''(0)}{6s^2} 
		3 \rho^{\alpha\beta} 
		\sqrt{\frac{c}{n}} |\varphi^\ell_\alpha| 
		\sqrt{\frac{c}{n}} |\varphi^\ell_\beta|
		\left( 
			\frac{c}{n} |\varphi^\ell_\alpha|^2 
			+ \frac{c}{n} |\varphi^\ell_\beta|^2 
		\right) 
	\right]
	\\ 
    &= 
    c \left[ 
    \frac{\varphi''(0)^2}{4s^2} 
    (V^{\alpha\alpha}_\ell 
    V^{\beta\beta}_\ell 
    2 (V^{\alpha\beta}_\ell)^2 ) 
    + 
    \frac{\varphi'''(0)}{2s^2} 
    V^{\alpha\beta}_\ell 
    ( V^{\alpha\alpha}_\ell + V^{\beta\beta}_\ell )
    \right] \,, 
\end{aligned}
\end{equation}

Putting these terms together with the fact that 
$c = 1 - \frac{b}{s^2} + O(s^{-3})$ with 
$b = \frac{3}{4} \varphi''(0)^2 + \varphi'''(0)$, we can write the update rule for $V^{\alpha\beta}_\ell$ as 
\begin{equation}
\label{eq:smooth_cov_markov_chain}
\begin{aligned}
    V^{\alpha\beta}_{\ell+1} 
    &= 
    V^{\alpha\beta}_\ell 
    + 
    \frac{1}{n} 
    \left[ 
    \frac{\varphi''(0)^2}{4a^2} 
    \left( 
    V^{\alpha\alpha}_\ell V^{\beta\beta}_\ell 
    + V^{\alpha\beta}_\ell ( 2 V^{\alpha\beta}_\ell - 3 )
    \right) 
    + 
    \frac{\varphi'''(0)}{2a^2} V^{\alpha\beta}_\ell 
    ( V^{\alpha\alpha}_\ell + V^{\beta\beta}_\ell - 2 ) 
    \right] 
    \\ 
    &\quad 
    + 
    c \sqrt{ V^{\alpha\alpha}_\ell V^{\beta\beta}_\ell } 
    \frac{R^{\alpha\beta}_\ell}{\sqrt{n}} 
    + O(n^{-3/2}) \,. 
\end{aligned}
\end{equation}

At this point, we have fully recovered the drift term, and we observe the covariance structure is the same as \cref{lm:cov_R_ab} in the limit as $n\to\infty$. 
Therefore we can invoke \cref{prop:conv_markov_chain_to_sde} to recover the desired SDE.

\end{proof}

\subsection{Proof of \saferef{Proposition}{prop:smooth_critical_exponent} (Critical Exponent, Smooth)}
\label{subsec:proof_smooth_critical_exponent}

We will restate and prove the proposition. 

\begin{prop}
[Critical Exponent, Smooth]
Let $\varphi$ satisfy \cref{asm:reg}, 
$V^{\alpha\beta}_\ell \defequal \frac{c}{n} \langle \varphi^\alpha_\ell, \varphi^\beta_\ell \rangle$ 
where $\varphi^\alpha_\ell = \varphi_s(z^\alpha_\ell)$ with $s = a n^p$ for some $p> 0$,  
and define $V_\ell \defequal [ V^{\alpha\beta}_\ell ]_{1\leq \alpha \leq \beta=m}$ to be the upper triangular entries thought of as a vector.
Then in the limit as $n\to\infty, \frac{d}{n} \to T$, the interpolated process $V_{\floor{tn}}$ converges locally in distribution w.r.t. the Skorohod topology of $D_{\bR_+, \bR^{m(m+1)/2}}$ to $V$, which depending on the value of $p$ is 
\begin{enumerate}[label=(\roman*)]
    \item {the degenerate limit:} if $0<p<\frac{1}{2}$ 
    \begin{equation}
        \begin{cases}
            V^{\alpha\alpha}_{t} = 0 \text{ or } \infty , 
            & 
            \text{ if } \frac{3}{4} \varphi''(0)^2 + \varphi'''(0) > 0 \text{ and } V^{\alpha\alpha}_0 \neq 0 \,, \\ 
            V^{\alpha\beta}_{t} = \text{const.} \,, 
            & 
            \text{ if } \frac{3}{4} \varphi''(0)^2 + \varphi'''(0) \leq 0 
            \,, \\ 
        \end{cases}
    \end{equation}
    for all $t > 0$ and $1\leq \alpha \leq \beta \leq m$, 
    \item the critical limit: the solution of the SDE from \cref{thm:smooth_joint_output_m}, if $p=\frac{1}{2}$, 
    \item {the linear network limit:} the stopped solution to the SDE $d V_t = \Sigma(V_t) \, dB_t$
    with coefficient $\Sigma$ defined in \cref{thm:relu_corr}, 
    if $p > \frac{1}{2}$. 
\end{enumerate}
\end{prop}

\begin{proof}

Similar to the proof of \cref{thm:smooth_joint_output_m}, we will borrow the same notation and write down the Markov chain update and consider the time scale depending on the value of $p$. 
In case (i) where $0 < p < \frac{1}{2}$, we will consider the time scale $h_n = \frac{1}{s^2} = \frac{1}{a^2 n^{2p}}$ and observe that based on the Taylor expansion of $\varphi_s$ about $0$, we can write 
\begin{equation}
\begin{aligned}
    V^{\alpha\alpha}_{\ell+1} 
    &= \frac{c}{n} \sum_{i=1}^n 
        \left( 
        \sqrt{ V^{\alpha\alpha}_\ell } g^{\alpha}_{\ell,i} 
        + 
        \frac{\varphi''(0)}{2s} V^{\alpha\alpha}_\ell 
        ( g^{\alpha}_{\ell,i} )^2 
        + \frac{\varphi'''(0)}{6s^2} 
        ( V^{\alpha\alpha}_\ell )^{3/2} 
        ( g^{\alpha}_{\ell,i} )^3 
        + O(s^{-3}) 
        \right)^2 
        \\ 
    &= 
        c V^{\alpha\alpha}_\ell \frac{1}{n} 
        \sum_{i=1}^n ( g^\alpha_{\ell,i} )^2 
        + (V^{\alpha\alpha}_\ell)^{3/2} \frac{\varphi''(0)}{2s} 
        \frac{c}{n} \sum_{i=1}^n 2 ( g^\alpha_{\ell,i} )^3 \\ 
    &\quad 
        + (V^{\alpha\alpha}_\ell)^2 
        \left( \frac{\varphi'''(0)}{3s^2} + \frac{\varphi''(0)^2}{4s^2} \right) 
        \frac{c}{n} \sum_{i=1}^n ( g^\alpha_{\ell,i} )^4 
        + O(s^{-3}) 
        \\ 
    &= 
        c V^{\alpha\alpha}_\ell 
        + 
        c V^{\alpha\alpha}_\ell \frac{1}{\sqrt{n}} R^{\alpha\alpha}_\ell 
        + 
        c (V^{\alpha\alpha}_\ell)^{3/2} \frac{\varphi''(0)}{s} 
        \frac{1}{\sqrt{n}} \widehat R^{\alpha\alpha}_\ell 
        + 
        c (V^{\alpha\alpha}_\ell)^2 
        \left( \frac{\varphi'''(0)}{s^2} + \frac{3\varphi''(0)^2}{4s^2} \right) 
        \\ 
    &\quad 
        + 
        c (V^{\alpha\alpha}_\ell)^2 
        \left( \frac{\varphi'''(0)}{3s^2} + \frac{\varphi''(0)^2}{4s^2} \right) 
        \frac{1}{\sqrt{n}} \widetilde R^{\alpha\alpha}_\ell 
        + O(s^{-3}) \,, 
\end{aligned}
\end{equation}
where we define $R^{\alpha\alpha}_\ell \defequal \frac{1}{\sqrt{n}} \sum_{i=1}^n ( g^\alpha_{\ell,i} )^2 - 1, 
\widehat R^{\alpha\alpha}_\ell \defequal \frac{1}{\sqrt{n}} \sum_{i=1}^n ( g^\alpha_{\ell,i} )^3, 
\widetilde R^{\alpha\alpha}_\ell \defequal \frac{1}{\sqrt{n}} \sum_{i=1}^n ( g^\alpha_{\ell,i} )^4 - 3$ and observe they all have zero mean and finite variance. 

In view of the time scale $s^{-2}$ for \cref{prop:conv_markov_chain_to_sde}, 
it is then only important to keep track of the expected value of the $s^{-2}$ terms and the covariance of the $s^{-1}$ terms. 
However, since there is no terms on the order of $s^{-1}$, we essentially have 
\begin{equation}
    V^{\alpha\alpha}_{\ell+1} 
    = V^{\alpha\alpha}_\ell 
    + \frac{1}{s^2} 
    \left( \varphi'''(0) + \frac{3}{4} \varphi''(0)^2 \right) V^{\alpha\alpha}_\ell ( V^{\alpha\alpha}_\ell  - 1 ) 
    + O(s^{-3} + n^{-1}) \,, 
\end{equation}
where we used the fact that $c = 1 - \frac{b}{s^2} + O(s^{-3})$ for $b = \varphi'''(0) + \frac{3}{4} \varphi''(0)^2$ from \cref{lm:c_expansion}. 

Hence, we have that 
$U^{\alpha\alpha,n}_{t} \defequal  V^{\alpha\alpha}_{\floor{t s^2}}$ 
converging to the ODE via \cref{prop:conv_markov_chain_to_sde} 
\begin{equation}
    \partial_{t} 
    U^{\alpha\alpha}_{t} 
    = 
    b 
    U^{\alpha\alpha}_{t} 
    ( U^{\alpha\alpha}_{t} - 1 ) \,, 
\end{equation}
where we observe if $b>0$ this ODE is ``mean avoiding'' as it will drift towards $0$ or $\infty$. 
And since the $V_t$ time scale is on the order of $\frac{1}{n}$, for all $t>0$ we have that 
\begin{equation}
    V^{\alpha\alpha}_{t} 
    = U^{\alpha\alpha}_{\infty} \,, 
\end{equation}
therefore if $b > 0$ we have that 
$V^{\alpha\alpha}_{t} = 0$ or $\infty$ as desired in the first case of (i). 
When $b=0$ we observe that $V^{\alpha\alpha}_t = V^{\alpha\alpha}_0$ since the time derivative is zero. 
Furthermore if $b < 0$ we also have that 
$V^{\alpha\alpha}_{t} = 1$ in the second case of (i).

When $b \leq 0$, we can also write down the ODE for $U^{\alpha\beta}_{t}$ 
using a similar argument and keeping only the $s^{-2}$ terms. 
More precisely, we can modify \cref{eq:smooth_cov_markov_chain} to get 
\begin{equation}
\begin{aligned}
    V^{\alpha\beta}_{\ell+1} 
    &= 
    V^{\alpha\beta}_\ell 
    + 
    \frac{1}{s^2} 
    \left[ 
    \frac{\varphi''(0)^2 }{4}
    \left( 
    V^{\alpha\alpha}_\ell V^{\beta\beta}_\ell 
    + V^{\alpha\beta}_\ell ( 2 V^{\alpha\beta}_\ell - 3 )
    \right) 
    + 
    \frac{\varphi'''(0)}{2} V^{\alpha\beta}_\ell 
    ( V^{\alpha\alpha}_\ell + V^{\beta\beta}_\ell - 2 ) 
    \right] 
    \\ 
    &\quad 
    + 
    c \sqrt{ V^{\alpha\alpha}_\ell V^{\beta\beta}_\ell } 
    \frac{R^{\alpha\beta}_\ell}{\sqrt{n}} 
    + O(n^{-3/2}) \,, 
\end{aligned}
\end{equation}
which leads to the following ODE 
\begin{equation}
    \partial_{t} 
    U^{\alpha\beta}_{t} 
    = 
    \frac{\varphi''(0)^2}{4}
    \left( 
    U^{\alpha\alpha}_{t} 
    U^{\beta\beta}_{t} 
    + 
    U^{\alpha\beta}_{t} 
    \left( 2 U^{\alpha\beta}_{t} - 3 \right)
    \right) 
    + 
    \frac{\varphi'''(0)}{2} 
    U^{\alpha\beta}_{t} 
    \left( 
    U^{\alpha\alpha}_{t} 
    + 
    U^{\beta\beta}_{t} 
    - 2 
    \right) 
    \,. 
\end{equation}

Since $U^{\alpha\alpha}_{t}, U^{\beta\beta}_{t}$ converge to constants as $t \to \infty$, 
$|U^{\alpha\beta}_t| \leq \sqrt{ U^{\alpha\alpha}_{t} U^{\beta\beta}_{t} }$ by definition and Cauchy--Schwarz inequality, 
and that $U^{\alpha\beta}_t$ satisfies a first order ODE (so it cannot have a periodic solution), we must also have that $\lim_{t\to\infty} U^{\alpha\beta}_{t} = \text{const.}$
This completes the proof for case (i). 

Case (ii) follows directly from \cref{thm:smooth_joint_output_m}, therefore we can then consider case (iii) with the same Taylor expansion, however this time on the time scale of $n^{-1}$ instead. 
We will again follow \cref{prop:conv_markov_chain_to_sde} to only track the mean of the order $n^{-1}$ term and the variance of the $n^{-1/2}$ term. 
Since $p>\frac{1}{2}$, 
the only term that remains is the diffusion on the order of $n^{-1/2}$ 
\begin{equation}
    V^{\alpha\beta}_{\ell+1} 
    = V^{\alpha\beta}_\ell 
    + V^{\alpha\beta}_\ell \frac{1}{\sqrt{n}} R^{\alpha\beta}_\ell \,, 
\end{equation}
which gives us the desired SDE from calculating the covariance from \cref{thm:smooth_joint_output_m}. 

\end{proof}

\subsection{Proof of \saferef{Proposition}{prop:finite_time_explosion} (Finite Time Explosion Criterion)}
\label{subsec:proof_finite_time_explosion}

We will start by recalling several definitions from \cite[Section 5.5]{karatzas2012brownian}. 
Firstly, we consider the one dimensional It\^o diffusion on $I \defequal (0,\infty)$
\begin{equation}
\label{eq:app_finite_time_explosion_sde}
    dX_t = b(X_t) \, dt + \sigma(X_t) \, dB_t \,, 
\end{equation}
where the drift and diffusion coefficients satisfy the following conditions 
\begin{equation}
\label{eq:nd_li}
\begin{aligned}
    &\sigma^2(x) > 0 \,, \forall x \in I \,, \\ 
    &\forall x \in I \,, \exists \epsilon > 0 : 
    \int_{x-\epsilon}^{x+\epsilon} 
    \frac{ |b(y)| }{ \sigma^2(y) } \, dy < \infty \,. 
\end{aligned}
\end{equation}

We will also define the following functions 
for some fixed $c \in I$ 
\begin{equation}
\begin{aligned}
    p(x) &\defequal 
    \int_c^x \exp\left( 
    -\int_c^\xi \frac{2b(z)}{\sigma^2(z)} dz 
    \right) d\xi \,, \\ 
    m(dx) &\defequal \frac{ 2 \, dx }{ p'(x) \sigma^2(x) } \,, \\ 
    v(x) &\defequal 
    \int_c^x p'(x) \int_c^y \frac{2 \, dz}{ p'(z) \sigma^2(z) } \, dy 
    = \int_c^x p(x) - p(y) \, m(dx) \,. 
\end{aligned}
\end{equation}

We will also define the following sequence of stopping times for $M > 0$ 
\begin{equation}
    \tau_M \defequal 
    \inf \left\{ t\geq 0 : X_t \geq M \text{ or } X_t \leq M^{-1} \right\} \,, 
\end{equation}
and let $\tau^\ast \defequal \sup_{M>0} \tau_M$. 
Now we will state the main results we need for finite time explosions. 

\begin{lem}
[{\cite[Problem 5.5.27]{karatzas2012brownian}}]
\label{lm:feller_p_implies_v}
We have the following implications 
\begin{equation}
\begin{aligned}
    &\lim_{x\to 0} p(x) = -\infty \implies 
    \lim_{x\to 0} v(x) = \infty \,, \\ 
    &\lim_{x\to \infty} p(x) = \infty \implies 
    \lim_{x\to \infty} v(x) = \infty \,. 
\end{aligned}
\end{equation}
\end{lem}

\begin{thm}
[Feller's Test for Explosions {\cite[Theorem 5.5.29]{karatzas2012brownian}}]
\label{prop:feller_test_ver1}
Assume the conditions in \cref{eq:nd_li} are satisfied. 
Then $\mathbb{P}[\tau^\ast = \infty] = 1$ if and only if 
\begin{equation}
    \lim_{x\to 0} v(x) 
    = \lim_{x\to\infty} v(x) = \infty \,. 
\end{equation}
\end{thm}

We will begin our derivations for the SDE 
\cref{eq:app_finite_time_explosion_sde}. 

\begin{lem}
[Geometric Brownian Motion, the $b=0$ Case]
\label{lm:explosion_geometric_brownian_motion}
Let $X_t$ be a solution to the following SDE 
\begin{equation}
    dX_t = \sqrt{2} X_t \, dB_t \,, \quad X_0 = x_0 > 0 \,, 
\end{equation}
then we have that $\tau^\ast = \infty$ a.s. 
\end{lem}

\begin{proof}

Here we observe that 
\begin{equation}
    p'(x) = \exp(0) = 1 \implies 
    p(x) = x-1 \,. 
\end{equation}

Then we have that 
\begin{equation}
    m(dx) = \frac{2 \, dx}{ p'(x) \sigma^2(x) } 
    = \frac{dx}{x^2} \,, 
\end{equation}
which implies 
\begin{equation}
    v(x) = (x-1) \int_1^x \frac{dy}{y^2} - \int_1^x \frac{y-1}{y^2} dy 
    = x - \log x - 1 \,, 
\end{equation}
and therefore 
\begin{equation}
    \lim_{x\to 0} v(x) = \lim_{x\to \infty} v(x) = \infty \,. 
\end{equation}

By Feller's test for explosions \cref{prop:feller_test_ver1}, we have the desired result. 

\end{proof}

\begin{prop}
[Calculate $p(x), m(dx)$ and the $b \leq -1$ Case]
\label{prop:explosion_b_minus}
Suppose $X_t$ is a solution of the following equation 
\begin{equation}
    dX_t = b X_t (X_t - 1) \, dt + \sqrt{2} X_t \, dB_t \,, \quad 
    X_0 = x_0 > 0\,, 
\end{equation}
then for all $b\neq 0$ we have that 
\begin{equation}
    p(x) = e^b \int_1^x e^{-by} y^b dy \,, 
    \quad 
    m(dx) = \frac{ dx }{ 
        e^{b} e^{-bx} x^{b+2} 
        }\,.  
\end{equation}
This implies that 
\begin{equation}
    \lim_{x\to 0} p(x) = 
    \begin{cases}
        -\infty \,, & b \leq -1 \,, \\ 
        \text{finite} \,, & b > -1 \,, 
    \end{cases}
    \quad 
    \lim_{x\to\infty} p(x) = 
    \begin{cases}
        \infty \,, & b \leq 0 \,, \\ 
        \text{finite} \,, & b > 0 \,. 
    \end{cases}
\end{equation}
In particular, when $b \leq -1$, we have that 
$\lim_{x\to 0} v(x) = \lim_{x\to \infty} v(x) = \infty$. 
\end{prop}

\begin{proof}

We start by writing 
\begin{equation}
    p'(x) 
    = \exp\left( - \int_1^x \frac{2 b(y)}{\sigma^2(y)} dy \right) 
    = \exp\left( - b (x-1) + b \log x \right) 
    = e^b e^{-bx} x^b \,. 
\end{equation}

Then we can also calculate the integral via a substitution of $y=bx$ to get the desired result. 

At this time, we observe that when $b>0$ 
\begin{equation}
\begin{aligned}
    p(x) 
    &= 
    \int_1^x p'(y) \, dy \\ 
    &= e^b \int_1^x e^{-by} y^b \, dy \\ 
    &= e^b b^{-b-1} \int_b^{bx} e^{-z} z^b \, dz \\ 
    &= e^b b^{-b-1} \left( \gamma( b+1, bx ) - \gamma( b+1, b ) \right) \,, 
\end{aligned}
\end{equation}
where $\gamma$ is the lower incomplete gamma function, and therefore finite for all values of $x$ including the limits $x\to 0, \infty$. 

The $b=0$ case follows from \cref{lm:explosion_geometric_brownian_motion}. 
Finally when $b<0$ we can write 
\begin{equation}
    p(x) = e^{-|b|} \int_1^x \frac{ e^{ |b|y } }{ y^{|b|} } \, dy \,, 
\end{equation}
which clearly diverges to $\infty$ as $x\to\infty$. 

On the other hand, we can observe as that as $x\to 0$, we have that $y \in [0,1]$ and therefore $1 \leq e^{|b|y} \leq e^{|b|}$. 
This implies we only need to consider the integral 
$-\int_x^1 y^{- |b|} \, dy$, 
which diverges to $-\infty$ 
if and only if $|b| \geq 1$. 
In other words we have 
\begin{equation}
    \lim_{x\to 0} p(x) = 
    \begin{cases}
        -\infty \,, & b \leq -1 \,, \\ 
        \text{finite} \,, & b > -1 \,. 
    \end{cases}
\end{equation}

The limits on $v(x)$ follows from \cref{lm:feller_p_implies_v}. 

\end{proof}

\begin{prop}
[The $b > -1$ Case]
\label{prop:explosion_b_plus}
Suppose $X_t$ is a solution of the following equation 
\begin{equation}
    dX_t = b X_t (X_t - 1) \, dt + \sqrt{2} X_t \, dB_t \,, \quad 
    X_0 = x_0 > 0\,, 
\end{equation}
then when $b>-1$, we have that 
\begin{equation}
    \lim_{x\to 0} v(x) = \infty \,, \quad 
    \lim_{x\to \infty} v(x) = 
    \begin{cases}
        \infty \,, & b \in (-1,0] \,, \\ 
        <\infty \,, & b > 0 \,. 
    \end{cases}
\end{equation}
\end{prop}

\begin{proof}

We will start by calculating the following integral using the exponential series expansion 
\begin{equation}
\begin{aligned}
    \int_1^y \frac{2 \, dz}{ p'(z) \sigma^2(z) } 
    &= 
    e^{-b} \int_1^y e^{bz} z^{-(b+2)} \, dz 
    \\ 
    &= 
    e^{-b} 
    \int_1^y \sum_{k \geq 0} \frac{(bz)^k}{k!}     z^{-(b+2)} \, dz 
    \\ 
    &= 
    e^{-b} 
    \sum_{k\geq 0, k \neq b+1} 
    \frac{b^k}{k!} \frac{y^{k-b-1} - 1}{k-b-1} 
    + \frac{b^{b+1}}{(b+1)!} \log (y) 
    \mathds{1}_{\{k = b+1\}} \,. 
\end{aligned}
\end{equation}

Now we can compute $v(x)$ 
\begin{equation}
\begin{aligned}
    v(x) 
    &= 
    \int_1^x e^{-by} y^b 
    \left( 
    \sum_{k\geq 0, k \neq b+1} 
    \frac{b^k}{k!} \frac{y^{k-b-1} - 1}{k-b-1} 
    + \frac{b^{b+1}}{(b+1)!} \log (y) 
    \mathds{1}_{\{k = b+1\}} 
    \right) \, dy 
    \\ 
    &= 
    \sum_{k\geq 0, k \neq b+1} 
    \frac{b^k}{k! (k-b-1)} 
    \int_1^x e^{-by} ( y^{k-1} - y^b ) dy 
    + 
    \frac{b^{b+1}}{(b+1)!} 
    \mathds{1}_{\{k = b+1\}} 
    \int_1^x e^{-by} y^b 
    \log y \, dy \,. 
\end{aligned}
\end{equation}

We first consider the case when $x\to 0$, in which case we have $e^{-|b|} \leq e^{-by} \leq e^{|b|}$ and therefore will not affect convergence or divergence, so we can safely ignore the factor $e^{-by}$ and write (for $k>0, x \to 0$) 
\begin{equation}
    \int_1^x e^{-by} ( y^{k-1} - y^b ) \, dy 
    \approx 
    \left[ \frac{y^k}{k} - \frac{y^{b+1}}{b+1} \right]_1^x 
    = \frac{x^k-1}{k} - \frac{x^{b+1} - 1 }{b+1} 
    \to \frac{-1}{k} + \frac{1}{b+1} 
    \,. 
\end{equation}

Since the exponential series $\sum_{k>0} \frac{b^k}{k!} = e^b - 1$ converges, and we have terms strictly smaller than the exponential series, we have convergence of these terms when $k>0$. 
We now return to handle a couple of edge case terms, firstly when $k=0$ 
\begin{equation}
    \frac{1}{-b-1} \int_1^x e^{-by} y^{-1} \, dy 
    \approx 
    \frac{- \log x }{|1+b|} 
    \to \infty \,, 
    \quad 
    \text{ as } x\to\infty \,, 
\end{equation}
which is a desired behaviour. 
Secondly we consider when $k = b+1$
\begin{equation}
    \int_1^x y^b \log y \, dy 
    = \frac{x^{b+1} \left[ (b+1) \log x - 1 \right] + 1 }{ (b+1)^2 } 
    \to (b+1)^{-2} \,, 
    \quad 
    \text{ as } x \to \infty\,, 
\end{equation}
from which we can conclude $\lim_{x\to 0} v(x) = \infty$. 

Next we consider the case when $x\to\infty$. 
Firstly, since we already have that $p(x) \to \infty$ when $b \leq 0$, therefore \cref{lm:feller_p_implies_v} implies $v(x) \to \infty$. 
Therefore we only need to consider when $b>0$. 

Since $b>0$ we will have that $e^{-bx}$ will dominate, and therefore we can safely ignore all the edge case terms and consider the series 
\begin{equation}
    v(x) \approx 
    \sum_{k > 0, k \neq b+1} 
    \frac{b^k}{k!(k-b-1)} 
    \int_1^x e^{-by} ( y^{k-1} - y^b ) \, dy \,. 
\end{equation}

Observe that as $x\to\infty$ we actually recover the gamma integral in the terms i.e. 
\begin{equation}
    \int_1^\infty e^{-by} ( y^{k-1} - y^b ) \, dy 
    = - b^{-k} \Gamma(k) 
    + b^{-b-1} \Gamma (b+1) \,, 
\end{equation}
where we observe the second term is independent of $k$, and therefore the series converges due to comparison with the exponential Taylor series. 
This implies we only need to focus on the first term, which is 
\begin{equation}
    v(x) 
    \approx 
    \sum_{k>0, k \neq b+1} 
    \frac{b^k}{k! (k-b-1)} (-b^{-k}) (k-1)! 
    = 
    \sum_{k>0, k \neq b+1} 
    \frac{-1}{k (k-b-1)} 
    < 
    \infty \,, 
\end{equation}
where the series converges since it's a sum of $k^{-2}$ type. 
This allows us to conclude that 
$\lim_{x\to\infty} v(x) < \infty$ as desired.

\end{proof}

We can now prove the desired result of \cref{prop:finite_time_explosion}, which we restate below. 

\begin{prop}
[Finite Time Explosion]
\label{prop:finite_time_explosion_app}
Let $X_t \in \bR_+$ be a solution to the following SDE 
\begin{equation}
    dX_t = b X_t (X_t - 1) \, dt + \sqrt{2} X_t \, dB_t \,, \quad 
    X_0 = x_0 > 0 \,, b \in \mathbb{R} \,. 
\end{equation}

Let $\ta^\ast = \sup_{M > 0} \inf\{t :X_t \geq M \text{ or } X_t \leq M^{-1} \}$ be the explosion time, and we say $X_t$ has a finite time explosion if $\ta^\ast < \infty$. 
For this equation, $\mathbb{P}[\ta^\ast = \infty] = 1$ if and only if $b \leq 0$. 
\end{prop}

\begin{proof}

Putting the results of \cref{prop:explosion_b_minus} and \cref{prop:explosion_b_plus} together, we have the following table 

\begin{center}
\begin{tabular}{ |c|c|c|c|c| } 
 \hline
  & $\lim_{x\to 0} v(x)$ 
  & $\lim_{x\to \infty} v(x)$ 
  & $\lim_{x\to 0} p(x)$
  & $\lim_{x\to \infty} p(x)$ 
  \\ 
 \hline
 $b\leq -1$ & $\infty$ & $\infty$ & $-\infty$ & $\infty$ \\ 
 $-1 < b \leq 0$ & $\infty$ & $\infty$ & finite & $\infty$ \\ 
 $b > 0$ & $\infty$ & finite & finite & finite \\ 
 \hline
\end{tabular}
\end{center}

Therefore, invoking Feller's test for explosions from \cref{prop:feller_test_ver1}, we have that $\mathbb{P}[\tau^\ast = \infty]$ if and only if $b \leq 0$.

\end{proof}

\section{Lower Bound for the \texorpdfstring{Recursion $\rho_{\ell+1} = cK_1(\rho_\ell)$}{Correlation Recursion}}
\label{sec:app_independent_interest}

In this section, we consider a Taylor expansion of $cK_1(\rho)$ around $\rho\to 1$ from the left hand side to get 
\begin{equation}
    \rho_{\ell+1} 
    = c K_1(\rho_\ell) 
    = \rho_\ell + \frac{2\sqrt{2}}{3\pi} (1-\rho_\ell)^{3/2} 
    + O( (1-\rho_\ell)^{5/2} ) \,, 
\end{equation}
which we can rewrite using $r_\ell = 1 - \rho_\ell$ as 
\begin{equation}
    r_{\ell+1}
    = r_\ell - \frac{2\sqrt{2}}{3\pi} r_\ell^{3/2} 
    + O( r_\ell^{3/2} ) \,. 
\end{equation}

We will compute an upper bound on $r_\ell$ inspired by the following result. 

\begin{lem}[Lemma A.6, \cite{campbell2019automated}]
\label{lm:logistic_lb}
The logistic recursion 
\begin{equation}
	x_{n+1} \leq \alpha x_n (1-x_n) \,, 
\end{equation}
for $x_0,\alpha \in [0,1]$ satisfies 
\begin{equation}
	x_n \leq \frac{x_0}{ \alpha^{-n} + x_0 n } \,. 
\end{equation}
\end{lem}

We will extend the above Lemma to a slightly modified update as well. 

\begin{lem}
Suppose the recursive map satisfies 
\begin{equation}
	x_{n+1} \leq x_n (1 - x_n^{1/2}) \,, 
\end{equation}
for $x_0 \in [0,1]$, then we also have that 
\begin{equation}
	x_n \leq \frac{ x_0 }{ \left(1 + \frac{1}{3} n x_0^{1/2} \right)^2 } \,. 
\end{equation}
\end{lem}

\begin{proof}

We will start the induction proof at $n=1$ 
\begin{equation}
	x_1 
	\leq x_0 (1 - x_0^{1/2}) 
	\leq \frac{x_0}{1 + x_0^{1/2}} \,. 
\end{equation}

When $x_0 \leq 9$ we have that 
\begin{equation}
	1 + x_0^{1/2} \geq 1 + \frac{1}{9} x_0 + \frac{2}{3} x_0^{1/2} 
	= \left( 1 + \frac{1}{3} x_0^{1/2} \right)^2 \,, 
\end{equation}
and hence 
\begin{equation}
	x_1 \leq \frac{x_0}{ \left( 1 + \frac{1}{3} x_0^{1/2} \right) } \,, 
\end{equation}
which proves the case for $n=1$. 

Then we assume the inequality holds for $x_n$, we will similarly write 
\begin{equation}
	x_{n+1} \leq x_n( 1 - x_n^{1/2} ) \leq \frac{x_n}{ 1 + x_n^{1/2} } \,, 
\end{equation}
and plugging in the inequality for $x_n$ we get 
\begin{equation}
	x_{n+1} 
	\leq \frac{ \frac{x_0}{\left(1 + \frac{1}{3} n x_0^{1/2} \right)^2} 
		}{ 1 + \frac{x_0^{1/2}}{1 + \frac{1}{3} n x_0^{1/2} } } 
	= \frac{x_0}{ \left( 1 + \left( \frac{n}{3}+1 \right) x_0^{1/2} \right) 
		\left( 1 + \left( \frac{n}{3} \right) x_0^{1/2} \right) 
		}
	\,, 
\end{equation}

To complete the proof it's sufficient to show 
\begin{equation}
	\left( 1 + \left( \frac{n}{3}+1 \right) x_0^{1/2} \right) 
		\left( 1 + \left( \frac{n}{3} \right) x_0^{1/2} \right) 
	\geq 
	\left( 1 + \left( \frac{n+1}{3} \right) x_0^{1/2} \right)^2 \,, 
\end{equation}
which is equivalent to 
\begin{equation}
	\left( \frac{n}{3} + 1 \right) \frac{n}{3} x_0 
	+ \left( \frac{2n}{3} + 1 \right) x_0^{1/2} 
	\geq 
	\frac{(n+1)^2}{9} x_0 + \frac{2(n+1)}{3} x_0^{1/2} \,. 
\end{equation}

Since $\frac{2n}{3} + 1 \geq \frac{2(n+1)}{3}$, 
we only need to compare the first coefficient, which is 
\begin{equation}
	\frac{ n^2 + 3n }{9} \geq \frac{n^2 + 2n + 1}{9} \,, 
\end{equation}
and this is equivalent to $n\geq 1$, 
and therefore satisfied by the induction. 
This completes the proof. 

\end{proof}

At the same time, we also conjecture the following bound. 

\begin{conjecture}
\label{conj:logistic_lb}
Suppose the recursive map satisfies 
\begin{equation}
	x_{n+1} = x_n (1 - x_n^{1/2}) \,, 
\end{equation}
for $x_0 \in [0,1]$, then 
\begin{equation}
	x_n \approx \frac{ x_0 }{ \left(1 + \frac{1}{2} n x_0^{1/2} \right)^2 } \,. 
\end{equation}
\end{conjecture}

\begin{proof}[Sketch of Conjecture]

Suppose we want to establish the approximation of 
\begin{equation}
	x_n \leq \frac{ x_0 }{ \left(1 + b n x_0^{1/2} \right)^2 } \,. 
\end{equation}

Then for the initial induction $n=1$ step, we only need 
\begin{equation}
	1 + x_0^{1/2} \geq (1 + b x_0^{1/2})^2 \,, 
\end{equation}
which is equivalent to 
\begin{equation}
	x_0 \leq \frac{(1-2b)^2}{b^4} \,. 
\end{equation}

Using WolframAlpha (probably through the quartic formula), we find the desired solution for $b \in (0,1/2)$ is 
\begin{equation}
	b = \frac{ \sqrt{ 1 + \sqrt{x_0} } - 1 }{ \sqrt{x_0} } \,. 
\end{equation}

This function $b(x_0)$ is a strictly decreasing function on $[0,1]$, 
and it satisfies $b(0) = \frac{1}{2}, b(1) = \sqrt{2} - 1$. 
This implies that whenever $x_0$ is small, we can choose $b$ closer to $\frac{1}{2}$ in the $n=1$ step of the induction. 

Similarly, for the induction step, it's sufficient to show 
\begin{equation}
	\left( 1 + (nb+1) x_0^{1/2} \right)
	\left( 1 + nb x_0^{1/2} \right) 
	\geq 
	\left( 1 + (n+1) b x_0^{1/2} \right)^2 \,, 
\end{equation}
which is equivalent to 
\begin{equation}
	nb x_0^{1/2} + 1 \geq (2n+1) b^2 x_0^{1/2} + 2b \,. 
\end{equation}

Again, since we are always choosing $b \leq 1/2$, 
therefore we have $1 \geq 2b$, 
and we will only need to focus on the first coefficient. 
To this end we rewrite the first term as 
\begin{equation}
	b ( n (1-2b) - b ) x_0^{1/2} 
	= b(1-2b) \left( n - \frac{b}{1 - 2b} \right) x_0^{1/2} \,. 
\end{equation}

This implies we require $n \geq \frac{b}{1 - 2b}$, 
which increases as we choose $b$ closer to $1/2$. 
However, if the induction starts the step $\lceil \frac{b}{1 - 2b} \rceil$, 
then this is not a problem, which leads to our conjecture. 

\end{proof}

Using the above results, we can have a similar approximation for 
$r_\ell$ given the infinite-width update 
\begin{equation}
	r_{\ell+1} = r_\ell 
	- \frac{ 2\sqrt{2} }{ 3\pi } r_\ell^{3/2} \,, 
\end{equation}
which we can rewrite using 
$\hat r_\ell := \left( \frac{ 2\sqrt{2} }{ 3\pi } \right)^2 r_\ell$ to get 
\begin{equation}
	\hat r_{\ell+1} = \hat r_\ell ( 1 - \hat r_\ell^{1/2} ) \,. 
\end{equation}

This allows us to consider the upper bound 
\begin{equation}
	\hat r_\ell 
	\leq \frac{ \hat r_0 }{ 
		\left( 1 + \frac{1}{3} \ell \, \hat r_0^{1/2} \right)^2 } \,, 
\end{equation}
or equivalently 
\begin{equation}
\label{eq:logistic_bd_proved}
	r_\ell 
	\leq \frac{ r_0 }{ \left( 
		1 + \frac{2 \sqrt{2}}{ 9 \pi } \ell \, r_0^{1/2}
	\right)^2 } \,. 
\end{equation}

Similarly, the conjecture leads to the following approximation 
\begin{equation}
\label{eq:logistic_bd_conj}
	r_\ell 
	\approx \frac{ r_0 }{ \left( 
		1 + \frac{\sqrt{2}}{ 3 \pi } \ell \, r_0^{1/2}
	\right)^2 } \,. 
\end{equation}

\begin{figure}[ht!]
\centering
\includegraphics[width=0.9\textwidth]{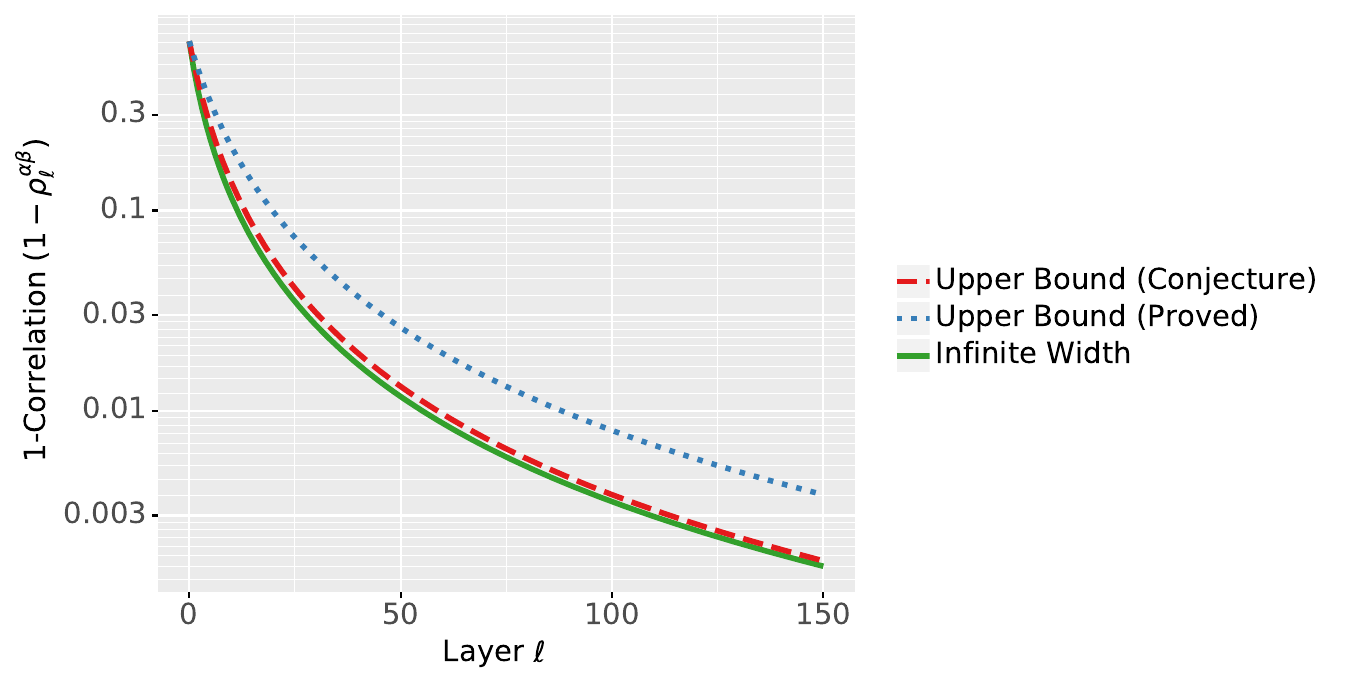}
\caption{
Plot of the convergence of correlation $\rho^{\alpha\beta}_\ell$ to $1$ for a ReLU network, and the lower bounds \cref{eq:logistic_bd_proved} and \cref{eq:logistic_bd_conj}. 
Computed with $d=n=150, \rho^{\alpha\beta}_0 = 0.3$, using the usual ReLU activation i.e. $\varphi_s(x) = \max(x,0)$. 
}
\label{fig:logistic_lb}
\end{figure}

\vfill
\pagebreak

\section{Additional Simulations and Discussions}
\label{sec:app_additional_simulations}

In this section, we have additional simulations plotting the densities of $\rho^{\alpha\beta}_d$ and $V^{\alpha\beta}_d$ for shaped ReLU-like, sigmoid, and softplus networks. 
In particular, the density of $V^{\alpha\beta}_d$ for ReLU-like networks can be found in \cref{fig:relu_cov_density}, the densities for sigmoid in \cref{fig:sigm_cov_rho_density}, and the densities for softplus in \cref{fig:softplus_cov_rho_density}.

\begin{figure}[ht!]
\centering
\includegraphics[width=0.45\textwidth]{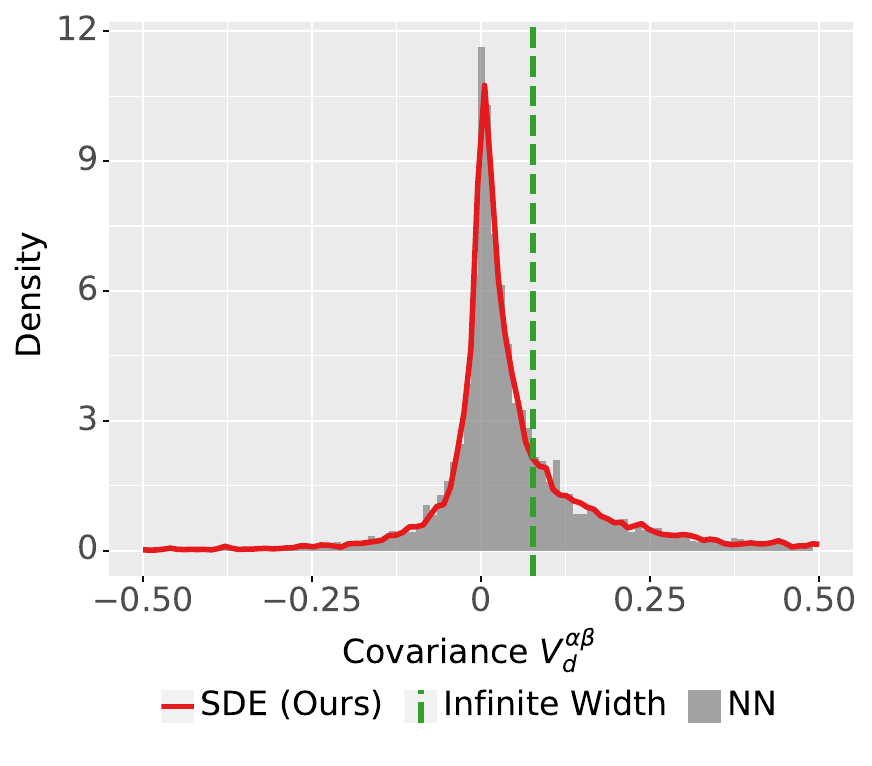}
\caption{
Empirical distribution of the covariance $V^{\alpha\beta}_d$ for a ReLU-like network, SDE sample density computed via kernel density estimation. 
Infinite width prediction simulated from the ODE $\partial_t \rho^{\alpha\beta}_t = \nu(\rho^{\alpha\beta}_t)$, and we note $V^{\alpha\alpha}_t = V^{\alpha\alpha}_0$ in the infinite width limit. 
Simulated with $n = d = 150, c_+ = 0, c_- = -1, \rho^{\alpha\beta}_0 = 0.3$, SDE and ODE step size $10^{-2}$, and $2^{13}$ samples. 
}
\label{fig:relu_cov_density}
\end{figure}

\begin{figure}[ht!]
     \centering
     \begin{subfigure}[b]{0.45\textwidth}
         \centering
          \includegraphics[width=\textwidth]{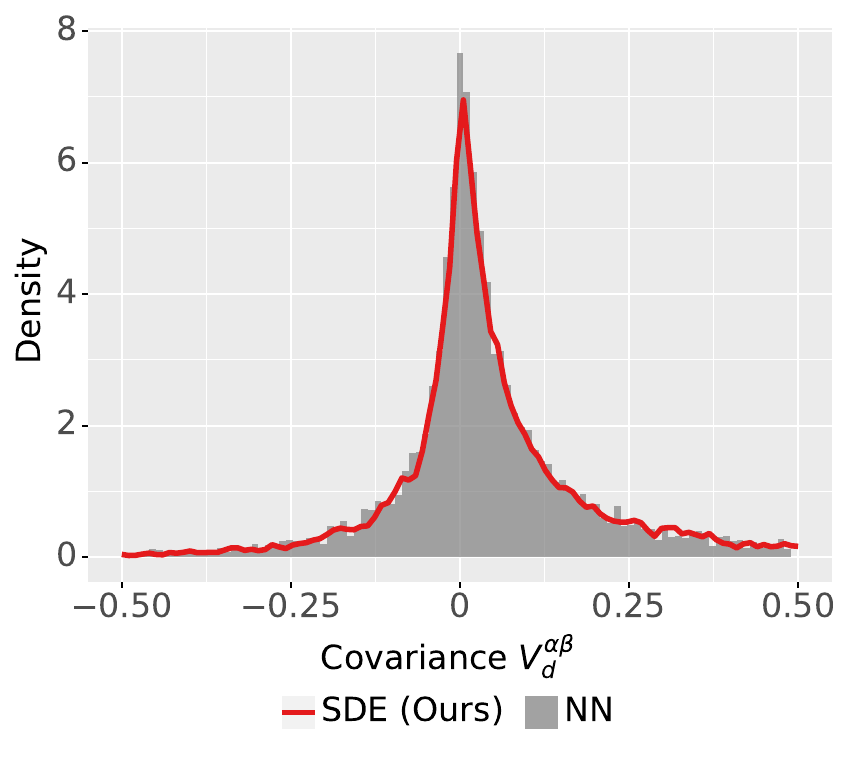}
     \end{subfigure}
     \hfill
     \begin{subfigure}[b]{0.45\textwidth}
         \centering
         \includegraphics[width=\textwidth]{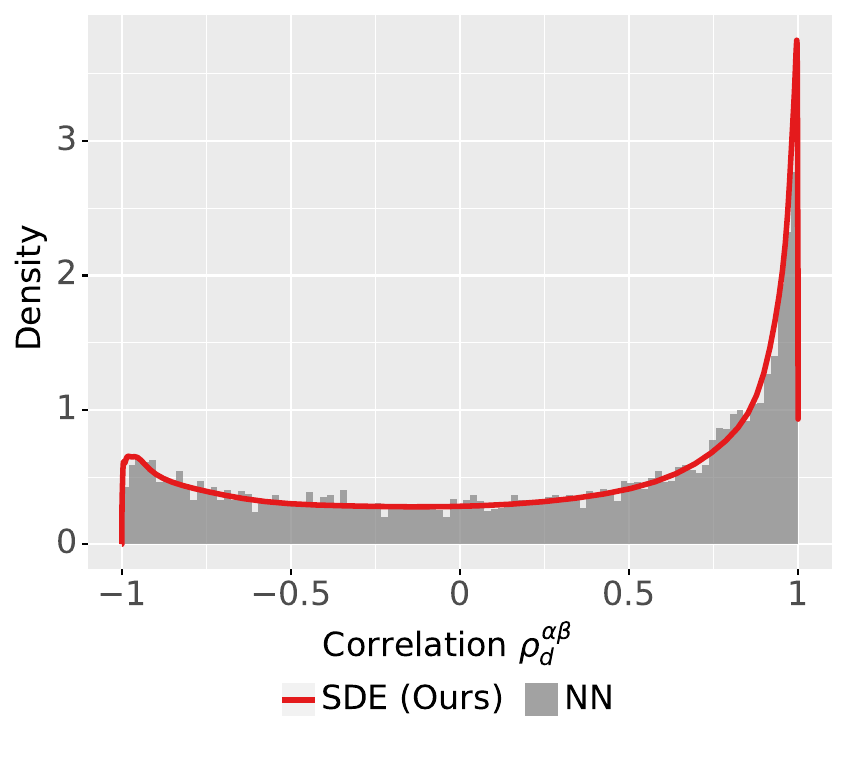}
     \end{subfigure}
\caption{Empirical distribution of the covariance $V^{\alpha\beta}_d$ and correlation $\rho^{\alpha\beta}_d$ for a shaped sigmoid network, SDE sample density computed via kernel density estimation.
Simulated with $n = d = 150, a = 1, \rho^{\alpha\beta}_0 = 0.3$, SDE step size $10^{-2}$, and $2^{13}$ samples. }
\label{fig:sigm_cov_rho_density}
\end{figure}

\begin{figure}[ht!]
     \centering
     \begin{subfigure}[b]{0.45\textwidth}
         \centering
          \includegraphics[width=\textwidth]{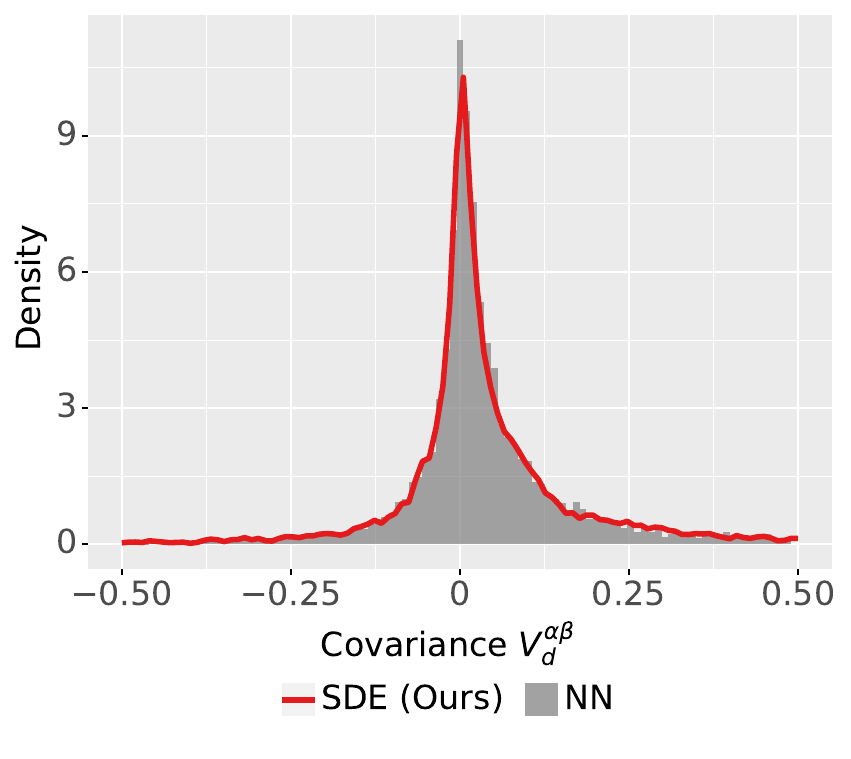}
     \end{subfigure}
     \hfill
     \begin{subfigure}[b]{0.45\textwidth}
         \centering
         \includegraphics[width=\textwidth]{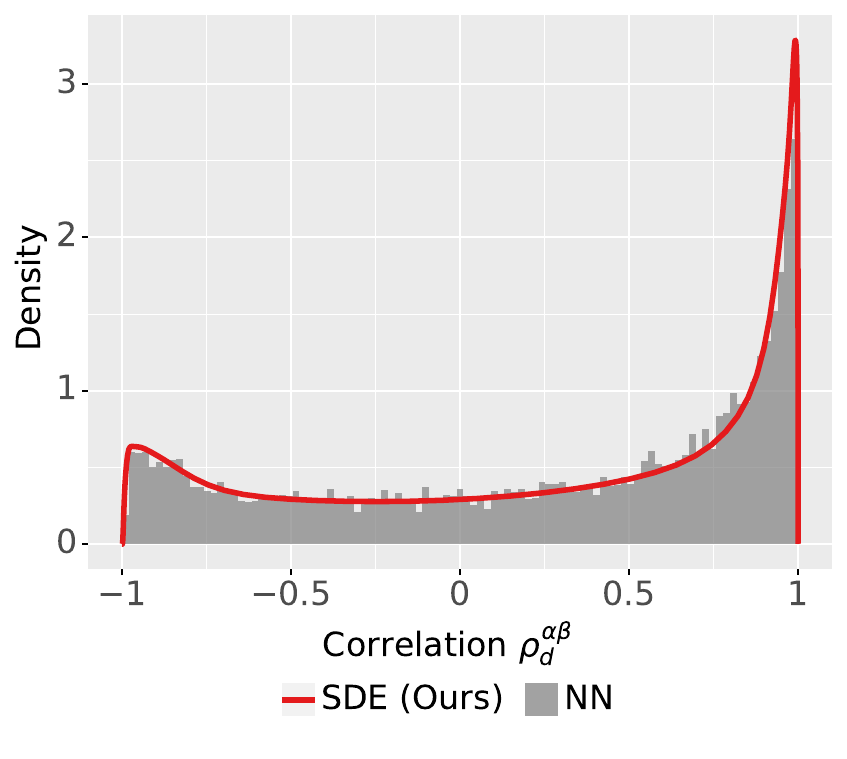}
     \end{subfigure}
\caption{Empirical distribution of the covariance $V^{\alpha\beta}_d$ and correlation $\rho^{\alpha\beta}_d$ for a shaped softplus network (centered at $x_0 = \log 2)$, SDE sample density computed via kernel density estimation.
Simulated with $n = d = 150, a = 1, \rho^{\alpha\beta}_0 = 0.3$, SDE step size $10^{-2}$, and $2^{13}$ samples. }
\label{fig:softplus_cov_rho_density}
\end{figure}

\subsection{Convergence in Kolmogorov--Smirnov Distance}

From \cref{fig:ks_vs_n}, we can show that our results (\cref{thm:relu_corr}) converges at a rate of $n^{-1/2}$ in terms of the KS-distance. 

\begin{figure}[ht!]
\centering
\includegraphics[width=0.45\textwidth]{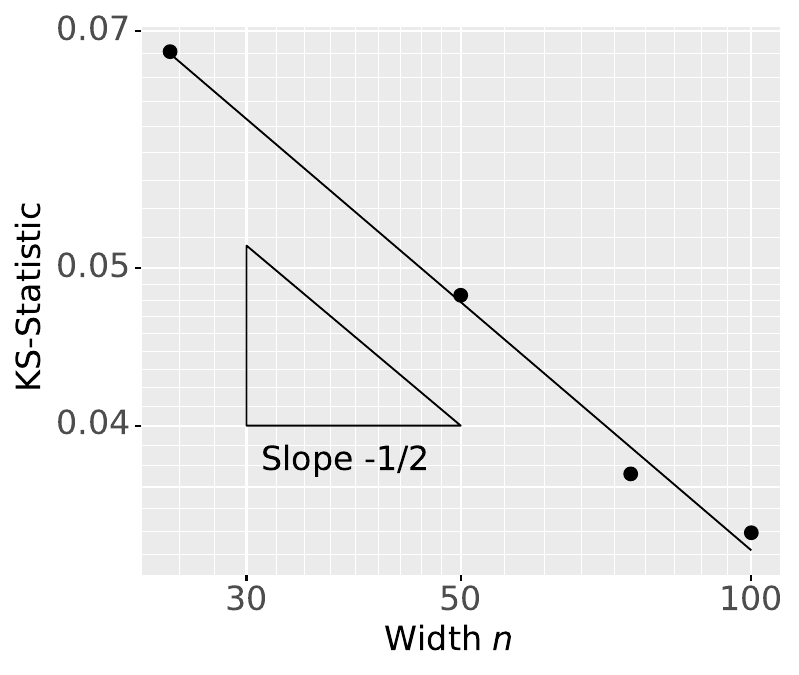}
\caption{
The Kolmogorov--Smirnov statistic (sup norm of the difference between two empirical CDFs) for the empirical samples of the correlation SDE \cref{eq:relu_corr_sde} and from a neural network at initialization. 
Simulated with $c_+ = 0, c_- = -1, \rho^{\alpha\beta}_0 = 0.3, \frac{d}{n} = T = 1$, SDE step size $10^{-2}$, and $2^{13}$ samples. 
}
\label{fig:ks_vs_n}
\end{figure}

\subsection{Tuning Shape and Depth-to-Width Ratio}

\begin{figure}[ht!]
\centering
\includegraphics[width=0.95\textwidth]{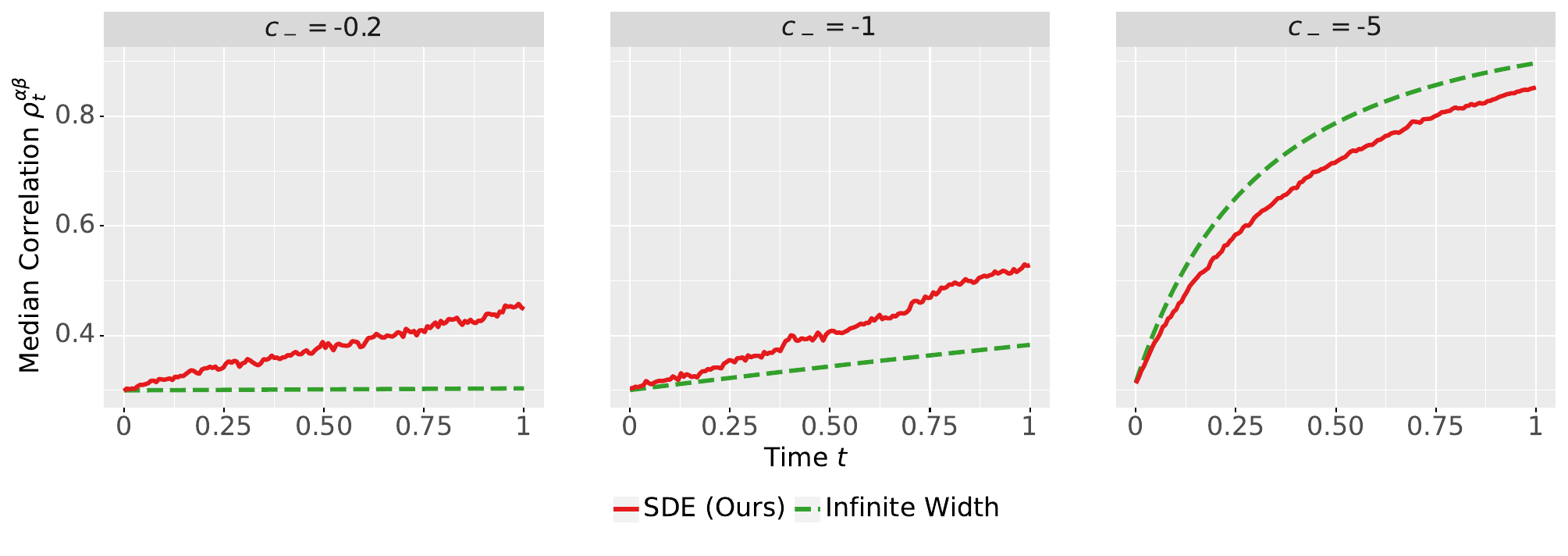}
\caption{
ReLU Correlation SDE \eqref{eq:relu_corr_sde} and ODE simulated with  
$c_+=0, \rho^{\alpha\beta}_0 = 0.3$ 
varying $c_-$ values, $2^{12}$ samples, 
and step size $10^{-2}$. 
infinite-width is from ODE $\partial_t \rho^{\alpha\beta}_t = \nu(\rho^{\alpha\beta}_t)$ with $\nu$. 
}
\label{fig:tuning}
\end{figure}

Since the existing shaping methods \cite{martens2021rapid,zhang2022deep} estimates the output correlation based on the infinite-width limit, we can easily improve the shape tuning based on the covariance SDEs. 
In particular, we consider the example of ReLU-like activations with correlation described by the SDE \cref{eq:relu_corr_sde}. 
By simulating both the SDE and the infinite-width limit ODE, we arrive at the results in \cref{fig:tuning}. 

We observe that simply by increasing $c_-$ towards zero does not automatically reduce effects on the correlation when time $t$ (the depth-to-width ratio) is large. 
In other words, even a linear network will observe an increase in correlation when depth is large enough. 
Therefore \emph{shaping the activation alone is insufficient}, but we also need to account for the depth-to-width ratio. 

We also remark that \cref{fig:tuning} only plotted the median for simplicity, but if we recall the density plots from \cref{fig:density_path_shape}, correlation is heavily skewed and concentrated near $1$. 
More precisely, while the median correlation is approximately $0.55$, roughly $20\%$ of the samples are larger than $0.9$. 
In other words, one in five random initializations will lead to a correlation worse than $0.9$! 
As a consequence, practitioners implementing the shaping methods of \cite{martens2021rapid,zhang2022deep} should consider simulating the correlation SDE to account for the heavy skew.

\end{document}